%% file: main.tex
	\documentclass[sigconf, review=false, 10pt]{acmart}
\setcopyright{none}
\acmConference[SIGMOD'20]{ACM SIGMOD conference}{June 14 2020}{ Portland, Oregon, USA}
\acmYear{2020}
\copyrightyear{2020}

\input{datalabmacros}

\copyrightyear{2020} 
\acmYear{2020} 
\setcopyright{acmlicensed}
\acmConference[SIGMOD'20]{Proceedings of the 2020 ACM SIGMOD International Conference on Management of Data}{June 14--19, 2020}{Portland, OR, USA}
\acmBooktitle{Proceedings of the 2020 ACM SIGMOD International Conference on Management of Data (SIGMOD'20), June 14--19, 2020, Portland, OR, USA}
\acmPrice{15.00}
\acmDOI{10.1145/3318464.3380577}
\acmISBN{978-1-4503-6735-6/20/06}

\begin{document}

\title[Factorized Graph Representations for Semi-Supervised Learning from Sparse Data]
	{Factorized Graph Representations for \\Semi-Supervised Learning from Sparse Data}

\author{Krishna Kumar P.}
\affiliation{%
  \institution{IIT Madras}
}

\author{Paul Langton}
\affiliation{%
  \institution{Northeastern University}
}

\author{Wolfgang Gatterbauer}
\affiliation{%
  \institution{Northeastern University}
}

\begin{abstract}

\input{sections/abstract}
\end{abstract}

\maketitle

\input{sections/introduction}
\input{sections/formal_setup}

\input{sections/label_propagation}

\input{sections/compatibility_estimation}
\input{sections/experiments}

\input{sections/conclusions}

\vspace{-4mm}
\smallsection{Acknowledgements}
This work was supported in part by NSF under award CAREER III-1762268.

\clearpage
\bibliographystyle{ACM-Reference-Format}
\bibliography{bib/linBP_SSL.bib}

\clearpage
\appendix
\input{sections/nomenclature}

\input{sections/appendix}

\end{document}

%% file: datalabmacros.tex
\usepackage{amsopn}		%
\usepackage{mathtools}				%
\usepackage{upgreek} 	%
\usepackage{bm}         %

\usepackage{xspace}            %

\usepackage{graphicx} %
\graphicspath{{figs/}} %

\usepackage[]{xcolor}      %
\usepackage{subcaption} 	%
\usepackage{adjustbox}
\def\eox{\unskip\kern 10pt{\unitlength1pt\linethickness{.4pt}$\diamondsuit${}}} %

\usepackage{enumitem}   
\setlist{noitemsep}
\setlist{nolistsep}

\usepackage{booktabs} %

\usepackage{tabularx} %

\usepackage{multirow} %

\hypersetup{                                                            %
	pdfpagemode=UseNone,               %
    pdfstartview=Fit,
    pdfpagelayout=SinglePage,       %
    bookmarks=false,
    bookmarksnumbered=true,
    bookmarksopen=false,             %
    pdftex=true,
    unicode,
    colorlinks=true,       %
    linkcolor=blue,          %
    citecolor=cyan,  %
    filecolor=magenta,      %
    urlcolor=cyan           %
}
\usepackage[ruled,vlined]{algorithm2e}       
     
\DontPrintSemicolon     %
\SetNlSty{}{}{}                %
\SetAlgoInsideSkip{smallskip}   %
\SetAlCapSkip{1.5mm}             %
\SetInd{1.5mm}{1.5mm}             %

\SetAlFnt{\small}			%
\SetAlCapFnt{\small}		%
\SetAlCapNameFnt{\small}

\usepackage[capitalise,nameinlink]{cleveref} %

\newtheorem{theorem}{Theorem}[section]	%
\newtheorem{problem}[theorem]{Problem}
\newtheorem{algorithmN}[theorem]{Algorithm}

\newtheorem{questionW}{Question}
{%
  {\textit{Question \arabic{questionW} #1}
  \stepcounter{questionW}}%
}%
{%
}
\newtheorem{resultW}{Result}
\renewenvironment{resultW}[1][]%
{%
\par\noindent
  {\textbf{Result \arabic{resultW} #1}
  \stepcounter{resultW}}%
  \begin{itshape}%
}%
{\end{itshape}
}
\AtEndPreamble{
    \hypersetup{colorlinks,
      linkcolor=purple,
      citecolor=blue,
      urlcolor=ACMDarkBlue,
      filecolor=ACMDarkBlue}}

\usepackage{tcolorbox}		%
\tcbuselibrary{breakable,skins}		%

\tcbset{examplestyle/.style={
		enhanced jigsaw,	%
		colback=blue!06,	%
		colframe=blue!06,	%
		arc=0mm,
		boxrule=0pt,		%
		left=1mm,
		right=1mm,
		left skip=-1mm,  %
		right skip=-1mm, %
		top=0pt,		%
		bottom=0pt,		%
		breakable,		%
		parbox = false,		%
		before={\par\pagebreak[0]\vspace{1mm}\parindent=0pt},		
		after={\par\pagebreak[0]\vspace{1mm}\parindent=0pt},				
		bottomrule = 0mm,
		boxsep = 0mm,					%
		topsep at break=0pt,			%
		bottomsep at break=0pt,			%
		pad at break=0mm,
		pad before break=0mm,		
		pad after break=1mm,		
		bottomrule at break=0mm,
		toprule at break=0mm,		
		}}

\tcolorboxenvironment{example}{examplestyle}

\tcbset{
		enhanced jigsaw,	%
		colback=red!5,
		colframe=red!75!black,			
		colframe=red!0!black,
		arc=0mm,
		boxrule=1pt,		%
		left=1pt,
		right=1pt,
		topsep at break=1mm,			%
		top=1pt,		%
		bottom=0mm,		%
		breakable,		%
		parbox = false		%
		left=0mm,
		right=0mm,
		top=0mm,
		}

\makeatletter
\DeclareRobustCommand*\uell{\mathpalette\@uell\relax}
\newcommand*\@uell[2]{
  \setbox0=\hbox{$#1\ell$}
  \setbox1=\hbox{\rotatebox{10}{$#1\ell$}}
  \dimen0=\wd0 \advance\dimen0 by -\wd1 \divide\dimen0 by 2
  \mathord{\lower 0.1ex \hbox{\kern\dimen0\unhbox1\kern\dimen0}}
}
\makeatother

\newcommand{\hide}[1]{} %

\newcommand{\smallsectionskip}{5mm}
\newcommand{\smallsection}[1]{\vspace{\smallsectionskip}\noindent\textbf{#1.}}	%
\newcommand{\introparagraph}[1]{\textbf{#1.}}        %

\newcommand{\specificref}[2]{\hyperref[#2]{#1~\ref*{#2}}}			%

\renewcommand{\vec}[1]{\boldsymbol{\mathbf{#1}}}

\newcommand{\transpose}{{\mkern-1.5mu\mathsf{T}}}

\newcommand{\sql}[1]{\textup{\textsf{\small#1}}}

\newcommand{\define}{\triangleq}							%
\newcommand{\half}[1]{\mathbf{\tilde{\text{$#1$}}}}
\newcommand{\thicksymbol}      [1]{\hat{#1} }			%

\renewcommand{\vec}[1]{\boldsymbol{\mathbf{#1}}}

\newcommand{\NB}{\mathrm{NB}}		%

\renewcommand{\H} { {H} }					%

\newcommand{\HVec} { \vec{\H} }				%
\newcommand{\HVecEst} { \thicksymbol{\vec{\H}} }
\newcommand{\VecEstEC}[1] { {\thicksymbol{\vec{#1}}_{\NB}} }

\renewcommand{\P} { {P} }					%
\newcommand{\PEst} { \thicksymbol{\P} }					%
\newcommand{\PVec} { \vec{P} }				%
\newcommand{\PVecEst} { \thicksymbol{\PVec} }
\newcommand{\PVecEstEC} {\VecEstEC{\P} }

\newcommand{\HCenterVec}  { \half{  {\HVec}  }   }	%
\newcommand{\h}  	{ {h}  }
\newcommand{\hVec} 	{ \vec{\h}  }

\renewcommand{\O}  	{ {\mathcal{O}}  }

\newcommand{\W}  	{ {W}  }
\newcommand{\WEC} { {\W_{\NB}} }				%
\newcommand{\WVec}  { \vec{\W}   }				%
\newcommand{\WVecEC} { \vec{\W}_{\NB} }			%

\newcommand{\e}  	{ {x}  }

\newcommand{\eVec}  { \vec{\e}   }
\newcommand{\eCenterVec}  { \half{\vec{\e}}   }	%
\newcommand{\E}  	{ {X}  }

\newcommand{\EVec}  { \vec{\E}   }
\newcommand{\ECenterVec}  { \half{\vec{\E}}   }	%

\newcommand{\m}  	{ {m}  }
\newcommand{\mVec}  { \vec{\m}   }

\newcommand{\mCenterVec}  { \half{\mVec}   }

\renewcommand{\b}  	{ {f}  }
\newcommand{\bVec}  { \vec{\b}   }

\newcommand{\bCenterVec}  { \half{\bVec}   }
\newcommand{\B}  	{ {F}  }
\newcommand{\BVec}  { \vec{\B}   }

\newcommand{\BCenterVec}  { \half{\BVec}   }

\newcommand{\R}  { \mathrm{ \mathbb{R}}}

\newcommand{\tr}  { \mathrm{tr}}		%

\DeclareMathOperator*{\argmin}{argmin}

\newcommand{\dist}{\textup{\textsf{dist}}\xspace} 	%

%% file: sections/abstract.tex
Node classification is an important problem in graph data management. 
It is commonly solved by various \emph{label propagation} methods 
that work iteratively starting from a few labeled seed nodes.
For graphs with arbitrary compatibilities between classes, 
these methods crucially depend on knowing the \emph{compatibility matrix}
that must be provided by either domain experts or heuristics.
Can we instead directly estimate the correct compatibilities from a sparsely labeled graph in a principled and scalable way? 
We answer this question affirmatively and suggest a method called \emph{distant compatibility estimation} 
that works even on extremely sparsely labeled graphs (e.g., 1 in 10,000 nodes is labeled)
in a fraction of the time it later takes to label the remaining nodes.
Our approach first creates multiple 
\emph{factorized graph representations}
(with size independent of the graph)
and then performs estimation on these smaller graph sketches. 
We refer to \emph{algebraic amplification} as the more general idea of 
leveraging algebraic properties of an algorithm's update equations to amplify sparse signals.
We show that our estimator is by orders of magnitude faster than an alternative approach
and that the end-to-end classification accuracy is comparable to using gold standard compatibilities.
This makes it a cheap pre-processing step for any existing label propagation method 
and removes the current dependence on heuristics.

%% file: sections/introduction.tex
\section{Introduction}\label{sec:introduction}

\begin{figure}[t]
\centering
\begin{subfigure}[t]{.4\linewidth}
	\centering
	\includegraphics[scale=0.35]{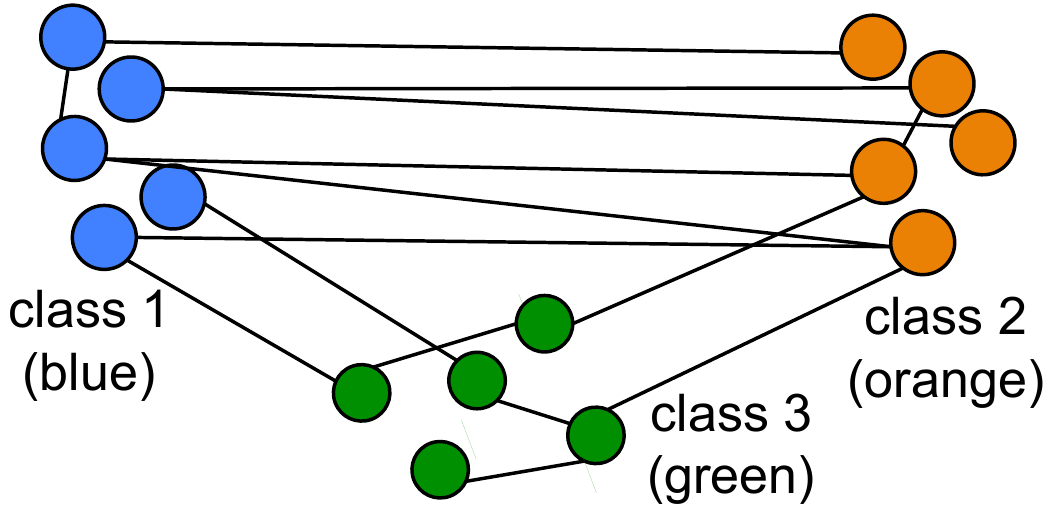}
	\caption{Unobserved truth}	
	\label{Fig_IntroBlockModel_a}
\end{subfigure}
\hspace{6mm}
\begin{subfigure}[t]{.44\linewidth}
	\centering
	\includegraphics[scale=0.38]{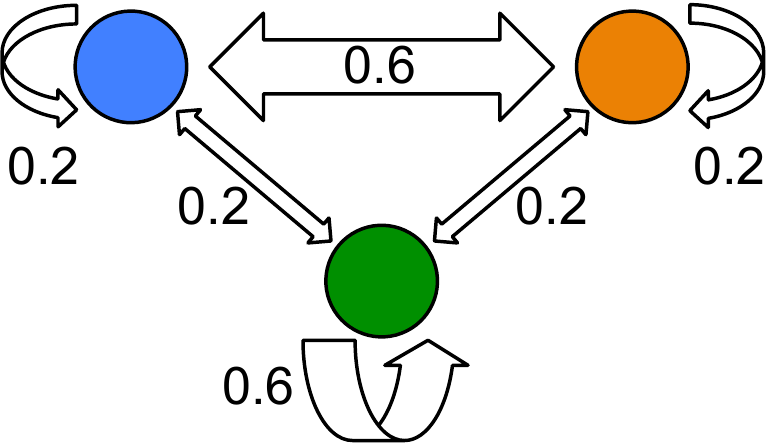}
	\caption{\small Class compatibilities $\HVec$}	
	\label{Fig_IntroBlockModel_c}
\end{subfigure}

\begin{subfigure}[t]{.44\linewidth}
	\centering
	\includegraphics[scale=0.35]{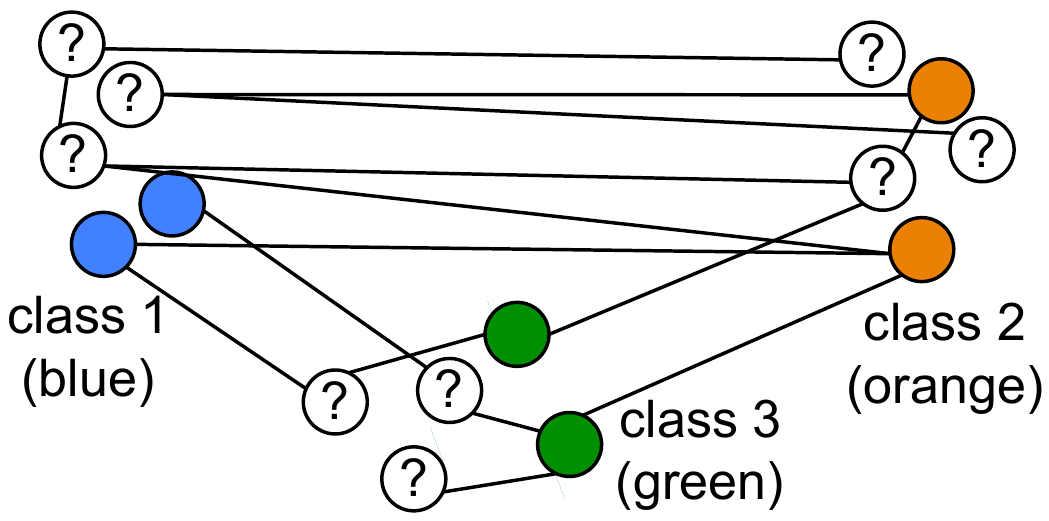}
    \caption{\small Partially labeled graph}	
	\label{Fig_IntroBlockModel_b}
\end{subfigure}
\hspace{6mm}
\begin{subfigure}[t]{.42\linewidth}
	\centering
	\includegraphics[scale=0.35]{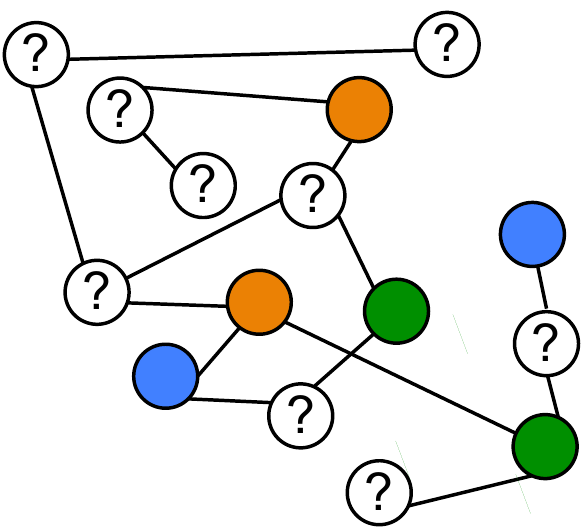}
	\caption{What we actually see}	
	\label{Fig_IntroBlockModel_d}
\end{subfigure}

\caption{(a, b): Graphs are formed based on relative compatibilities between
    classes of nodes.
(c, d): 
We have access to only a few labels $n_\ell \ll n$ and want to classify the remaining nodes
\emph{without} knowing the compatibilities between classes.
}
\label{Fig_IntroBlockModel}
\end{figure}

Node classification (or label prediction) \cite{DBLP:books/sp/social2011/BhagatCM11}  
is an integral 
component of graph data management. 
In a broadly applicable scenario, we are given a large graph 
with edges that reflect affinities between their adjoining nodes
and a small fraction of labeled nodes.  
Most 
graph-based semi-supervised learning (SSL) methods attempt to infer the labels 
of the remaining nodes 
by assuming similarity of neighboring labels. 
For example, people with similar political affiliations are more likely to follow each other on social networks.
This problem is well-studied, and 
solutions are often variations of random
walks that are fast and sufficiently accurate.
However, at other times
opposites attract or complement each other (also 
called \emph{heterophily} or disassortative mixing)
\cite{Koller:2009:PGM:1795555}.
For example,
predators might form a functional group in a biological food
web, not because they interact with each other, but because they eat similar prey~\cite{DBLP:dblp_conf/kdd/MooreYZRL11},
groups of proteins that serve a certain purpose often don't interact wich each other
but rather with complementary protein~\cite{2016:Bhowmick:TKDE},
and in some social networks pairs of nodes are
more likely connected if they are from different classes 
(e.g., members on the social network ``Pokec''~\cite{pokec} being more likely to interact with the opposite gender than the same one).

In more complicated scenarios, such as online auction fraud, 
fraudsters are more likely linked to
accomplices, and we
have a mix of
homophily and heterophily between multiple classes of nodes~\cite{DBLP:conf/www/PanditCWF07}.

\begin{example}[Email]
Consider a corporate email network with three
different classes of users. Class 1, the marketing people, often email 
class 2, the engineers (and v.v.), whereas users of class 3, the C-Level
Executives, tend to email amongst themselves (\cref{Fig_IntroBlockModel_c}).
Assume we are given the labels (classes) of very few nodes (\cref{Fig_IntroBlockModel_b}).
How can we infer the labels of the remaining nodes?
\end{example}

For these scenarios, standard random walks do not work
as they cannot capture such arbitrary compatibilities.
Early works addressing this problem propose belief propagation (BP) for labeling graphs,
since BP can express arbitrary compatibilities between labels.
However, the update equations of BP
are more complicated than standard label propagation algorithms, 
have well-known convergence problems
~\cite[Sec.\ 22]{Murphy:2012uq}, 
and are difficult to use in practice~\cite{DBLP:journals/aim/SenNBGGE08}.
A number of recent papers found ways to
circumvent the convergence problems of BP by linearizing the update equations
\cite{DonohoMA:2009:PNAS, Eswaran:2017:ZBP:3055540.3055554, 
    DBLP:conf/aaai/Gatterbauer17, DBLP:journals/pvldb/GatterbauerGKF15,
DBLP:conf/pkdd/KoutraKKCPF11, Krzakala24122013},
and thus transforming the update equations of BP into an efficient matrix
formulation.
The resulting updates are similar to random walks 
but propagate messages ``modulated'' with \emph{relative class compatibilities}.

A big challenge for deploying this family of algorithms is knowing the appropriate compatibility matrix $\HVec$, 
where each entry $\H_{cij}$ captures the relative affinity between neighboring nodes of labels $i$ and $j$.
Finding appropriate compatibilities was identified as a challenging open problem~\cite{DBLP:conf/kdd/McGlohonBASF09},
and the current state of the art is to have them given by domain experts
or by ad-hoc and rarely justified heuristics.

\begin{figure}[t]
\centering
\includegraphics[scale=0.45]{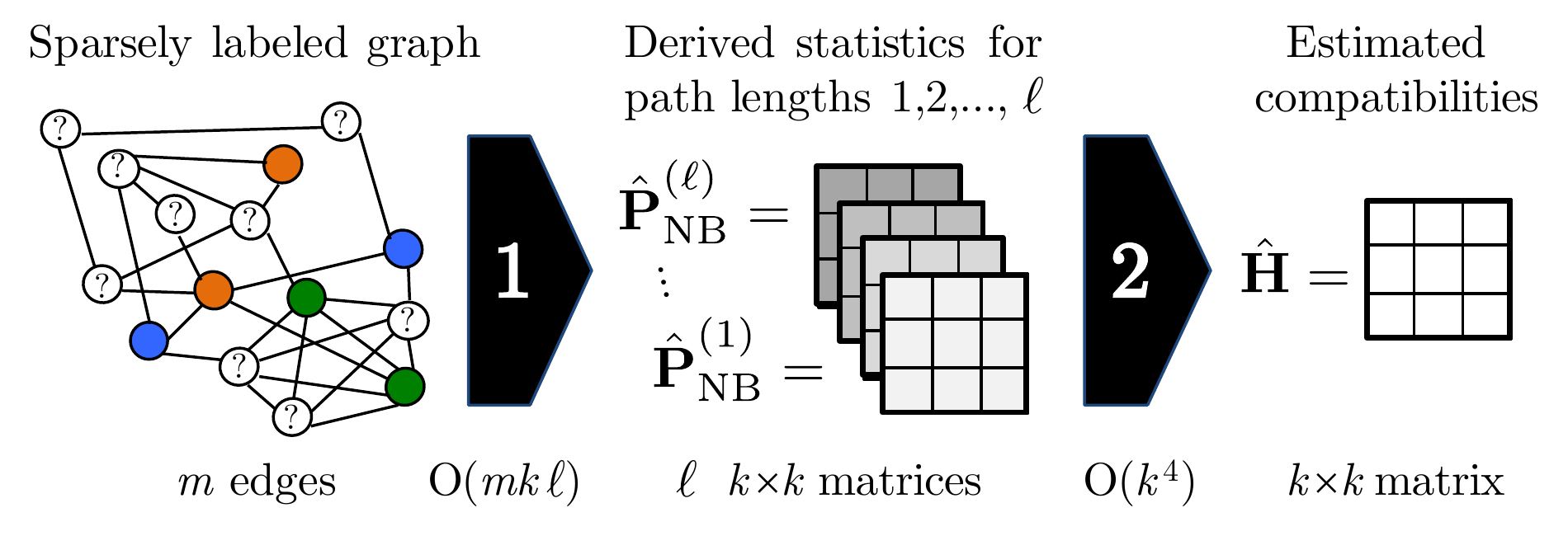}
\caption{Our approach for {compatibility estimation} proceeds in two steps:
(1) an efficient \emph{graph summarization} that creates 
sketches 
in linear time of the number of edges $m$ and
classes $k$
, see \cref{sec:efficientCalculation}; and 
(2) an \emph{optimization} step which is 
\emph{independent} 
of the size of the graph, 
see
\cref{sec:DCE}.
}
\label{Fig_TwoSteps}
\end{figure}

\introparagraph{Our contribution}
We propose an approach that \emph{does not need any prior domain knowledge of compatibilities.}
Instead, we estimate the compatibilities on the \emph{same} 
graph for which we later infer the labels of unlabeled nodes (\cref{Fig_IntroBlockModel_d}).
We achieve this by deriving an estimation method that
($i$) can handle extreme label scarcity,
($ii$)  
is orders of magnitude faster than textbook estimation methods,
and ($iii$) results in labeling accuracy that is
nearly indistinguishable from the actual gold standard (GS) compatibilities.
In other words, we suggest an end-to-end solution for a difficult \emph{within-network classification},
where compatibilities are not given to us:

\begin{problem}[Automatic Node Classification]\label{def:formalCEP}
    Given an undirected graph $G(V,$ $E)$ with a set of labeled nodes $V_\ell \subset V$
    from $k$ classes and unknown compatibilities between classes.
	Classify the remaining nodes, $v \in V \setminus V_l$.
\end{problem}

\begin{figure}[t]
\begin{subfigure}[t]{.50\linewidth}
	\includegraphics[scale=0.4]{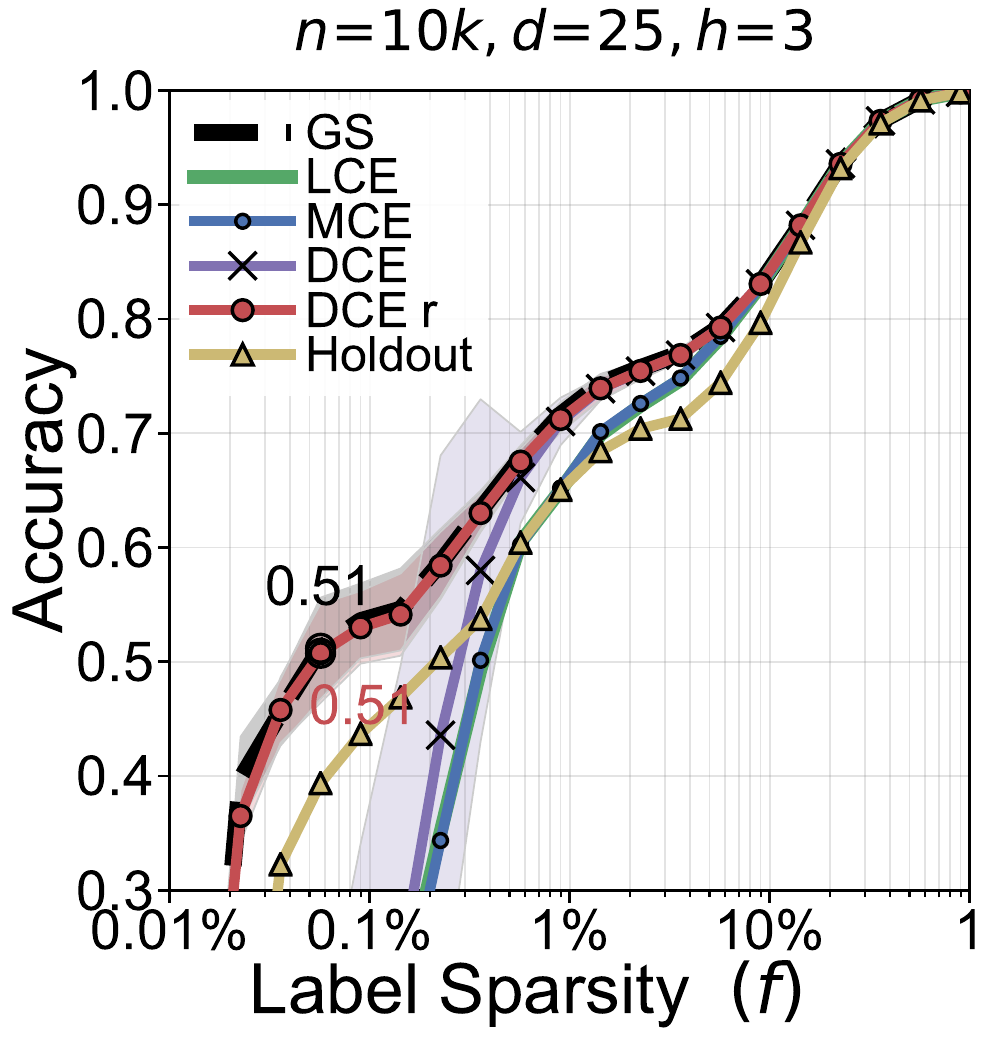}
	\caption{Estimation \& propagation}
	\label{Fig_End-to-End_accuracy108}	
\end{subfigure}
\hspace{0mm}
\begin{subfigure}[t]{.45\linewidth}
	\includegraphics[scale=0.39]{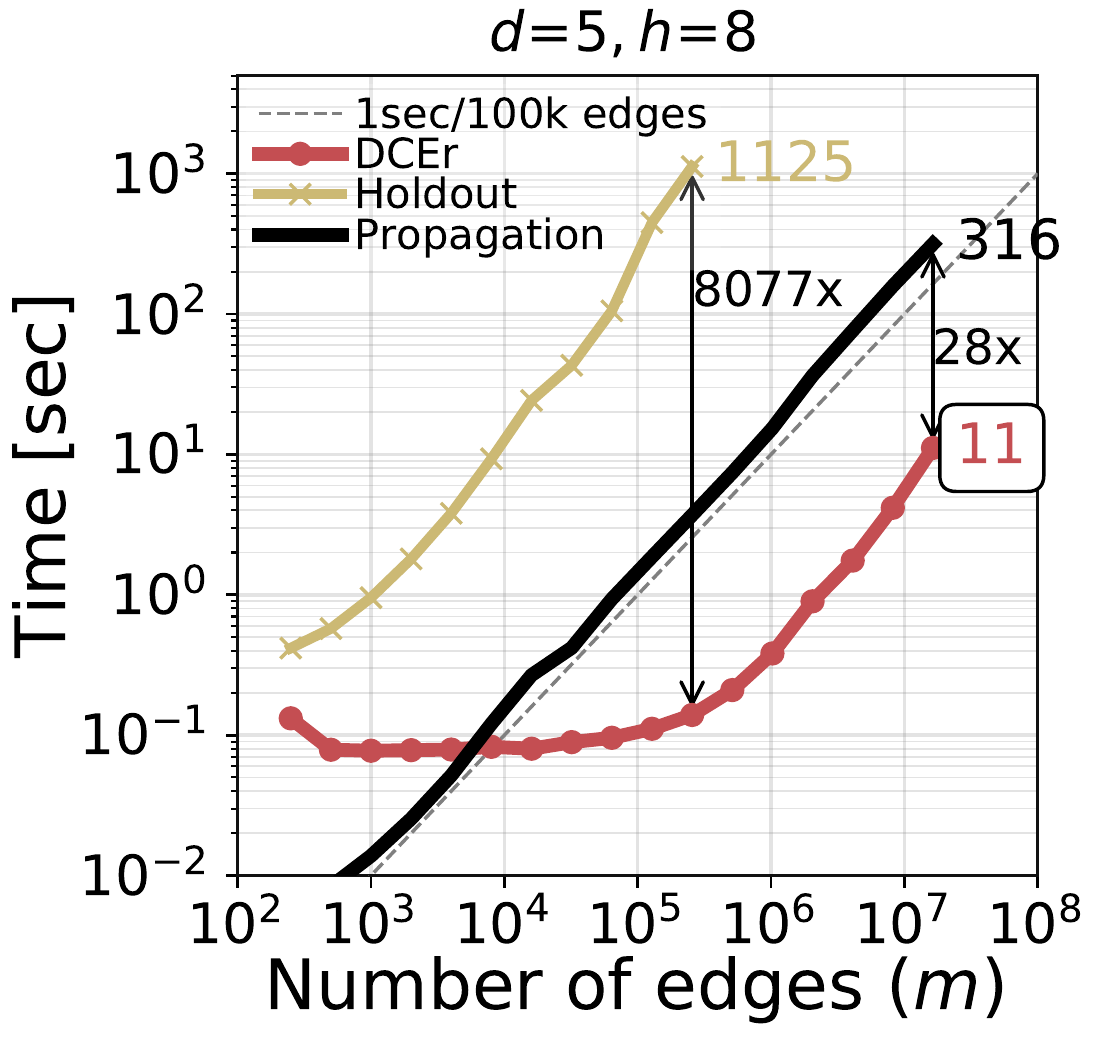}
	\caption{Scalability
	}\label{Fig_Timing_3_overview}
\end{subfigure}
\caption{
(a): 
Our methods infer labels with similar accuracy as if we were given the gold standard compatibilities (GS): e.g., labeling accuracy of 0.51 in a graph with 10k nodes 
and only 8 labeled nodes with our best method \emph{distance compatibility estimation with restarts} (DCEr) in red as compared to the same accuracy with GS.
(b): 
The additional step of estimating compatibilities is fast: 
DCEr 
learns the compatibilities on a graph with 16.4m edges in 11 sec, which is 28 times faster than node labeling (316 sec) and 3-4 orders of magnitude faster than a baseline holdout method.
}
\label{fig:overviewResults}
\end{figure}

\introparagraph{Summary of approach}
We develop a novel, consistent, and scalable graph summarization
that allows us to split compatibility estimation into two steps (\cref{Fig_TwoSteps}):
(1) First calculate
the number of paths of various lengths $\ell$ between nodes for all pairs of classes.
While the number of paths is exponential in the path's length, we develop 
efficient factorization and sparse linear algebra methods that calculate them 
in time linear of the graph size and path length.
\Cref{ex:efficientCalculation} illustrates evaluating $10^{14}$ such paths
in less than $0.1$ sec.
(2) Second use 
a combination of these compact graph statistics to estimate $\HVec$.
We derive
an explicit formula for the gradient of the loss function that allows us to
find the global optimum quickly.
Importantly, this second optimization step takes time \emph{independent
of the graph size} (!).
In other words, \emph{we reduce compatibility estimation over a sparsely labeled graph 
into an optimization problem 
over a set of small factorized graph representations}
with an explicit gradient.  
Our approach has only one relatively insensitive hyperparameter.

Our approach is orders of magnitude faster than 
common
parameter estimation methods
that rely on log-likelihood estimations and variants of expectation maximization.
For example, recent work~\cite{DBLP:dblp_conf/kdd/MooreYZRL11} develops methods that can learn compatibilities
on graphs with hundreds of nodes in minutes time.
In contrast, we learn compatibilities in graphs with $16.4$ million edges in $11$ sec
using an off-the-shelf optimizer
and running on a single CPU
(see \cref{Fig_Timing_3_overview}).
In a graph with $10k$ nodes and only $8$ labeled nodes,
we estimate $\HVec$ such that
the subsequent labeling has equivalent accuracy (0.51)
to a labeling using the
actual compatibilities 
(GS in \cref{Fig_End-to-End_accuracy108}).
We are not aware of any reasonably fast approach that can learn the compatibilities from the sparsely labeled graph.
All recent work in the area uses simple heuristics to specify the compatibilities:
e.g., \cite{DBLP:journals/pvldb/GatterbauerGKF15,DBLP:conf/aaai/Gatterbauer17,Eswaran:2017:ZBP:3055540.3055554,DBLP:conf/pkdd/KoutraKKCPF11}.

\introparagraph{Outline}
We start by giving a precise meaning to \emph{compatibility matrices} by showing that 
prior label propagation methods based on \emph{linearized belief propagation} essentially
propagate \emph{frequency distributions} of labels between neighbors (\cref{sec:propFrequency})
and deriving the corresponding energy minimization framework (\cref{sec:energyFunctionsForLinBP}).
Based on this formulation, we derive two convex optimization methods 
for parameter estimation (\cref{sec:learningHeterophily}): 
\emph{Linear compatibility estimation} (LCE) and \emph{myopic compatibility estimation}
(MCE).
We then develop a novel consistent estimator which counts 
``$\ell$-distance non-backtracking paths:''
\emph{distant compatibility estimation} (DCE).
Its objective function is not convex anymore, but well-behaved enough so we
can find the global optimum in practice with a few repeated restarts:
\emph{DCE with restarts} (DCEr).
\Cref{sec:experiments}
gives an extensive comparative study on synthetic and real-world data.

%% file: sections/formal_setup.tex
\section{Formal setup and related work}\label{sec:background}

We first define essential concepts and review 
related work on
semi-supervised node labeling.
We denote vectors ($\eVec$) and matrices ($\EVec$) in bold. 
We use row-wise ($\EVec_{i:}$), column-wise ($\EVec_{:j}$), and element-wise ($\E_{ij}$) matrix indexing,
e.g., $\EVec_{i:}$ is the $i$-th row vector of $\EVec$ (and thus bold), whereas $\E_{ij}$ is a single number (and thus not bold).
\subsection{Semi-Supervised Learning (SSL)}\label{sec:sslh}

Traditional graph-based Semi-Supervised Learning (SSL)
predict the labels of unlabeled nodes under the assumption of \emph{homophily} or \emph{smoothness}.
Intuitively, a label distribution is ``smooth''
if a label ``x'' on a node makes the same label 
on a neighboring node more likely,
i.e.\ nodes of the same class tend to link to each other.
The various methods differ mainly in their definitions of ``smoothness'' between classes 
of neighboring nodes
\cite{BengioDL:2006:LabelPropagation, 
DBLP:conf/icml/LuG03,
SubramanyaTalukdar:2014:SSL,
DBLP:journals/jmlr/WuS07, 
Zhu05,
ZhuKLG:2006:GraphKernels}.\footnote{Notice a possible naming ambiguity: ``\emph{learning}'' in SSL stands for \emph{classifying} unlabeled nodes (usually assuming homophily).
In our setup, we first need to ``\emph{learn}'' (or \emph{estimate}) the compatibility parameters, before we can classify the remaining nodes with a variant of label propagation.}

Common to all approaches, we are given a graph
${G} = ({V},{E})$ with $n = |{V}|$, $m = |{E}|$, and real edge
weights given by $w : {E} \rightarrow \R$.  
The weight $w(e)$ of an edge $e$
indicates the similarity of the incident nodes, and a missing edge corresponds
to zero similarity. 
These weights are captured in the symmetric weighted {adjacency matrix}
$\WVec \in \R^{n \times n}$  defined by
$\W_{ij} \define w(e)$ if $e = (i, j) \in {E}$, and 0 otherwise.
Each node is a member of exactly one of $k$ classes
which have increased edge incidence between members of the same class.
Given a set of labeled nodes $V_L \subset V$ with labels in $[k]$, predict the labels of the remaining
unlabeled nodes $V \setminus V_L$.

Most binary SSL algorithms \cite{10.1109/TKDE.2007.190672,
DBLP:conf/nips/ZhouBLWS03, DBLP:conf/icml/ZhuGL03} specify the existing labels
by a vector $\eVec = [\e_1, \ldots, \e_n]^\transpose$ with $\e_i \in L = \{+1,
-1 \}$ for $i \leq n_L$ and $\e_i=0$ for $n_L+1 \leq i \leq n$.
Then a real-valued 
``{labeling function}'' assigns a value $f_i$ with $1 \leq i \leq n$
to each data point $i$. The final classification is performed as
$\sql{sign}(f_i)$ for all unlabeled nodes.
This binary approach can be extended to multi-class classification~\cite{10.1109/TKDE.2007.190672}
by assigning a vector to each node.
Each entry represents the belief that a node is in the corresponding class.
Each of the classes is propagated \emph{separately} and, at convergence,
compared at each node with a ``one-versus-all'' approach~\cite{Bishop:2006ee}.
SSL methods differ in how they compute $f_i$ for each node $i$ and commonly
justify their formalism from a ``regularization framework''; i.e., by motivating
a different energy function
and proving that the derived labeling function $f$ is the solution to the
objective of minimizing the energy function.

\introparagraph{Contrast to our work}
The labeling problem we are interested
in this work is a generalization of standard SSL.
In contrast to the commonly used smoothness assumption (i.e.~labels of the same class tend to connect more often),
we are interested in the more general scenario of arbitrary compatibilities between classes.

\subsection{Belief Propagation (BP)}\label{sec:bp}

{Belief Propagation} (BP) \cite{DBLP:journals/aim/SenNBGGE08} 
is a widely used method for 
reasoning in networked data.
In contrast to typical semi-supervised label propagation, BP handles the case of arbitrary compatibilities.
By using the symbol $\odot$ for the 
component-wise multiplication
and writing $\mVec_{ji}$ for the $k$-dimensional ``{message}'' that node $j$ sends to node $i$, 
the BP update equations \cite{Murphy:2012uq, DBLP:journals/neco/Weiss00}
can be written
as:
\begin{align*}
\bVec_i \leftarrow Z_i^{-1} 
		\eVec_i \odot \! \bigodot_{j \in N(i)} \mVec_{ji}
\hspace{10mm}
\mVec_{ij} 	\leftarrow  
		\HVec \Big(\eVec_i \odot \!\!\!\bigodot_{v \in N(i) \setminus j}\!\! \mVec_{v i} \Big) 	
\end{align*}
Here, $Z_i$ is a normalizer that makes the elements of $\bVec_i$ sum to 1, and
each entry $\H_{ce}$ in $\HVec$ is a proportional ``compatibility'' that indicates the relative
influence of a node of class $c$ on its neighbor of class $e$.  Thus, an
outgoing message from a node is computed by multiplying all incoming messages 
(except the one sent previously by the recipient) and then multiplying the outgoing
message by the edge potential $\HVec$. 

Unlike other SSL methods, BP has no simple linear algebra formulation and has
well-known convergence problems.  Despite extensive research on the convergence
of BP \cite{DBLP:conf/uai/ElidanMK06,DBLP:journals/tit/MooijK07} exact criteria
for convergence are not known~\cite[Sec.\ 22]{Murphy:2012uq} and practical use
of BP is non-trivial~\cite{DBLP:journals/aim/SenNBGGE08}.

\introparagraph{Contrast to our work}
Parameter estimation in graphical models quickly becomes intractable for even
moderately-sized datasets~\cite{DBLP:dblp_conf/kdd/MooreYZRL11}.
We transform the original problem into a linear algebra formulation that allows us to leverage existing highly optimized tools 
and that can learn compatibilities often \emph{faster} than the time needed to label the graph.

\subsection{Linearized Belief Propagation}\label{sec:LinBPexplained}
Recent work
\cite{DBLP:conf/pkdd/KoutraKKCPF11,DBLP:journals/pvldb/GatterbauerGKF15} 
suggested to ``linearize'' BP 
and showed that 
the original update equations of BP can be reasonably approximated by linearized equations
\begin{align*}
\bCenterVec_i
	\leftarrow  \eCenterVec_i + \frac{1}{k} \cdot \!\! \sum_{j \in N(i)} \mCenterVec_{ji}
\hspace{10mm}
\mCenterVec_{ij} 
	\leftarrow  
		\HCenterVec 
		\Big(\bCenterVec_i \underbracket[0.5pt]{- \frac{1}{k} \mCenterVec_{ji}}_{\textrm{EC}} \Big)
\end{align*}
by ``centering'' the belief vectors $\eVec$, $\bVec$ and the potential matrix around $\frac{1}{k}$.
If a vector $\eVec$ is centered around $c$, 
then the residual vector around $c$ is defined as $\eCenterVec = [ \e_1 - c, \e_2 - c , \ldots ]$ and centered around 0.
This centering allowed the authors to rewrite BP in terms of the residuals.
The ``echo cancellation'' (EC) term is a result of the condition ``$v \in N(i) \setminus j$'' in the original BP equations. 

While the EC term has a strong theoretical justification for BP
and
appears to have been kept
for the correspondence between BP and LinBP,
in our extensive simulations, we have not identified any parameter regime 
where including the EC term for propagation consistently gives better results.
It rather slows down evaluation and complicates determining the convergence threshold
(the top eigenvalue becomes negative slightly above the convergence threshold).
We will thus explicitly ignore the EC term
in the remainder of this paper.
The update equations of LinBP then become:
\begin{align}		
\BCenterVec
	& \leftarrow 
		\ECenterVec 
		+ \WVec \BCenterVec \HCenterVec
	\hspace{10mm}
	(\mathrm{LinBP})
	\hspace{-15mm}			
	\label{eq:updateEquation} 
\end{align}
The advantage of LinBP over standard BP is that the linearized formulation allows provable convergence guarantees.
The process was shown to converge iff the following condition holds on the spectral radii\footnote{The spectral radius of a matrix is the largest absolute value among its eigenvalues.} $\rho$ 
of $\HCenterVec$ and $\WVec$:
\begin{align}
	\rho
	\big(
	\HCenterVec
	\big)
		 <1/\rho\big(\WVec)
	\label{eq:guaranteeLinBP}
\end{align}
Follow-up work \cite{DBLP:conf/aaai/Gatterbauer17} generalizes LinBP to the most general case of arbitrary pairwise Markov networks 
which include heterogeneous graphs with fixed number of node and edge types.
Independently, ZooBP~\cite{Eswaran:2017:ZBP:3055540.3055554} follows a similar motivation, 
yet restricts itself to the mathematically less challenging special case of constant row-sum symmetric potentials.

\introparagraph{Contrast to our work}
Our work focuses on homogeneous graphs and makes a complementary contribution to that of label propagation:
that of learning compatibilities from a sparsely labeled graph \emph{in a fraction of the time} it takes to propagate the labels
(\cref{sec:learningHeterophily}).
This avoids the reliance on domain experts or heuristics and results in an end-to-end
estimation and propagation method. 
An earlier version of the ideas in our paper was made available on arXiv as \cite{SSLH:Gatterbauer:arxiv}.

\subsection{Iterative Classification Methods}
\label{sec:IterativeClassification}

\introparagraph{Random walks with Restarts (RWR)}
Random walk-based methods make the assumption that the graph is homophilous; i.e., that instances belonging to the same class tend to link to each other or have higher edge weight between them~\cite{DBLP:conf/asunam/LinC10}.
 In general, given a graph $G = (V, E)$, random walk algorithms return as output a ranking vector $\bVec$ that results from iterating following equation until convergence:
 \begin{align}
 	\bVec \leftarrow \bar \alpha \vec u + \alpha \WVec^{\textrm{col}} \bVec
 	\label{eq:RWR}
 \end{align}
 Here, $\vec u$ is a normalized teleportation vector with
 $|\vec u| = |V|$ and $||\vec u||_1 = 1$,
 and $\WVec^{\textrm{col}}$ is column-normalized.
  Notice that above \cref{eq:RWR} can be interpreted as the probability of a random walk on $G$ arriving at node $i$,  with teleportation probability $\bar \alpha$ at every step to a node with distribution $\vec u$~\cite{DBLP:conf/asunam/LinC10}.
 Variants of this formulation are used by
 PageRank~\cite{PageRank1999},
 Personalized PageRank~\cite{haveliwala2003analytical, DBLP:conf/www/Chakrabarti07},
 Topic-sensitive PageRank~\cite{haveliwala2003topic},
 Random Walks with Restarts~\cite{Pan:2004:AMC:1014052.1014135},
 and MultiRankWalk~\cite{DBLP:conf/asunam/LinC10} which runs $k$ random walks in parallel (one for each class $c$).

To compare it with our setting, MultiRankWalk~\cite{DBLP:conf/asunam/LinC10}  and other forms of random walks can be stated as special cases of the more general formulation:
 (1) For each class $c \in [k]$:
 (a) set $\vec u_i \leftarrow 1$ if node $i$ is labeled $c$,
 (b) normalize $\vec u$ s.t.\ $||\vec u ||_1 = 1$.
 (2) Let $\vec U$ be the $n \times k$ matrix with column $i$ equal $\vec u_i$.
 (3) Then iterate until convergence:
 \begin{align*}
 	\BVec \leftarrow \bar \alpha \vec U + \alpha \WVec^{\textrm{col}}  \BVec \vec I_k
 \end{align*}
 (4) After convergence, label each node $i$ with the class $c$ with maximum value:
 $c = \arg \max_j\B_{ij}$.

\introparagraph{Other Iterative Classification Methods}
Goldberg et al.\ \cite{DBLP:journals/jmlr/GoldbergZW07} consider a concept of similarity and dissimilarity between nodes. 
This method only applies to classification tasks with 2 labels and cannot generalize to arbitrary compatibilities.
Bhagat et al.\ \cite{Bhagat2007}  look at commonalities across the direct neighbors of nodes in order to classify them. The paper calls this method leveraging ``co-citation regularity''  which is indeed equally expressive as heterophily. 
The experiments in that paper require at least 2\% labeled data (Figure 6e in \cite{Bhagat2007}), 
which is similar to the regimes up to which MCE works. 
Similarly, Peel~\cite{DBLP:conf/sdm/Peel17} suggests an interesting method that skips compatibility matrices by propagating information across nodes with common neighbors.
The method was tested on networks with 10\% labeled nodes and it will be interesting to investigate its performance in the sparse label regime.

\subsection{Recent neural network approaches}

Several recent papers propose neural network (NN) architectures for node labeling
(e.g., \cite{DBLP:conf/iclr/KipfW17,Hamilton:2017:IRL:3294771.3294869,DBLP:conf/aaai/MooreN17}).  
In contrast to our work (and all other work discussed in this section), 
those NN-based approaches
\emph{require additional features} from the nodes.
For example, in the case of Cora, \cite{DBLP:conf/iclr/KipfW17} also has access to node content 
(i.e.\ which words co-occur in a paper).
Having access to the actual text of a paper allows better classification than the network structure alone. 
As a result, \cite{DBLP:conf/iclr/KipfW17} can learn and use a large number of parameters in their trained NN. 

\introparagraph{Constrast to our work}
We classify the nodes based \emph{on the graph structure alone}, without access to additional features.
The result is that while \cite{DBLP:conf/iclr/KipfW17} achieves an accuracy of $81.5\%$ for $5.2\%$ labeled nodes in Cora
(see Section 5.1 and Section 6.1 of \cite{DBLP:conf/iclr/KipfW17}),
we still achieve $66\%$ accuracy based on the network alone and only $21$ estimated parameters.

\subsection{Non-backtracking paths (NB)}

\Cref{sec:NonBacktrackingPaths} derives estimators 
for the powers of 
$\HVec$
by counting labels over all \sloppy{``non-backtracking''} (NB)
paths in a partially labeled graph.  We prove our estimator to be
consistent
and thus with negligible bias for increasing
$n$.
Prior work already points to the advantages of NB paths for various different graph-related problems, 
such as graph sampling~\cite{Lee:2012:BRW:2254756.2254795}),
calculating eigenvector centrality~\cite{PhysRevE.90.052808}, 
increasing the detectability threshold for community detection~\cite{Krzakala24122013}, 
improving estimation of graphlet statistics~\cite{Chen:2016:GFE:3021924.3021940},
or measuring the distance between graphs~\cite{DBLP:journals/ans/TorresSE19}.
To make this work, all these papers replace the $n \times n$ 
adjacency matrix with a $2m \times 2m$ ``Hashimoto matrix''~\cite{hashimoto1989zeta} which represents the link structure of a graph 
in an augmented state space with $2m$ states (one state for each directed pair of nodes) and in the order of $O(m(d-1))$ non-zero entries, 
and then perform random walks.
The only work we know that uses NB paths without Hashimoto is \cite{AlonBLS:2007:Non-backtracking}, which
calculates the mixing rate of a NB random walk on a regular expanders (thus graphs with identical degree across all nodes).
That work does not generalize to graphs with varying degree distribution and does not allow an efficient path summarization.

\introparagraph{Contrast to our work}
Our approach does not perform random walks,
does not require an augmented state space (see \cref{prop:nonbacktracking}),
and still allows an efficient path summarization (see \cref{prop:complexitystatistics}).
To the best of our knowledge, ours is the first proposal to 
($i$) estimate compatibilities from NB paths
and ($ii$) propose an efficient calculation.

\subsection{Distant supervision}

The idea of \emph{distant supervision} is to
adapt existing ground truth data 
from a related yet different task
for providing \emph{additional lower quality labels} (also called \emph{weak labels}) to sparsely labeled data
\cite{DBLP:conf/acl/MintzBSJ09,DBLP:conf/acl/HoffmannZLZW11,DBLP:journals/pvldb/RatnerBEFWR17}.
The methods are thus also often referred to as \emph{weak supervision}.

\introparagraph{Contrast to our work}
In our setting, we are given no other outside ground truth data nor heuristic rules to label more data.
Instead, we leverage certain algebraic properties of an algorithm’s update equations to \emph{amplify sparse signals in the available data}.
We thus refer to the more general idea of our approach as \emph{algebraic amplification}.

%% file: sections/label_propagation.tex
\section{\mbox{Properties of Label Propagation}}\label{SEC:PROPERTIES}
This section makes novel observations about linearized versions of BP
that help us later find efficient ways to learn the compatibility matrix $\HVec$ from sparsely labeled graphs.

\subsection{Propagating Frequency Distributions}\label{sec:propFrequency}

Our first observation is that centering of prior beliefs $\EVec$ and 
compatibility matrix $\HVec$ in LinBP~\cref{eq:updateEquation} 
is \emph{not necessary}
and that the final labels are identical 
whether we use $\ECenterVec$ or $\EVec$, and  $\HCenterVec$ or  $\HVec$.
We state this result in a slightly more general form:
Let $\BVec = \textrm{LinBP}(\WVec, \EVec, \HVec, \epsilon, r)$ 
stand for the label distribution after iterating the LinBP update equations
$r$ times,
starting from $\EVec$ and using scaling factor $\epsilon$.
Let $\vec l = \textrm{label}(\BVec)$ stand for the operation of assigning each node the class with the maximum belief:
$l_i = \arg \max_{j} \B_{ij}$. Then:
\begin{theorem}[Centering in LinBP is unnecessary]\label{prop:nonCenteredLinBP}
Given constants $c_1$ and $c_2$ s.t.\
$\HVec_2 = \HVec_1 + c_1$
and $\EVec_2 = \EVec_1 + c_2$.\footnote{We use here ``broadcasting notation:'' 
adding a number to a vector or matrix is a short notation for adding the number to each entry in the vector.}
Then,
$\forall \WVec, \epsilon, r$:
$
\textup{label} \big( \textup{LinBP}(\WVec, \EVec_2, \HVec_2, \epsilon, r)\big)
=$
$
\textup{label} \big( \textup{LinBP}(\WVec, \EVec_1, \HVec_1, \epsilon, r)\big)
$.
\end{theorem}

Modulating beliefs of a node with $\HVec$ instead of $\HCenterVec$ allows a natural interpretation of
label propagation as ``propagating frequency distributions'' and thus
imposing an \emph{expected frequency distribution} on the labels of neighbors of a node. 
This observation gives us an intuitive interpretation of our later derived approaches for learning $\HVec$ from observed frequency distributions (\cref{sec:MCE}).
For the rest of this paper, we will thus replace \cref{eq:updateEquation} with the ``uncentered'' version:
\begin{align}
\BVec
	\leftarrow 
		\EVec 
		+ \WVec \BVec \HVec
\end{align}

A consequence is that compatibility propagation works identically 
whether the compatibility matrix $\HVec$ is centered or kept as doubly-stochastic.
In other words, if the relative frequencies by which different node classes connect to each other is known, then this matrix can be used \emph{without centering} for compatibility propagation and will lead to identical results and thus node labels.
\subsection{Labeling as energy minimization}\label{sec:energyFunctionsForLinBP}

Our next goal is to formulate the solution to the update equations of LinBP as the solution to an 
{optimization problem}; i.e., as an \emph{energy minimization framework}.
While LinBP was derived from probabilistic principles
(as approximation of the update equations of belief propagation~\cite{DBLP:journals/pvldb/GatterbauerGKF15}), 
it is currently not known whether there is a simple objective function that a solution minimizes.
Knowledge of such an objective is helpful
as it allows principled extensions to the core algorithm.
We will next give the objective function for LinBP and 
will use it later in \cref{sec:learningHeterophily} 
to solve the problem of {parameter learning}; 
i.e., estimating the compatibility matrix from a partially labeled graph.

\begin{proposition}[LinBP objective function]\label{prop:LinBP_energy}
The energy function minimized by the LinBP update equations \cref{eq:updateEquation}  is
given by:
\begin{align}
	E(\BVec) = || \BVec - \EVec - \WVec \BVec \HVec 
		||^2		
	\label{eq:LinBP_lossfunction}
\end{align}		
\end{proposition}

%% file: sections/compatibility_estimation.tex
\section{Compatibility Estimation}\label{sec:learningHeterophily}

In this section we develop
a scalable algorithm to learn compatibilities from
partially labeled graph.  
We proceed step-by-step, starting from a baseline until we finally arrive at our suggested consistent and scalable method called ``\emph{Distant Compatibility Estimation with restarts}'' (DCEr).

The compatibility matrix we wish to estimate is a $k \times k$-dimensional doubly stochastic
matrix $\HVec$.  Because any symmetric doubly-stochastic matrix has $k^* \define
\frac{k(k-1)}{2}$ degrees of freedom, we parameterize all $k^2$
entries as a function of $k^*$ appropriately chosen parameters.
In all following approaches, we parameterize $\HVec$ as 
a function of the $k^*$ entries of $H_{ij}$ with $i \leq
j, j\neq k$. We can calculate the remaining matrix entries from symmetry and stochasticity conditions as follows:
\begin{align} H_{ij} = 
    \begin{cases}
        H_{ji}, & \text{if }  i < j, j\neq k  \\ 1-\sum_{\ell=1}^{k-1} H_{i \ell}, &
        \text{if } i \neq k, j = k \\ 1-\sum_{\ell=1}^{k-1} H_{\ell j}, & \text{if }
        i = k, j \neq k \\	2-k+\sum_{\ell, r < k} H_{\ell r}, & \text{if } i = j =
        k \\		
	\end{cases}
	\label{eq:H_parameterization}
\end{align}

For example, for $k=3$, $\HVec$ can be reconstructed from a $k^*=3$-dimensional vector
$\hVec = [H_{11}, H_{21}, H_{22}]^\transpose$
 as follows:
\begin{align*}
	\HVec(\hVec) = \left[\begin{smallmatrix}
		H_{11} 			&\, H_{12} 			&\, 1-H_{11}-H_{12}  \\
		H_{21} 			&\, H_{22} 			&\, 1-H_{21}-H_{22}  \\
		1-H_{11}-H_{21} \,&\, 1-H_{12}-H_{22} \,&\, H_{11}+2 H_{21}+ H_{22}-1  \\
	\end{smallmatrix}\right]
\end{align*}

More generally, let $\hVec \in \R^{k^*}$ and define $\HVec$ as function of the $k^* \define \frac{k(k-1)}{2}$ entries of $\hVec$ as follows:
\begin{align*}
	\HVec = \left[\begin{smallmatrix}
		\h_1 & . 	& . 	& \ldots & . & .  \\
		\h_2 & \h_3 & . 	& \ldots & . & .     \\
		\h_4 & \h_5 & \h_6 	& \ldots & . & . 	\\
		\tiny{^{\vdots}} & \tiny{^{\vdots}} & \tiny{^{\vdots}} 	& \tiny{^{\ddots}} & . & . 	\\
		\h_{...} 	& \h_{...} 	& \h_{...} 	& \ldots & \h_{k^*} & .  \\
		. 	& . 	& . 	& \ldots & . & .  \\
	\end{smallmatrix}\right]		
\end{align*}
The remaining matrix entries can be calculated 
from \cref{eq:H_parameterization}.

\subsection{Baseline: Holdout method}\label{sec:baseline}

Our first approach for estimating $\HVec$ 
is a variant of a standard textbook method~\cite{Mohri:2012kq,Witten:2011la,Koller:2009:PGM:1795555}
and serves as baseline against which we compare all later approaches:
we split the labeled data into two sets and learn the compatibilities
that fit best when propagating labels from one set to the other.

Formally, let $\mathcal{Q}$ be a partition of the available labels into a 
$\texttt{Seed}$  
and a
$\texttt{Holdout}$ set.
For a fixed partition $\mathcal{Q}$ and given compatibility matrix $\HVec$, 
the ``holdout method'' runs label propagation \cref{eq:updateEquation}
with $\texttt{Seed}$ as seed labels
and evaluates accuracy over $\texttt{Holdout}$. Denote
$\textrm{Acc}_{\mathcal{Q}}(\HVec)$ the resulting accuracy.
Its goal is then to find the 
matrix $\HVec$ 
that maximizes the accuracy.
In other words, the energy function that holdout minimizes is the negative accuracy:
\begin{align}
	E(\HVec )
	=
	-\textrm{Acc}_{\mathcal{Q}}(\HVec)
	 \notag
\end{align}
The optimization itself is then a search over the parameter space given by the $k^*$ free parameters of $\HVec$:
\begin{align}
	\HVecEst = \arg \min_{\HVec} E(\HVec)
	\textrm{, s.t.\ \cref{eq:H_parameterization}}	 
	\notag
\end{align}

The result may depend on the choice of partition $\mathcal{Q}$.
We could thus use $b$ different partitions $\mathcal{Q}_i, i \in [b]$:
For a fixed $\HVec$ we run label propagation $b$ times, 
each starting from a different $\texttt{Seed}_i$,
and each evaluated over its corresponding test set $\texttt{Holdout}_i$. 
The energy function to minimize is then the negative compound accuracy:
\begin{align}
	E(\HVec )
	=
   -
   \sum_i \textrm{Acc}_{\mathcal{Q}_i}(\HVec)
	\hspace{10mm} \textrm{(Holdout)} \hspace{-10mm}
	\label{eq:HeterophilyLearningEnergy}
\end{align}

We suggest this method as reasonable baseline as 
it mimics parameter estimation methods in probabilistic graphical models 
that optimize over a parameter space by \emph{using multiple executions of inference as a subroutine}~\cite{Koller:2009:PGM:1795555}.
Similarly, our holdout method maximizes the accuracy by using inference as a ``black box'' subroutine. 
The downside of the holdout method is that 
each step in this iterative algorithm performs inference 
over the whole graph which makes parameter \emph{estimation 
considerably more expensive than inference} (label propagation).
The number of splits $b$ has an obvious trade-off: higher $b$ smoothens the
energy function and avoids overfitting to one partition, but increases runtime. 

In the following sections, we introduce novel path summarizations that avoid
running estimation over the whole graph.
Instead we use a few concise graph
summaries
of size $O(k^2)$, \emph{independent} of the graph size.  In other words, the
expensive iterative estimation steps can now be performed on a reduced size
summary of the partially labeled graph.  This conceptually simple idea allows us
to perform \emph{estimation faster than inference} (recall
\cref{Fig_Timing_3_overview}).
\subsection{Linear Compatibility Estimation (LCE)}\label{sec:LCE}

We obtain our first novel 
approach from energy minimization objective of LinBP 
in \cref{prop:LinBP_energy}:
\begin{align*}
	E(\BVec) = || \BVec - \EVec - \WVec \BVec \HVec 
	||^2		
\end{align*}		

Note that 
for an unlabeled node $i$, the final label distribution is the weighted  average of its neighbors:
$\BVec_{i:} = ( \WVec \BVec \HVec)_{i:}$.
To see this, consider a single row for a node $i$: 
\begin{align*}
	|| \big( \BVec - \EVec - \WVec \BVec \HVec \big)_{i:} ||^2
\end{align*}
If $i$ is unlabeled then its corresponding entries in $\EVec_{i:}$ are 0, and the minimization objective is equivalent to
\begin{align*}
	|| \big( \BVec - \WVec \BVec \HVec \big)_{i:} ||^2
\end{align*}
which leads to $\BVec_{i:} = ( \WVec \BVec \HVec)_{i:}$ for an unlabeled node.
Next notice that if we knew $\BVec$ and ignored the few explicit labels, then $\HVec$ could be learned from minimizing 
\begin{align*}
	E(\HVec) = || \BVec - \WVec \BVec \HVec 
	||^2		
\end{align*}	
In our case, we only have few labels in the form of $\EVec$ instead of $\BVec$.
Our first novel proposal for learning the compatibility matrix $\HVec$ is to thus use the available labels $\EVec$ and to minimize the following energy function:
\begin{align}
	E(\HVec )
	=
	|| {\EVec} 
	-  \WVec {\EVec} {\HVec} 
	||^2		
	\hspace{10mm} \textrm{(LCE)} \hspace{-15mm}
	\label{eq:HeterophilyLearningEnergy}
\end{align}

Notice that~\cref{eq:HeterophilyLearningEnergy} defines a \emph{convex} optimization problem.
Thus any standard optimizer can solve it in considerably faster time than the Holdout method and
it is no longer necessary to use inference as subroutine.
We call this approach ``linear compatibility estimation'' as the optimization criterion stems directly from the optimization objective of linearized belief propagation.

\subsection{Myopic Compatibility Estimation: MCE}\label{sec:MCE}

We next introduce a powerful yet simple idea that allows our next approaches to truly scale:
we first (1) summarize the partially labeled graph into a small summary, and then (2) use this summary to perform the optimization.
This idea was motivated by the observation that
\cref{eq:HeterophilyLearningEnergy}
requires an iterative
gradient descent algorithm and has to multiply large adjacency
matrix $\WVec$ in each iteration.
We try to derive an approach that can
``factor out'' this calculation into small but
sufficient \emph{factorized graph representation}, which can then be repeatedly used during optimization.

Our first method is called
\emph{myopic compatibility estimation}
(MCE).
It is ``myopic'' in the sense that it tries to summarize \emph{the relative frequencies
of classes between observed neighbors}.
We describe below the three variants to transform this summary into a symmetric,
doubly-stochastic matrix.

Consider a partially labeled $n \times k$-matrix $\EVec$ with $\E_{ic} = 1$. If
node $i$ has label $c$ (recall that unlabeled nodes have a corresponding null
row vector in $\EVec$), then the $n \times k$-matrix ${\vec N} \define \WVec
\EVec$ has entries $N_{ic}$ representing the {number of labeled neighbors
of node $i$ with label $c$}.
Furthermore, the $k \times k$-matrix $\vec M \define \EVec^\transpose \vec N =
\EVec^\transpose \WVec \EVec$ has entries $M_{cd}$ representing the {number
of nodes with label $c$ that are neighbors of nodes with label $d$}. This
symmetric matrix represents the observed number of labels among labeled nodes.
Intuitively, we are trying to find a compatibility matrix which
is ``similar'' to $\vec M$.
We normalize $\vec M$ into an observed \emph{neighbor statistics} matrix $\PVecEst$
and then find the closest doubly-stochastic matrix $\HVec$:

We consider three  variants for normalizing $\vec M$.
The first one appears most natural (creating a stochastic matrix representing label frequency distributions between neighbors, then finding the closest doubly-stochastic matrix).
We conceived of two other approaches, just to see if the choice of normalization has an impact on the final labeling accuracy.
\begin{enumerate}

\item
Make $\vec M$ row-stochastic by dividing each row by its sum.
The vector of row-sums can be expressed in matrix notation as $\vec M \vec 1$.
We thus define the first variant of neighbor statistics matrix as:
\begin{align}
	\hspace{0mm}
	\PVecEst
	&=
	\Vec M^{\textrm{row}}
	\define
	\mathrm{diag}(\vec M \vec 1)^{-1} \vec M
	\hspace{8mm} \textrm{(Variant 1)}
	\label{eq:variant1}
\end{align}
Note, we use $\Vec M^{\textrm{row}}$ 
as short notation for
row-normalizing the matrix $\vec M$.
For each class $c$, the entry $\PEst_{ce}$ gives the relative frequency of a node being connected to class $e$
While the matrix is row-stochastic, it is not yet doubly-stochastic.

\item
The second variant uses the symmetric normalization method LGC~\cite{DBLP:conf/nips/ZhouBLWS03} from \cref{sec:background}:
\begin{align}
	\hspace{0mm}
	\PVecEst
	&=
	\mathrm{diag}(\vec M \vec 1)^{-\frac{1}{2}}
	\vec M
	\,
	\mathrm{diag}(\vec M \vec 1)^{-\frac{1}{2}}
	\hspace{4mm} \textrm{(Variant 2)}
\end{align}
The resulting $\PVecEst$ is symmetric but not stochastic.

\item The third variant scaled $\vec M$ s.t.\ the average matrix
    entry is $\frac{1}{k}$.  This divisor is the sum of all entries divided by
    $k$ (in
    vector notation written as $\vec 1^\transpose \vec M \vec 1$) :
\begin{align}
	\hspace{7mm}
	\PVecEst
	&=
	k({\vec 1}^\transpose \vec M \vec 1)^{-1}
	\vec M
	\hspace{13mm} \textrm{(Variant 3)}
	\hspace{-3mm}
\end{align}
This scaled matrix is neither symmetric nor stochastic.

\end{enumerate}

We then find the ``closest''
symmetric, doubly stochastic matrix $\HVec$
(i.e., it fulfills the $k^* \define \frac{k(k-1)}{2}$ 
conditions implied by symmetry $\HVec = \HVec^\transpose$
and stochasticity $\HVec \, \vec 1 = \vec 1$).
We use the Frobenius norm because of its favorable computational properties
and thus minimize the following function:
\begin{align}
	\hspace{10mm}
	E(\HVec )
	=
	|| \HVec - \PVecEst ||^2
		\hspace{14mm}
		\textrm{(MCE)}
		\hspace{-2mm}
\end{align}

Notice that all three normalization variants above have an alternative, simple justification:
on a {fully labeled graph}, each variant will learn the same compatibility matrix;
i.e., the matrix that captures the relative label frequencies between neighbors in a graph.
On a graph with sampled nodes, however, $\vec M$ will not necessarily be constant row-sum anymore.
The three normalizations and the subsequent optimization are alternative approaches for
finding a ``smoothened'' matrix $\HVec$ that is close to the observations.
Our experiments have shown that the ``most natural'' normalization variant 1 consistently performs best among the three methods.
Unless otherwise stated, we thus
imply using variant 1.

Notice that MCE and all following approaches estimate $\HVec$ \emph{without performing label propagation}; and only because we avoid propagation, our method turns out to be faster than label propagation on large graphs and moderate $k$.

\subsection{Distant Compatibility Estimation: DCE}\label{sec:DCE}

While MCE addresses the scalability issue,
it still requires a sufficient number of neighbors that are both labeled.
For very small fractions $f$ of labeled nodes, this may not be enough.
Our next method, ``\emph{distant compatibility estimation}'' (DCE),
takes into account
longer distances between labeled nodes.

In a graph with $m$ edges and a small fraction $f$ of labeled nodes, the number
of neighbors that are \textit{both} labeled can be quite small ($\sim mf^2$).
Yet the number of
``distance-2-neighbors'' (i.e., nodes which are connected via a path of length $2$) is higher in proportion to the average node degree $d$ ($\sim dmf^2$). Similarly for distance-$\ell$-neighbors ($\sim d^{\ell-1}mf^2$).
As information travels via a path of length $\ell$, it gets modulated $\ell$ times;
i.e., via a power of the compatibility matrix: $\HVec^{\ell}$.\footnote{Notice that this  is strictly correct only  in graphs with balanced labels. Our experiments verify the quality of estimation also on unbalanced graphs.}
We propose to use powers of the matrix $\HVec$ to be estimated by comparing them
against an ``\emph{observed length-$\ell$ statistics matrix}.''

Powers of the adjacency matrix $\WVec^\ell$ with entries $\W^\ell_{ij}$
count the number of paths of length $\ell$ between any nodes $i$ and $j$.
Extending the ideas in \cref{sec:MCE},
let ${\vec N}^{(\ell)} \define \WVec^\ell \EVec$
and ${\vec M}^{(\ell)} \define \EVec^\transpose  {\vec N}^{(\ell)} = \EVec^\transpose \WVec^\ell \EVec$.
Then entries $M^{(\ell)}_{ce}$ represent the number of labeled nodes of class $e$
that are connected to nodes of class $c$ by an $\ell$-distance path.
Normalize this matrix (in any of the previous 3 variants)
to get the observed length-$\ell$ statistics matrix
$\PVecEst^{(\ell)}$.
Calculate these length-$\ell$ statistics for several path lengths $\ell$,
and then find the compatibility matrix that best fits these multiple statistics.

Concretely, minimize a ``distance-smoothed'' energy
\begin{align}
E(\HVec) &=
	\sum_{\ell=1}^{\ell_{\max}} w_{\ell} || \HVec^\ell - \PVecEst^{(\ell)} ||^2
		\hspace{14mm}
		\textrm{(DCE)}
		\hspace{-2mm}
\label{eq:DCE_energy}
\end{align}
where
$\ell_{\max}$ is the maximal distance considered, and
the 
weights $w_{\ell}$ balance having more (but weaker) data points for
bigger $\ell$ 
the more reliable (but sparser) signal from smaller
$\ell$.

To parameterize the weight vector $\vec w$, we use a ``scaling factor'' $\lambda$ defined by
$w_{\ell+1} = \lambda w_\ell$. For example, a distance-3 weight vector is then
$[1, \lambda, \lambda^2]^\transpose$.
The intuition is that in a typical graph, the fraction of number of paths of length $\ell$
to the number of paths of length $\ell-1$ is \emph{largely independent} of $\ell$ (but proportional to the average degree).
Thus, $\lambda$ determines the relative weight of paths of one more hop.
As consequence our framework has only one single hyperparameter $\lambda$.

In our experiments (\cref{sec:experiments}), we initialize the optimization with
a $k^*$-dimensional vector with all entries equal to $\frac{1}{k}$ and
discuss our choice of $\ell_{\max}$ and $\lambda$.

\subsection{Non-Backtracking Paths (NB)}
\label{sec:NonBacktrackingPaths}
In our previous approach of learning from more distant neighbors,
we made a slight but consistent mistake. We illustrate this mistake with \cref{Fig_Why_NonBacktrackingMatrix}. Consider the blue node $i$ which has one orange neighbor $j$, which has two neighbors, one of which is green node $u$.
Then the blue node $i$ has one distance-2 neighbor $u$ that is green. However, our previous approach will consider all length-2 paths, one of which leads back to node $i$.
Thus, the row entry for node $i$ in $\vec N$ is $\vec N^{(2)}_{i:} = [1, 0, 1]$ (assuming blue, orange, and green represent classes 1, 2, and 3, respectively).
In other words, $\vec M^{(2)}$ will consistently \emph{overestimate the diagonal entries}.

To address this issue, we consider only \emph{non-backtracking paths} (NB)
in the powers of the adjacency matrix.
A NB path on an undirected graph $G$ is a path which does not traverse the same edge twice in a row.
In other words, a path $(u_1, u_2, \ldots,  u_{\ell + 1})$
of length $\ell$ is non-backtracking iff
$\forall j \leq \ell -1: u_j \neq u_{j+2}$.
In our notation, we replace
$\WVec^\ell$ with $\WVecEC^{(\ell)}$.
For example, $\WVecEC^{(2)} = \WVec^2 - \vec D$
(a node $i$ with degree $d_i$ has
$D_{ii} = d_i$ as diagonal entry in $\vec D$).
A more general calculation of $\WVecEC^{(\ell)}$ for any length $\ell$ is presented in \cref{sec:efficientCalculation}.
We now calculate new graph statistics $\PVecEstEC^{(\ell)}$ from
$\vec M_{\NB}^{(\ell)} \define \EVec^\transpose \WVecEC^{(\ell)} \EVec$
instead of $\vec M^{(\ell)}$,
and replace $\PVecEst^{(\ell)}$
with
$\PVecEstEC^{(\ell)}$ in
\cref{eq:DCE_energy}:
\begin{align}
E(\HVec) &=
	\sum_{\ell=1}^{\ell_{\max}} w_{\ell} || \HVec^\ell - \PVecEstEC^{(\ell)} ||^2
		\hspace{8mm}
		\textrm{(DCE NB)}
		\hspace{-2mm}
\label{eq:DCE_energyNB}
\end{align}

We next show that--assuming a label-balanced graph--this change gives us a \emph{consistent estimator} with bias in the order of $\O(1/m)$,
in contrast to the prior bias in the order of $\O(1/d)$:

\begin{figure}[t]
\centering
\begin{subfigure}[b]{.62\linewidth}
	\centering
	\includegraphics[scale=0.33]{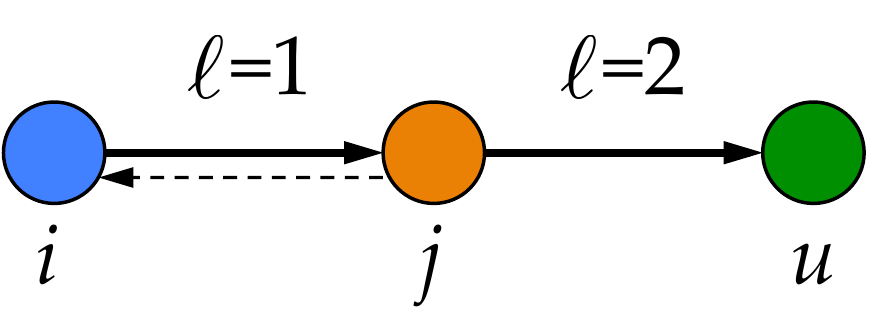}
\end{subfigure}
\caption{Illustration for non-backtracking paths}
\label{Fig_Why_NonBacktrackingMatrix}
\end{figure}

\begin{theorem}[Consistency of statistics $\PVecEstEC^{(\ell)}$]\label{th:consistency}
Under mild assumptions for the degree distributions,
$\PVecEstEC^{(\ell)}$ is a consistent estimator for $\HVec^\ell$,
whereas
$\PVecEst^{(\ell)}$ is not:
\begin{align*}
\lim_{n \rightarrow \infty} \PVecEstEC^{(\ell)} = \HVec^\ell
\hspace{7mm}
\textrm{whereas,}
\hspace{7mm}
\lim_{n \rightarrow \infty} \PVecEst^{(\ell)} \neq \HVec^\ell
\end{align*}
\end{theorem}

\begin{example}[Non-backtracking paths]
\label{ex:non-backtracking}
	Consider the compatibility matrix
	$\HVec =
	\left[\begin{smallmatrix}
		0.2 & 0.6 & 0.2  \\
		0.6 & 0.2 & 0.2  \\
		0.2 & 0.2 & 0.6  \\
	\end{smallmatrix}\right]$.
	Then
	$\HVec^2 =
	\left[\begin{smallmatrix}
		0.44 & 0.28 & 0.28  \\
		0.28 & 0.44 & 0.28  \\
		0.28 & 0.28 & 0.44  \\
	\end{smallmatrix}\right]$,
	and the maximum entry (permuting between first and second position in the first row)
	follows the series
	$0.6, 0.44, 0.376, 0.3504, \ldots$ for increasing $\ell$
	(shown as continuous green line $\HVec^{\ell}$ in \cref{Fig_Backtracking_Advantage_21}).
	We create synthetic graphs with $n = 10$k nodes, average node degree $d = 20$,
	uniform degree distribution, and
	compatibility matrix $\HVec$.
	We remove the labels from  $1-f = 90\%$ nodes,
	then calculate the
	top entry
	in both $\PVecEst^{(\ell)}$ and $\PVecEstEC^{(\ell)}$.
	The two bars in \cref{Fig_Backtracking_Advantage_21} show the mean and standard deviation of the corresponding matrix entries,
	illustrating that the approach based on non-backtracking paths
	leads to an unbiased estimator (height of orange bars are identical to the red circles),
	in contrast to the full paths (blue bars are higher than the red circles).
\qed
\end{example}

\begin{figure}[t]
\centering
\begin{subfigure}[b]{.48\linewidth}
	\centering
	\includegraphics[scale=0.38]{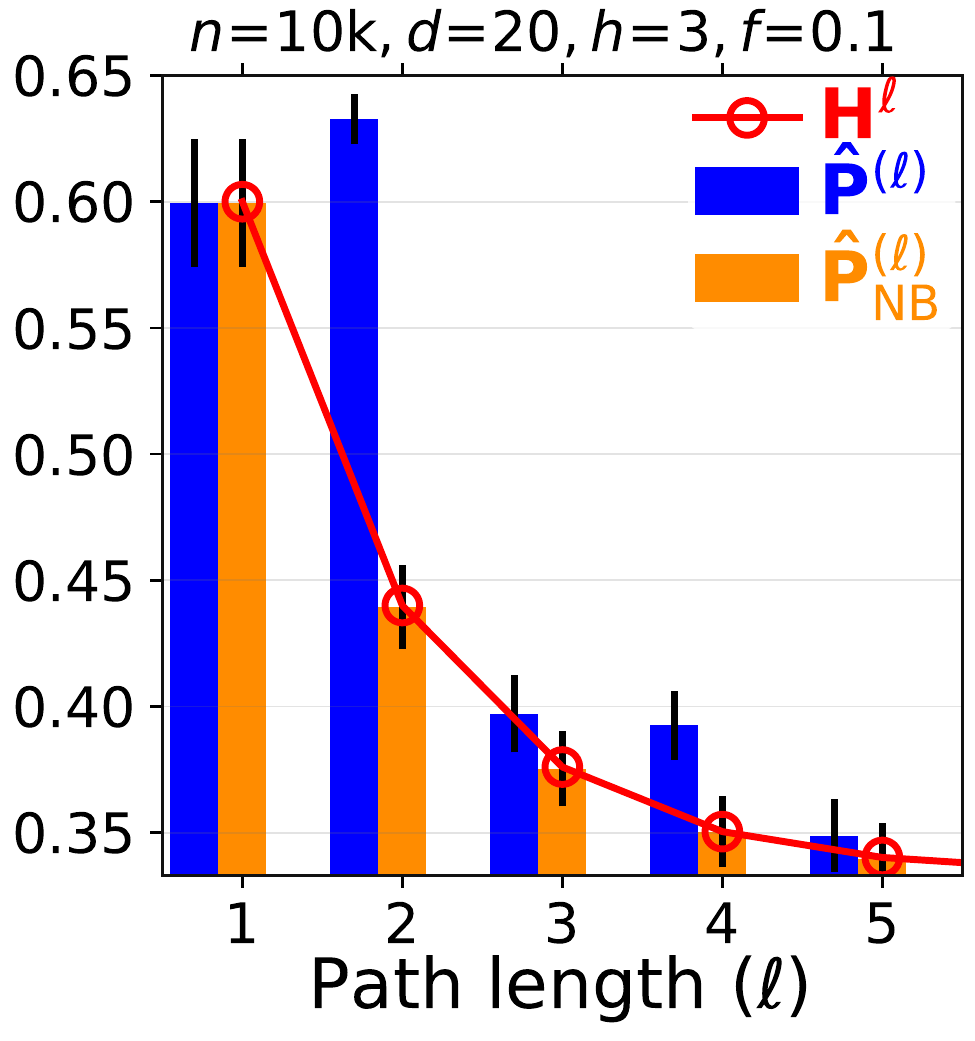}
	\caption{\Cref{ex:non-backtracking}
	\label{Fig_Backtracking_Advantage_21}
	}
\end{subfigure}
\hspace{1mm}
\begin{subfigure}[b]{.48\linewidth}
	\centering
	\includegraphics[scale=0.38]{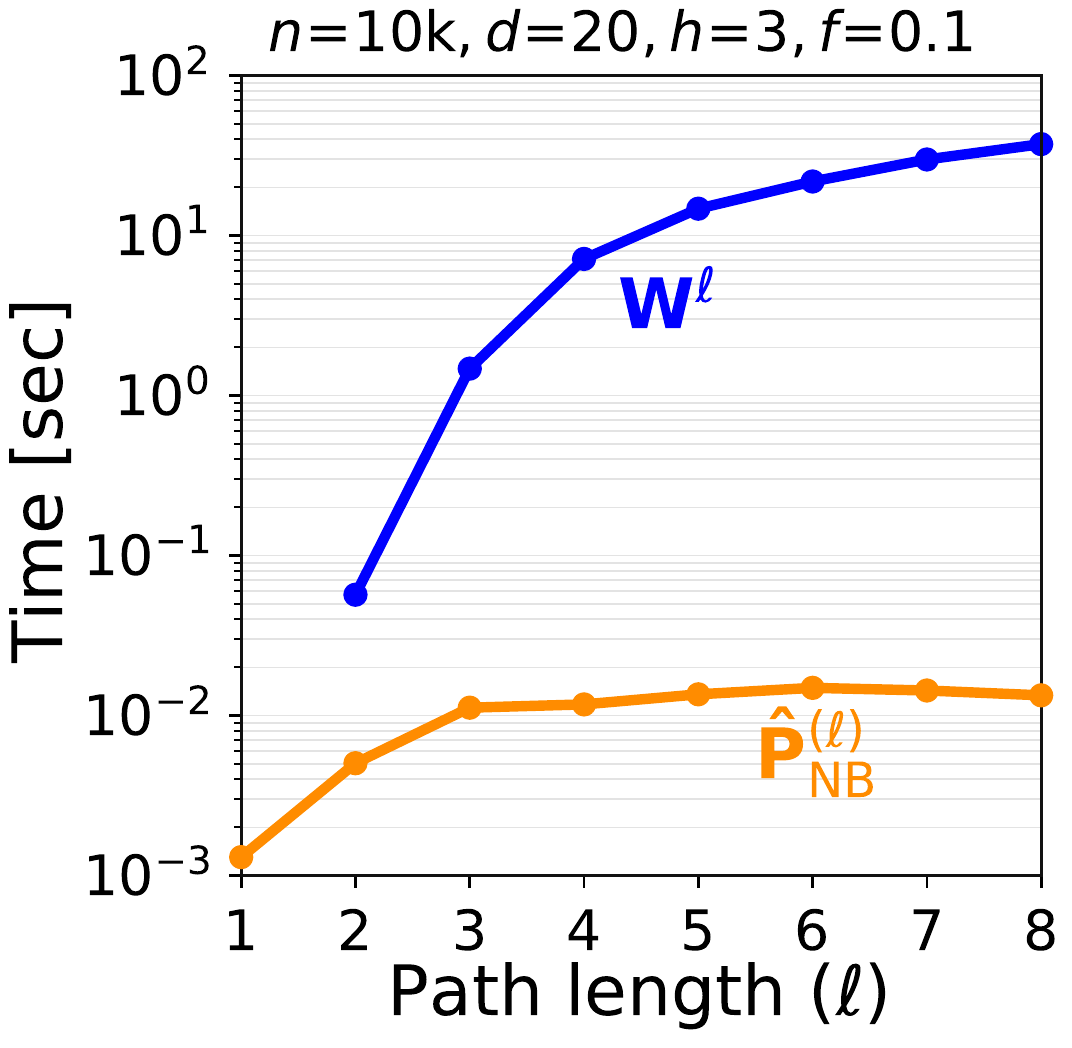}
	\caption{\Cref{ex:efficientCalculation}
	\label{Fig_Scaling_Hrow_4}
	}
\end{subfigure}
\caption{(a):
\Cref{ex:non-backtracking}:
$\PVecEstEC^{(\ell)}$
uses non-backtracking paths only
and is a consistent estimator,
in contrast to $\PVecEst^{(\ell)}$.
(b)
\Cref{ex:efficientCalculation}:
Calculating $\WVec^\ell$  for increasing $\ell$ is costly,
while
our factorized calculation of $\PVecEstEC^{(\ell)}$
avoids evaluating
$\WVec^\ell$
explicitly
and thus \emph{scales to arbitrary path lengths} $\ell$.}
\label{Fig_Scaling_Hrow}
\end{figure}

\subsection{Scalable, Factorized Path Summation}
\label{sec:efficientCalculation}
Calculating longer NB paths is more involved.
For example:
$\WVecEC^{(3)} =  \WVec^3 - (\vec D \WVec + \WVec \vec D - \WVec)$.
However, we can calculate them recursively as follows:

\begin{proposition}[Non-backtracking paths]\label{prop:nonbacktracking}
Let $\WVecEC^{(\ell)}$ be the matrix with
${\W_\NB^{(\ell)}}_{ij}$ being the number of non-backtracking paths of length $\ell$ from node $i$ to $j$.
Then $\WVecEC^{(\ell)}$ for $\ell \geq 3$ can be calculated via following recurrence relation:
\begin{align}
	\WVecEC^{(\ell)} &= \WVec \WVecEC^{(\ell-1)} - (\vec D - \vec I) \WVecEC^{(\ell-2)}
	\label{eq:recurrence}
\end{align}
with starting values
$\WVecEC^{(1)} = \WVec$ and
$\WVecEC^{(2)} = \WVec^2 - \vec D$.
\qed
\end{proposition}

Calculating $\PVecEstEC^{(\ell)}$ requires multiple matrix multiplications.
While matrix multiplication is associative,
the order in which we perform the multiplications
considerably affects the time to evaluate a product.
A straight-forward evaluation strategy quickly becomes infeasible for increasing $\ell$.

We illustrate with $\vec M^{(3)}$: a default strategy is to first calculate
$\WVec^{(3)} = \WVec (\WVec \WVec)$ and then
$\vec M^{(3)} = \EVec^\transpose (\WVec^{(3)} \EVec)$.
The problem is that the intermediate result $\WVec^{(\ell)}$ becomes dense.
Concretely, if $\WVec$ is sparse with $m$ entries and average node degree $d$,
then $\WVec^2$ has in the order of $d$ more entries ($\sim d m$),
and $\WVec^\ell$ exponential more entries ($\sim d^{\ell-1} m$).
Thus intuitively, we like to choose the evaluation order so that intermediate results are as sparse as possible.\footnote{
This is well known in linear algebra and is analogous to query optimization in relational algebra:
The two query plans
$\pi_y \big(R(x) \Join S(x,y)\big)$
and
$R(x) \Join\big(\pi_y S(x,y)\big)$
return the same values, but the latter can be considerably faster.
Similarly, the ``evaluation plans''
$(\WVec \WVec) \EVec$
and
$\WVec (\WVec \EVec)$
are algebraically equivalent, but the latter can be considerably faster
for $n \gg k$.
}
The ideal way to calculate the expressions is to keep $n \times k$ intermediate matrices as in
$\vec M^{(3)} = \EVec^\transpose (\WVec (\WVec (\WVec \EVec))$.

Our solution is thus to re-structure the calculation in a way that minimizes the result sizes of intermediate results and caches results used across estimators with different $\ell$.
The reason of the scalability of our approach is that we can calculate all $\ell_{\max}$ graph summaries very efficiently.

\begin{algorithmN}[Factorized path summation]\label{def:factorizedpathsummation}
Iteratively calculate the graph summaries $\PVecEstEC^{(\ell)}$, for $\ell \in [\ell_{\max}]$ as follows:
\begin{enumerate}

\item
Starting from
$\vec N_{\NB}^{(1)} = \WVec \EVec$
and
$\vec N_{\NB}^{(2)} = \WVec \vec N_{\NB}^{(1)} - \vec D \EVec$,
iteratively calculate
$\vec N_{\NB}^{(\ell)} =
\WVec \vec N_{\NB}^{(\ell-1)} - (\vec D - \vec I)
\vec N_{\NB}^{(\ell-2)}$.

\item Calculate $\vec M_{\NB}^{(\ell)} = \EVec^\transpose  \vec N_{\NB}^{(\ell)}$.

\item Calculate $\PVecEstEC^{(\ell)}$ from normalizing ${\vec M}^{(\ell)}$
with \cref{eq:variant1}.

\end{enumerate}
\end{algorithmN}

\begin{proposition}[Factorized path summation]\label{prop:complexitystatistics}
	\Cref{def:factorizedpathsummation} calculates 
	all
	$\PVecEstEC^{(\ell)}$ for $\ell \in [\ell_{\max}]$
	in ${\O(m k \ell_{\max})}$.	
\end{proposition}

\begin{example}[Factorized path summation]
\label{ex:efficientCalculation}
	Using the setup from \cref{ex:non-backtracking},
	\cref{Fig_Scaling_Hrow_4} shows the times for evaluating
	$\WVec^{\ell}$ against our more efficient evaluation strategy for $\PVecEstEC^{(\ell)}$.
	Notice the three orders of magnitude speed-up for $\ell=5$.
	Also notice that $\PVecEstEC^{(8)}$ summarizes statistics over more than $10^{14}$ paths
	in a graph with 100k edges
	in less than $0.02$ sec.
\end{example}

\subsection{Gradient-based optimization}

Our objective to find a symmetric, doubly stochastic matrix that minimizes 
\cref{eq:DCE_energyNB} can be represented as
\begin{align}
\HVecEst &= \argmin_{\HVec}
	 E(\HVec)
	 \textit{ s.t. }
	 \HVec \vec 1 = \vec 1, \HVec^\transpose=\HVec
\label{eq:Hminimization}
\end{align}

\noindent
For $\ell_{\max}>1$, the function is non-convex and unlikely to have a closed-form solution.
We thus minimize the function with gradient descent.
However, we would require to calculate the gradient with regard to the free parameters.

\begin{proposition}[Gradient]\label{prop:gradient}
The gradient for \cref{eq:Hminimization}
and energy function \cref{eq:DCE_energyNB}
with regard to the free parameters $H_{ij}, i \leq j, j \neq k$
is the dot product $\vec S \vec G$ calculated from
\begin{align*}
	\vec G
	&=
	2
	\sum_{\ell=1}^{\ell_{\max}} w_{\ell}
	\Big(
	\ell \HVec^{2 \ell-1}
		- \sum_{r=0}^{\ell-1} \HVec^r \HVecEst^{(\ell)} \HVec^{\ell-r-1}
	\Big)
\\
\!\!\!\!\!\vec S^{ij} \! &= \!\begin{cases}
	\vec J^{ij} \!+ \vec J^{ji}
		\!- \vec J^{ik} \!- \vec J^{kj} \!- \vec J^{jk} \!- \vec J^{ki}
		\!+ 2 \vec J^{kk},
		&\!\!\!\!\text{if } i < j, j \neq k \\
	\vec J^{ij} \!- \vec J^{ik} \!- \vec J^{kj}  \!+ \vec J^{kk},
		&\!\!\!\!\text{if } i = j, j \neq k
	\end{cases}
\end{align*}
where $\vec J^{ij}$ is single-entry matrix with 1 at $(i, j)$ and 0 elsewhere.
\end{proposition}

\subsection{DCE with restarts (DCEr)}\label{sec:RandomInit}

Whereas MCE solves a convex optimization problem, the objective function for DCE
becomes non-convex for a sparsely labeled graph (i.e.\ $f \ll 1$).
Given a number of classes $k$,
DCE optimizes over $k^* = \Theta(k^2)$ free parameters.
Since the parameter space have several local minimas,
the optimization should be {restarted} from multiple points in the $k^*$-dimensional parameter space,
in order to find the global minimum.
Thus DCEr optimizes the same energy function \cref{eq:DCE_energy} as DCE, but
with \emph{multiple restarts} from different initial
values.

Here our two-step approach of separating the estimation into two steps
(recall \cref{Fig_TwoSteps})
becomes an asset:
Because
the optimizations run on small sketches of the graph that are independent of the graph size,
starting multiple optimizations is actually cheap.
For small $k$, restarting from within each of the $2^{k^*}$
possible hyper-quadrants of  parameter space
(each free parameter being $\frac{1}{k} \pm \delta$ for some small $\delta < \frac{1}{k^2}$)
is negligible as compared to the graph summarization:
This is so as for increasing $m$ (large graphs),
calculation of the graph statistics dominates the cost for optimization
(see \cref{Fig_Timing_3}, where DCE and DCEr are effectively equal for large graphs).
For higher $k$, our extensive experiments show that
in practice \cref{eq:DCE_energy} has nice enough properties
that
just restarting from a limited number of restarts
usually leads to a compatibility matrix that achieves the optimal labeling accuracy (see \cref{Fig_fast_optimal_restarts_Accv2_107} and discussion in \cref{sec:scalability}).

\subsection{Complexity Analysis}\label{sec:complexityanalysis}

\Cref{prop:complexitystatistics}
shows that
our factorized approach for calculating all $\ell_{\max}$ graph estimators $\PVecEstEC^{(\ell)}$ is
${\O(m k \ell_{\max})}$ 
and thus is linear in the size of the graph (number of edges).
The second step of estimating the compatibility matrix $\HVec$ is then
\emph{independent of the graph size}
and dependents only on $k$ and the number of restarts $r$.
The number of free parameters in the optimization is $k^* = \O(k^2)$,
and calculating the Hessian is quadratic in this number. Thus, the second step is 
${\O(k^4 r)}$.

%% file: sections/experiments.tex
\section{Experiments}\label{sec:experiments}

We designed the experiments to answer two key questions: 
(1) How \emph{accurate} is our approach in predicting the remaining nodes (and how sensitive is it with respect to our single hyperparameter)?
(2) How \emph{fast} is it and how does it scale? 

We use two types of datasets:
we first perform \emph{carefully controlled experiments}
on synthetic datasets that allow us 
to change various graph parameters. 
We then use \emph{8 real-world datasets} with high levels of class imbalance,
various mixes between homophily and heterophily, and extreme skews of compatibilities.
There, we verify that our methods also work well on a variety of datasets which we did not generate.

Using both datasets, we show that: 
(1) Our method ``Distant Compatibility Estimation with restarts'' (DCEr) is 
largely insensitive to the choice of its hyperparameter and consistently 
competes with the labeling accuracy of using the ``true'' compatibilities (gold standard).
(2) DCEr is faster than label propagation with LinBP \cite{DBLP:journals/pvldb/GatterbauerGKF15} for large graphs, which makes it a simple and cheap pre-processing step (and thereby again rendering heuristics and domain experts obsolete).

\introparagraph{Synthetic graph generator}
We first use a completely controlled simulation environment. 
This setup allows us to {systematically change} parameters of the planted compatibility matrix and see the effect on the accuracy of the techniques as result of such changes.
We can thus make observations from many repeated experiments that would not be
be feasible otherwise.
Our synthetic graph generator is a variant of the widely studied stochastic block-model described in~\cite{SenGetoor:2007:LinkBasedClassification}, but 
with two crucial generalizations: 
(1) we {actively control the degree distributions} in the resulting graph (which allows us to plant more realistic power-law degree distributions); and
(2) we \emph{``plant'' exact graph properties} instead of fixing a property only in expectation.
In other words, our generator creates a desired degree distribution and 
compatibility matrix during graph generation,
which allows us to control all important parameters. 
The input to the generator is a tuple of parameters $(n, m, \bm\upalpha, \HVec, \dist)$ where
$n$ is the number of nodes,
$m$ the number of edges,
$\bm\upalpha$ the node label distribution with $\alpha(i)$ being the fraction of nodes of class $i$ ($i \in [k]$),
$\HVec$ any symmetric doubly-stochastic
compatibility matrix,
and ``\dist'' a family of degree distributions.
Notice that $\bm\upalpha$ allows us to simulate arbitrary \emph{node imbalances}.
In some of our synthetic experiments, 
we parameterize the compatibility matrix
by a value $h$ representing the ratio between min and max entries. Thus parameter $h$ models the ``{skew}'' of the potential:
For $k=3$,
$\HVec = \left[\begin{smallmatrix}
	1 & \,h & \,1  \\
	h & \,1 & \,1  \\
	1 & \,1 & \,h  \\
\end{smallmatrix}\right] / (2+h)$.
For example, 
$\HVec = \left[\begin{smallmatrix}
	0.1 & \,0.8 & \,0.1  \\
	0.8 & \,0.1 & \,0.1  \\
	0.1 & \,0.1 & \,0.8  \\
\end{smallmatrix}\right]$
for $h=8$, 
and
$\HVec = \left[\begin{smallmatrix}
	0.2 & \,0.6 & \,0.2  \\
	0.6 & \,0.2 & \,0.2  \\
	0.2 & \,0.2 & \,0.6 \\
\end{smallmatrix}\right]$
for $h=3$ (see \cref{ex:non-backtracking}).
We create graphs with $n$ nodes and, assuming class balance, assign equal fractions of nodes to one of the $k$ classes, e.g.\
$\bm\upalpha = [\frac{1}{3},\frac{1}{3},\frac{1}{3}]$.
We also vary the average degree of the graph
$d=2\frac{m}{n}$
and perform experiments assuming 
power law (coefficient $0.3$) distributions.

\introparagraph{Quality assessment}
We randomly sample a 
{stratified fraction $f$} of nodes as seed labels
(i.e.\ classes are sampled in proportion to their frequencies)
and evaluate end-to-end accuracy as the fraction of the remaining nodes that
receive correct labels.
Random sampling of seed nodes mimic real-world setting, like social networks, where people who choose to disclose their data, like gender label or political affiliation, are randomly scattered.
Notice that decreasing $f$ represents increasing \emph{label sparsity}.
To account for class imbalance, we {macro-average} the accuracy, i.e.\ we take the 
mean of the partial accuracies for each class.

\begin{figure*}[t]
\centering

\begin{subfigure}[t]{.24\linewidth}
	\includegraphics[scale=0.4]{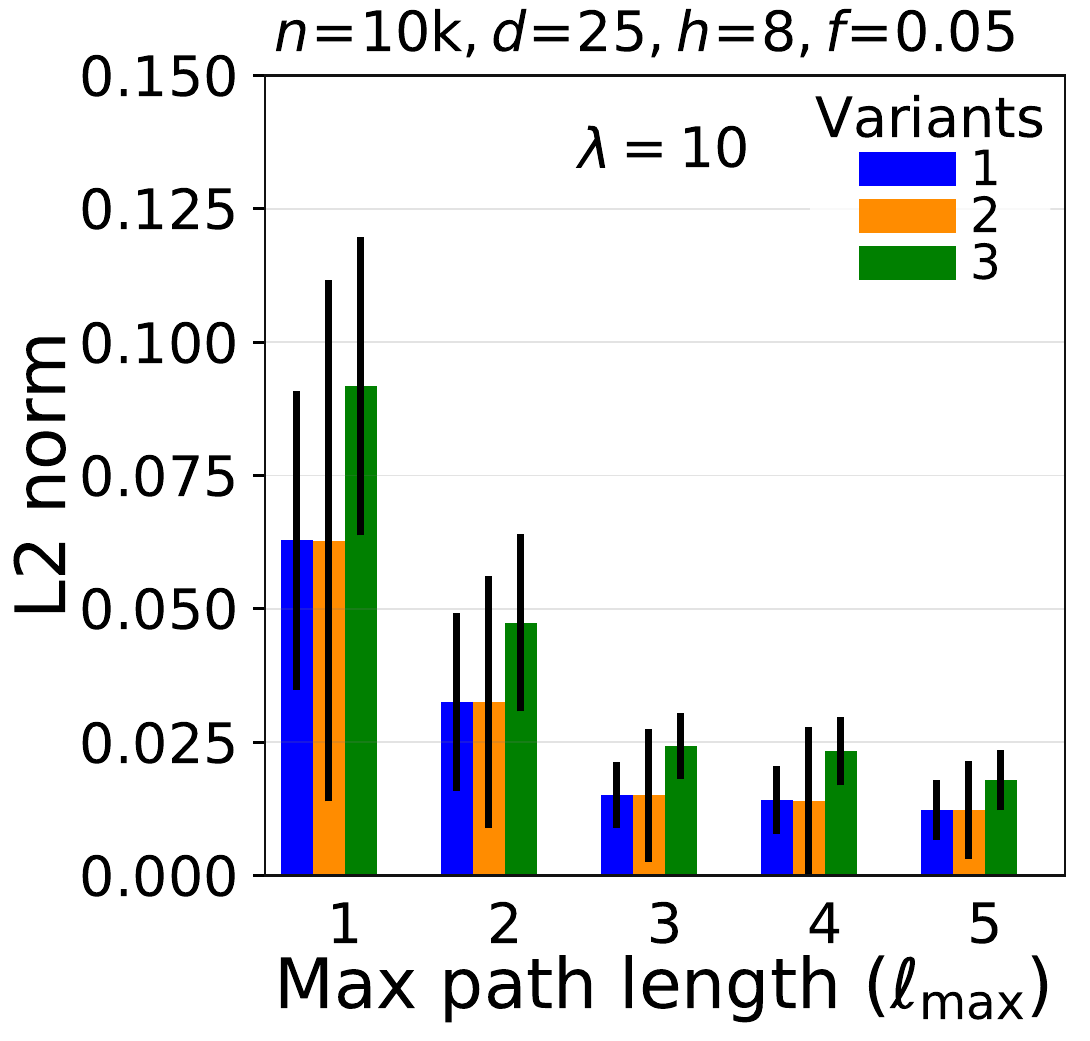}
	\caption{L2 norm DCE for 3 variants}\label{Fig_MHE_Variants_15_21}
\end{subfigure}
\begin{subfigure}[t]{.24\linewidth}
	\includegraphics[scale=0.4]{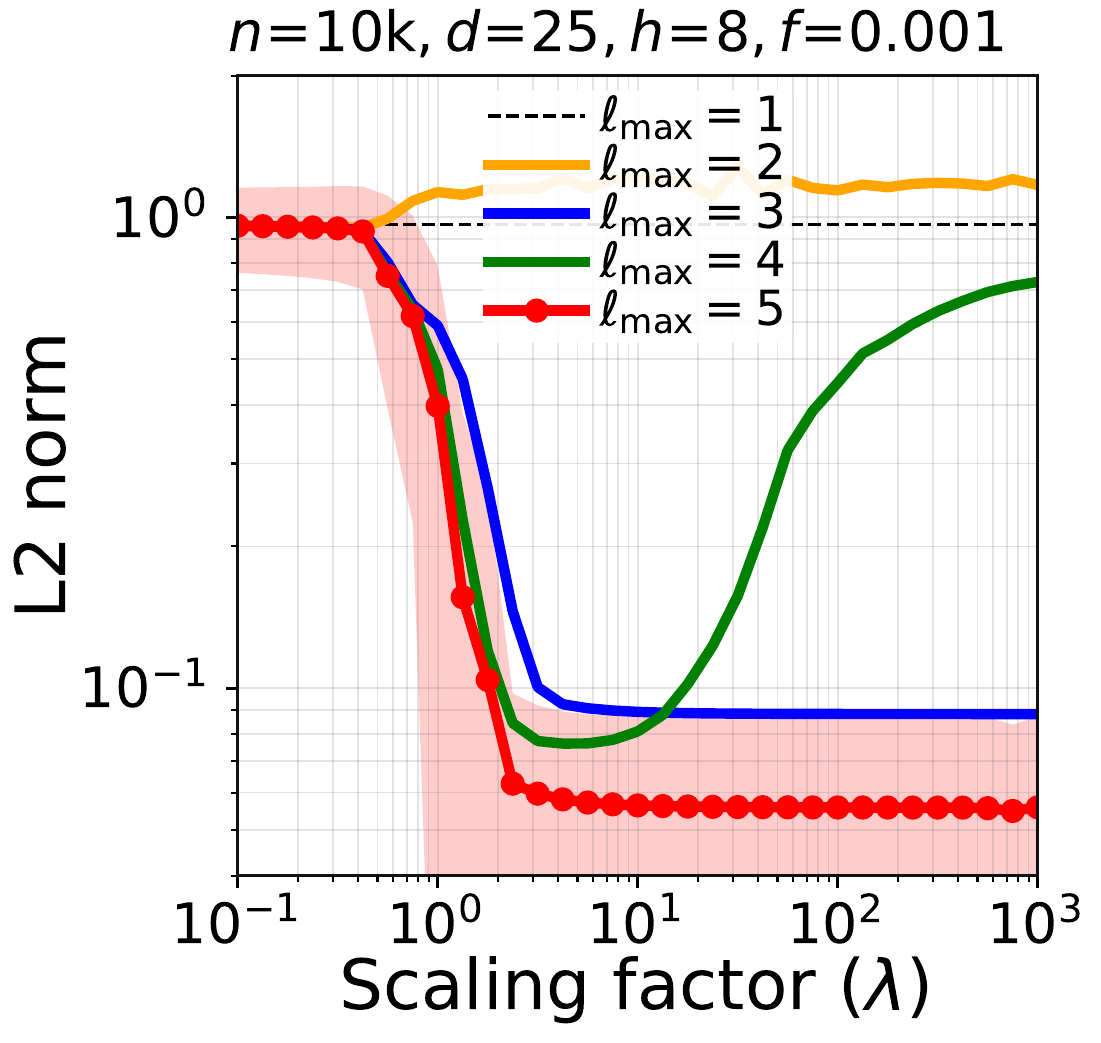}
	\caption{L2 norm DCEr $\lambda$ and $\ell_{\max}$}
	\label{Fig_MHE_ScalingFactor_133}
\end{subfigure}
\begin{subfigure}[t]{.24\linewidth}
	\includegraphics[scale=0.38]{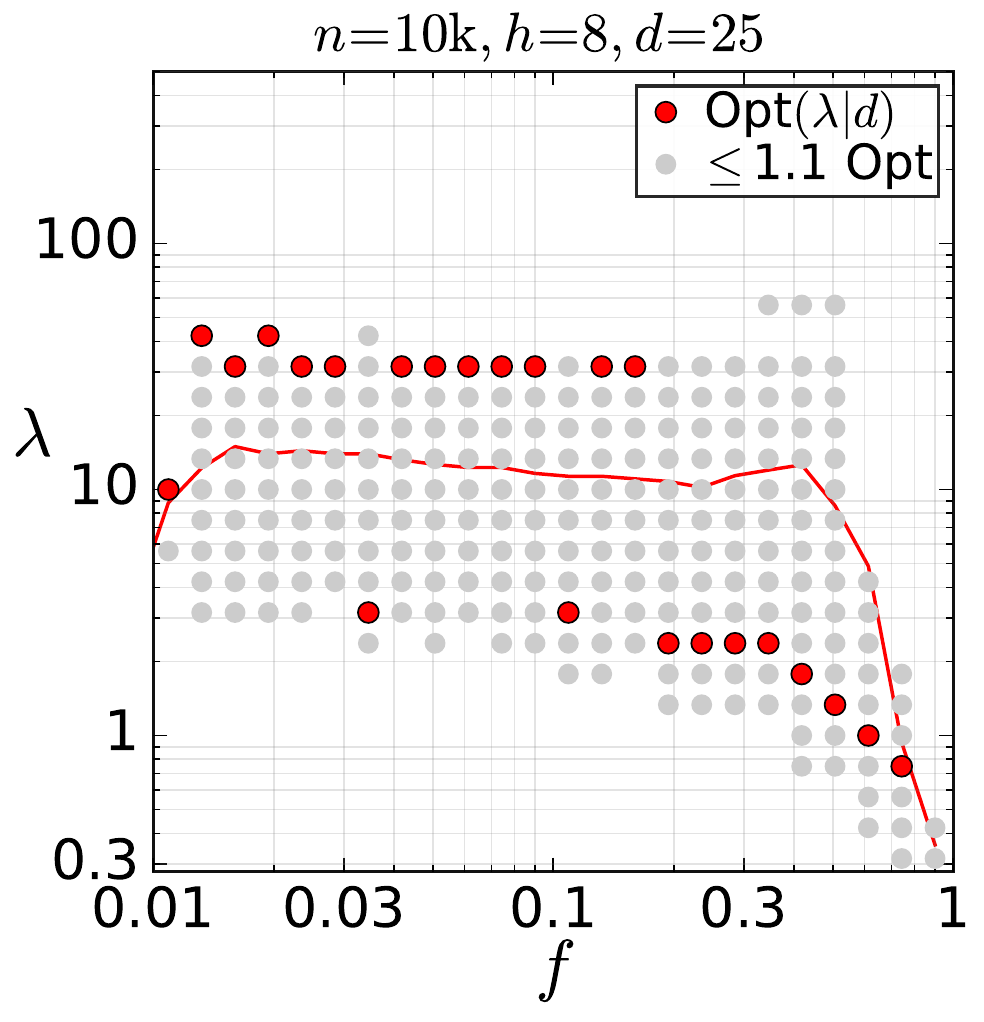}
	\caption{Robustness with $f$}
	\label{Fig_MHE_Optimal_ScalingFactor_lambda_f_3}
\end{subfigure}
\begin{subfigure}[t]{.24\linewidth}
	\includegraphics[scale=0.38]{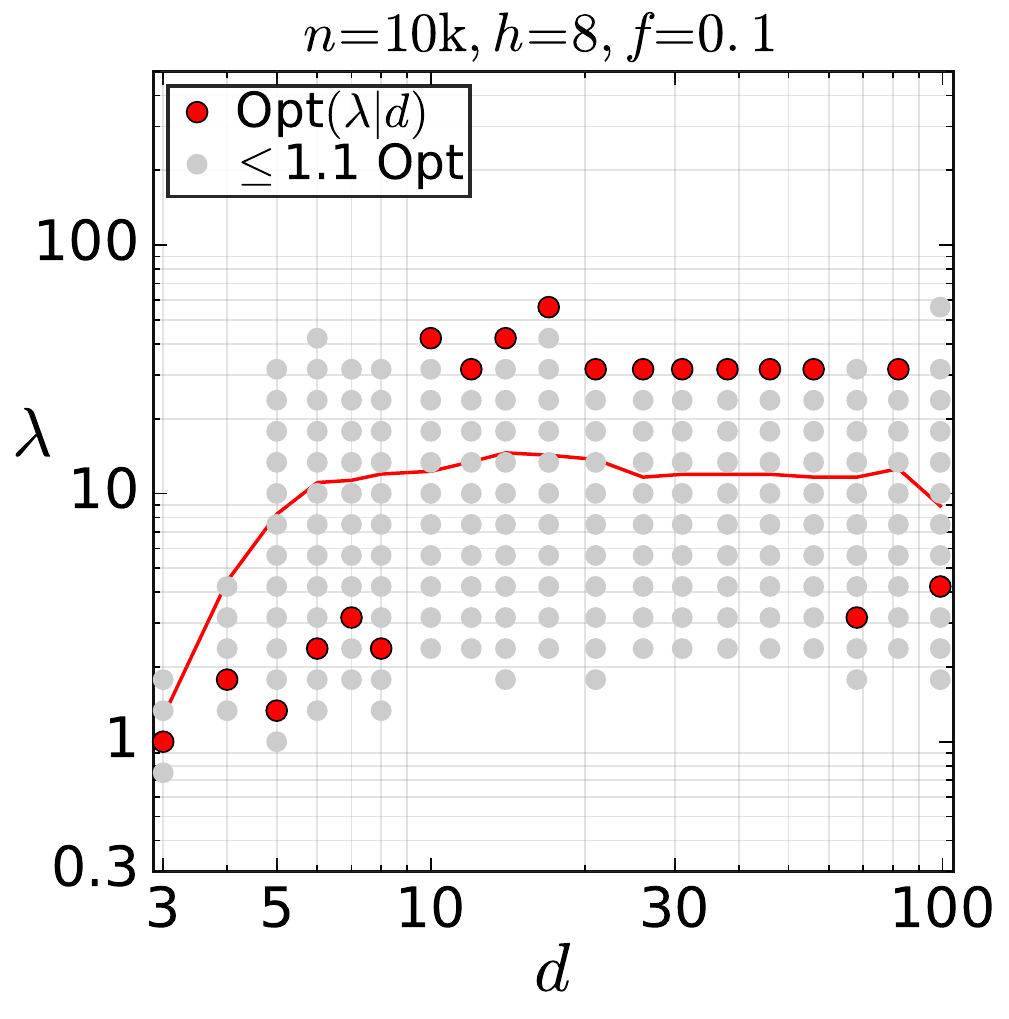}
	\caption{Robustness with $d$}
	\label{Fig_MHE_Optimal_ScalingFactor_lambda_d_2}
\end{subfigure}

\begin{subfigure}[t]{.24\linewidth}
	\includegraphics[scale=0.4]{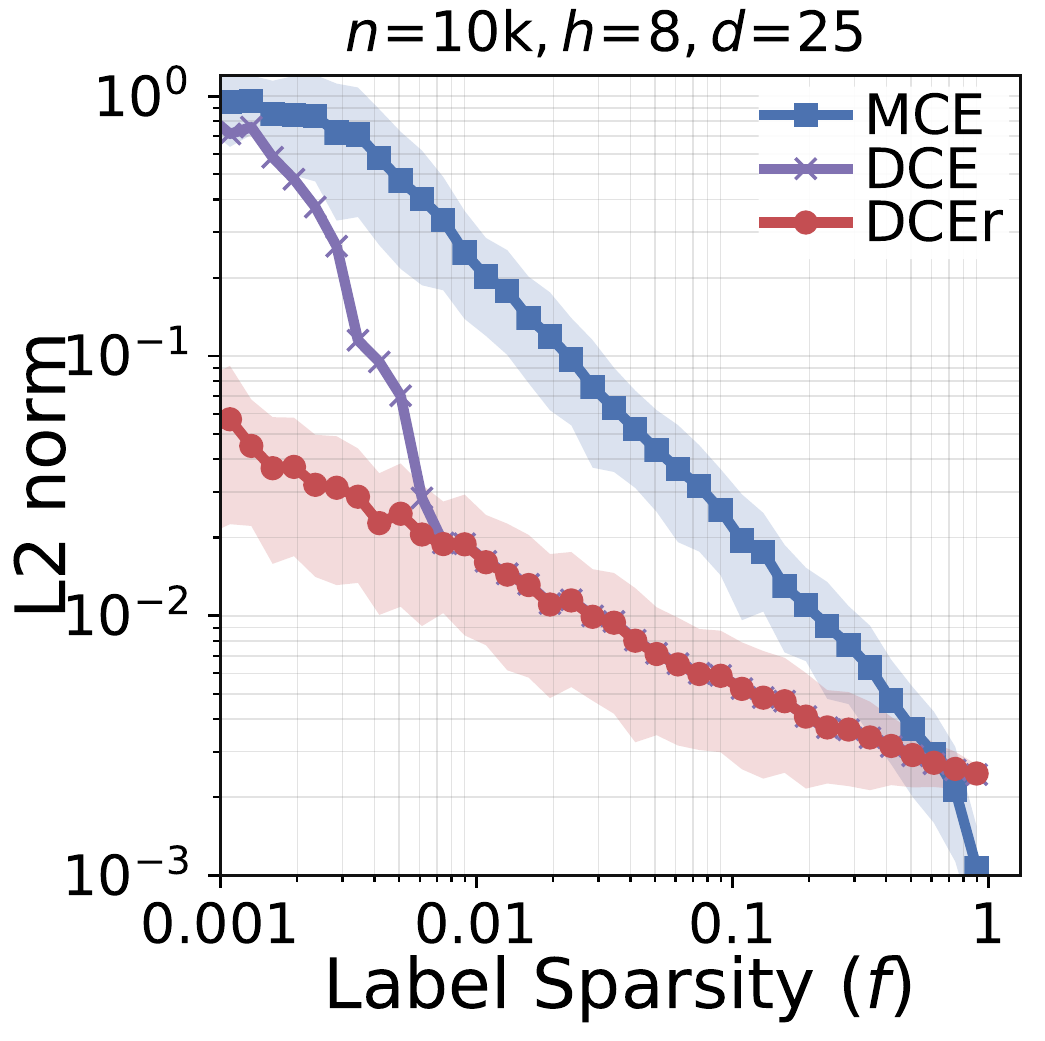}
	\caption{L2 norm MCE, DCE, DCEr}
	\label{Fig_MHE_Optimal_ScalingFactor_diff_f_lambda10_19}
\end{subfigure}
\begin{subfigure}[t]{.24\linewidth}
	\centering
		\includegraphics[scale=0.4]{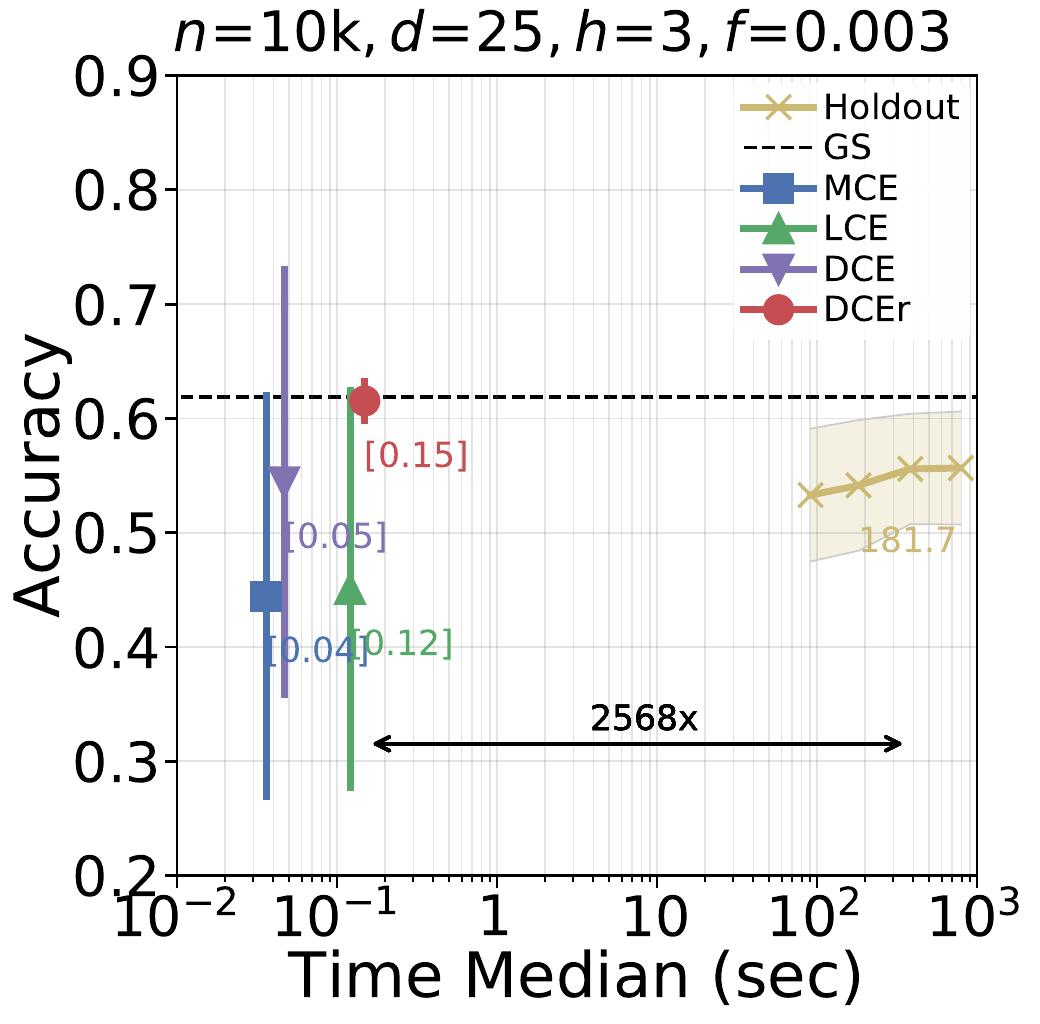}
	\caption{Accuracy vs. time}
	\label{Fig_timing_accuracy_learning_6}
\end{subfigure}
\begin{subfigure}[t]{.24\linewidth}
    \includegraphics[scale=0.4]{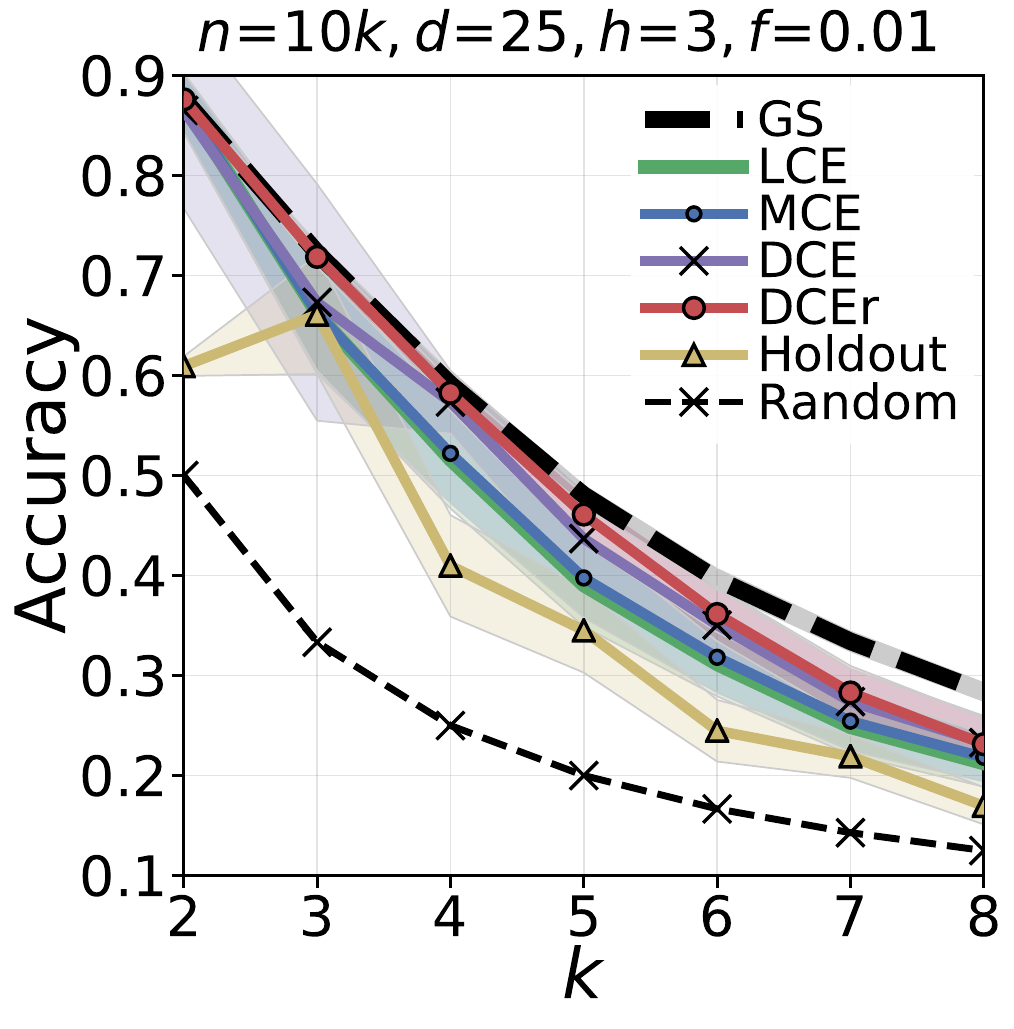}
    \caption{Estimation \& propagation
	}
    \label{Fig_End-to-End_accuracy507}
\end{subfigure}
\begin{subfigure}[t]{.24\linewidth}
	\includegraphics[scale=0.4]{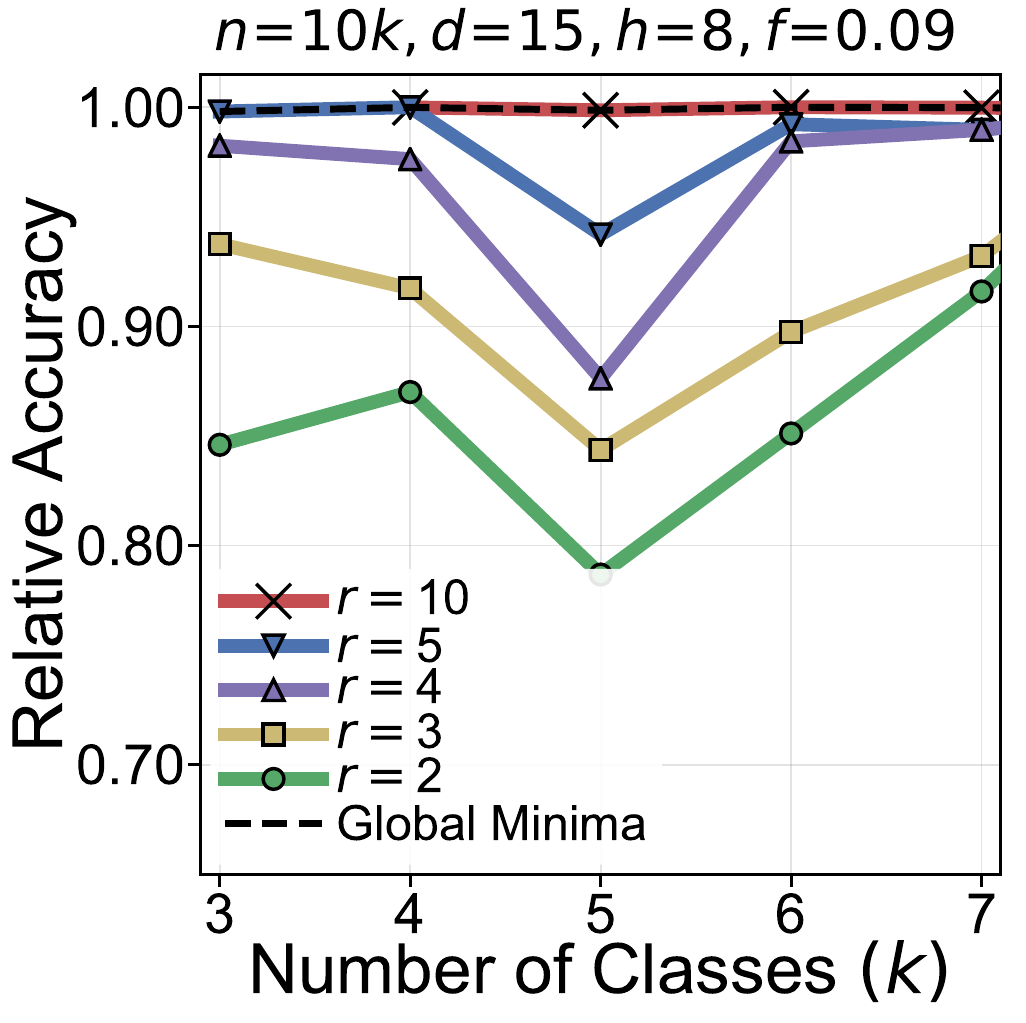}
	\caption{Restarts for DCEr}
	 \label{Fig_fast_optimal_restarts_Accv2_107}	
\end{subfigure}

\begin{subfigure}[t]{.24\linewidth}
	\includegraphics[scale=0.4]{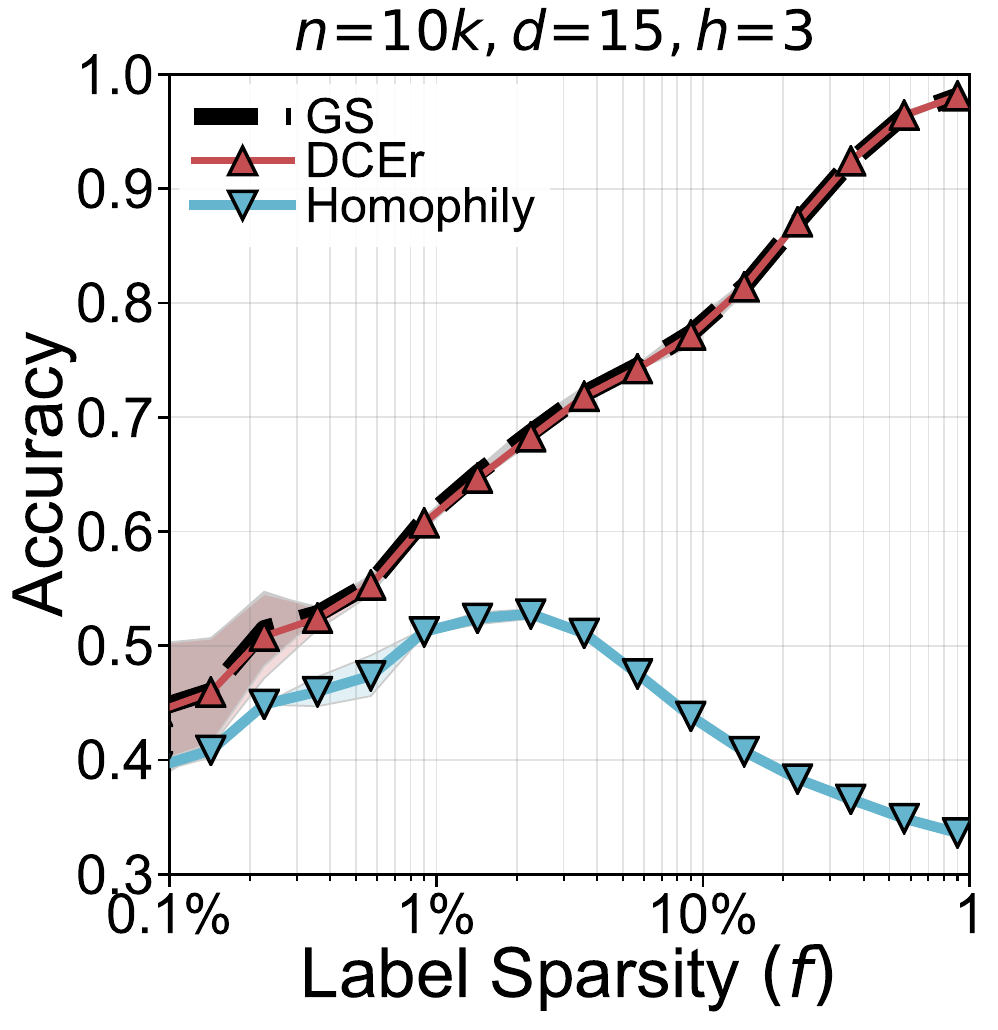}
	\caption{Homophily Comparison}	
		\label{fig:homophily_synthetic}	
\end{subfigure}
\begin{subfigure}[t]{.24\linewidth}
	\includegraphics[scale=0.4]{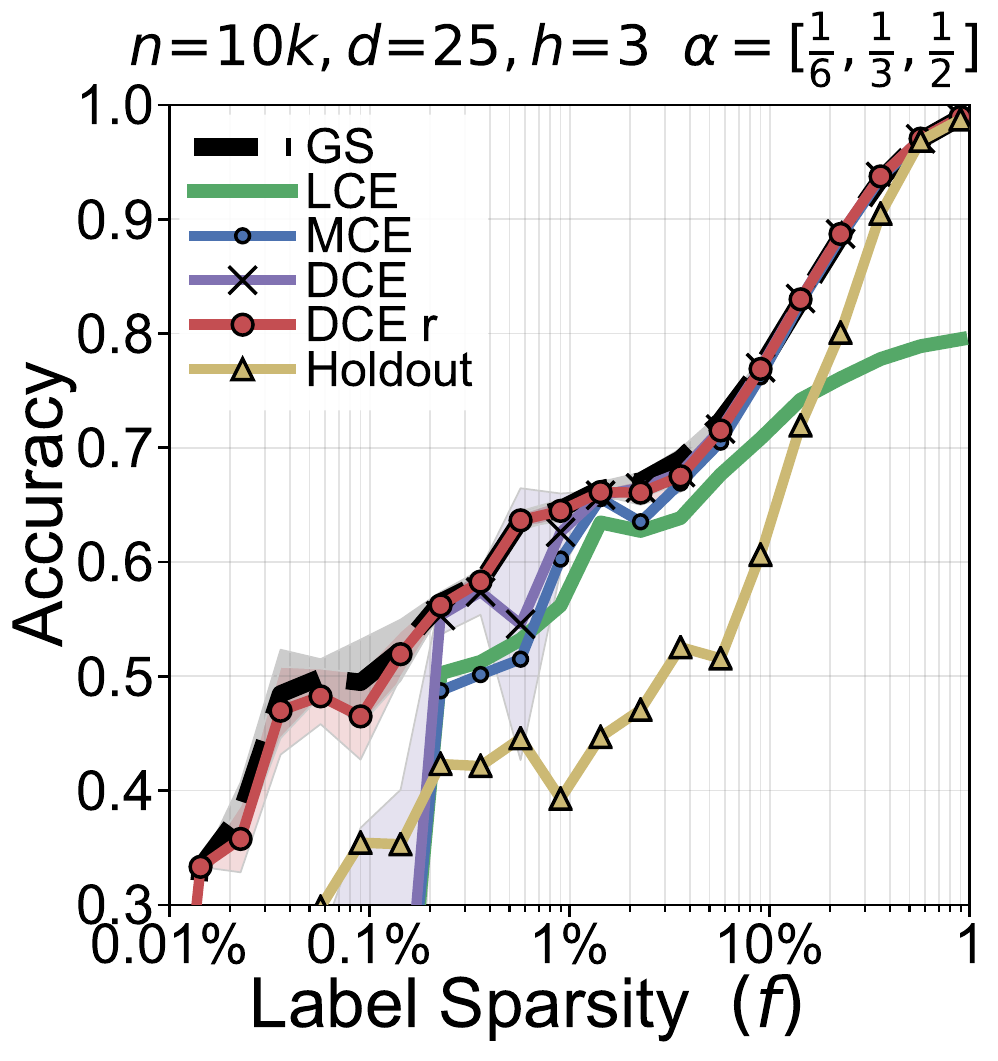}
	\caption{Estimation \& propagation}
	\label{Fig_End-to-End_accuracy213}
\end{subfigure}
\begin{subfigure}[t]{.24\linewidth}
	\includegraphics[scale=0.38]{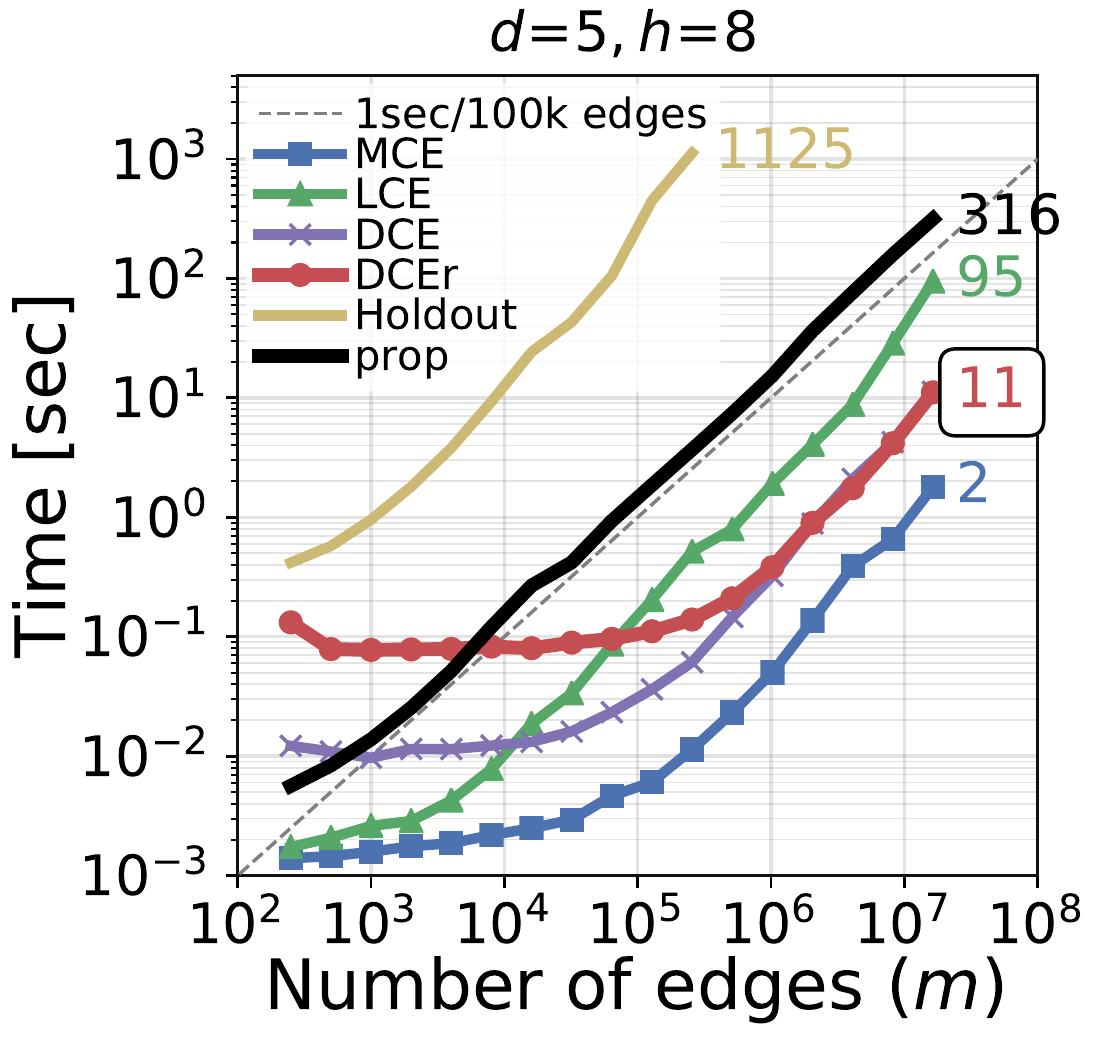}
	\caption{Scalability with $m$}\label{Fig_Timing_3}
\end{subfigure}
\begin{subfigure}[t]{.24\linewidth}
	\includegraphics[scale=0.38]{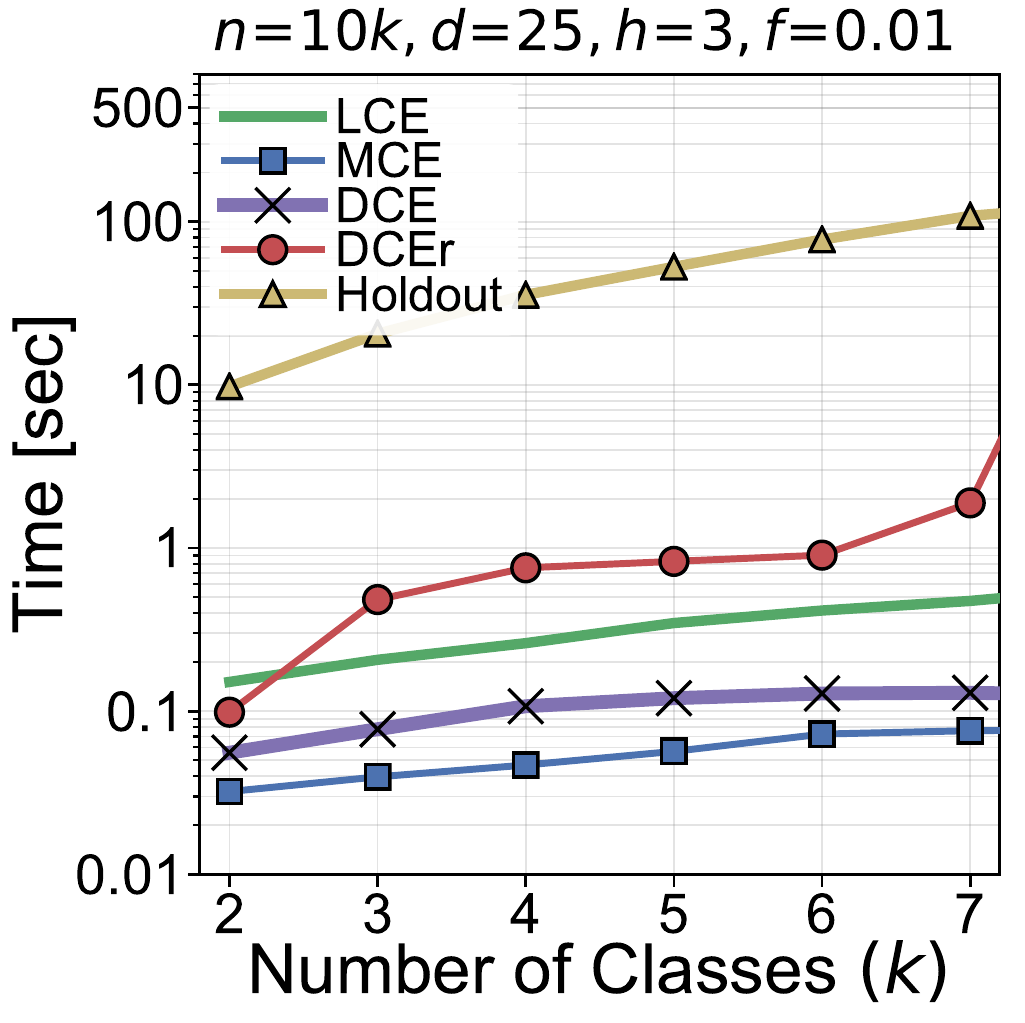}
	\caption{Scalability with $k$}
	\label{Fig_Time_varyK_607}
\end{subfigure}

\caption{
Experimental results for 
(a)-(j) accuracy (\cref{sec:estimation}),
and scalability of our methods (\cref{sec:scalability})
}\label{fig:experiments}
\end{figure*}

\introparagraph{Computational setup and code} We implement our algorithms in
Python using optimized libraries for sparse matrix operations
(NumPy~\cite{numpy} and Scipy~\cite{Jones:2001fk}).
Timing data was taken on a 2.5 Ghz Intel Core i5 with 16G of main memory and a
1TB SSD hard drive. 
Holdout method uses \sql{scipy.optimize} with the Nelder-Mead Simplex algorithm \cite{DBLP:journals/cj/Akitt77}, which is specifically suited for discrete non-contiguous functions.\footnote{We tried
    alternative optimizers, such as the Broyden-Fletcher-Goldfarb-Shanno
    (`BFGS') algorithm  \cite{DBLP:journals/siamjo/ArmandGJ00}. Nelder-Mead
performed best for the baseline holdout method due to its gradient-free nature.} All other estimation methods
use Sequential Least SQuares Programming (SLSQP).
The spectral radius of a matrix is calculated with an approximate method from
the PyAMG library~\cite{BeOlSc2011} that implements a technique described
in~\cite{Bai:2000fk}.  Our code (including the data generator) is inspired by
Scikit-learn \cite{scikit-learn} and will be made publicly available 
to encourage reproducible research.\footnote{\url{https://github.com/northeastern-datalab/factorized-graphs/}}

\subsection{Accuracy of Compatibility Estimation}\label{sec:estimation}
We show accuracy of parameter estimation by DCEr and compare it with ``holdout'' baseline and linear, myopic and simple distant variants. 
We consider propagation using `true compatibility' matrix as our gold standard (GS). 

\begin{tcolorbox}
\begin{resultW}(\textbf{Parameter choice of DCEr}) 
Normalization variant 1 and longer paths $\ell_{\max}=5$ are optimal for DCE. 
Choosing the hyperparameter $\lambda=10$ performs robustly for a wide range of 
average degrees $d$ and label sparsities $f$.
\end{resultW}
\end{tcolorbox}

\Cref{Fig_MHE_Variants_15_21} shows DCE used with our three normalization variants
and different maximal path lengths $\ell_{\max}$.
The vertical axis shows the L2-norm between estimation and GS
for $\HVec$. 
Variant 3 generally performs worse, and variant 2 generally has higher variance.
Our explanation is that finding the L2-norm closest symmetric doubly-stochastic matrix to a stochastic one is a well-behaved optimization problem.

\Cref{Fig_MHE_ScalingFactor_133} shows DCEr for various values of $\lambda$ and $\ell_{\max}$. 
Notice that DCEr for $\ell_{\max}\!=\!1$ is identical to MCE, and that DCEr works better for longer paths $\ell_{\max}\!=\!5$, 
as those can overcome sparsity of seed labels.
This observation holds over a wide range of parameters and becomes stronger for small $f$.
Also notice that even numbers $\ell_{\max}\!=\!2$ do not work as well as the objective has multiple minima with identical value.
\Cref{Fig_MHE_Optimal_ScalingFactor_lambda_f_3,Fig_MHE_Optimal_ScalingFactor_lambda_d_2} show the optimal choices of hyperparameter $\lambda$ (giving the smallest L2 norm from GS) for a wide range of $d$ and $f$. 
Each red dot shows an optimal choice of $\lambda$. 
Each gray dot shows a choice with L2-norm that is within 10\% of the optimal choice.
The red line shows a moving trendline of averaged choices. We see that choosing $\lambda=10$
is a general robust choice for good estimation, unless we have enough labels (high $f$): then we don't need longer paths and can best just learn from immediate neighbors (small $\lambda$).

\Cref{Fig_MHE_Optimal_ScalingFactor_diff_f_lambda10_19} shows the advantage of using $\ell_{\max} \!=\! 5$, $\lambda\!=\!10$ 
and random restarts for estimating $\HVec$ as compared to just MCE or DCE:
for small $f$, DCE may get trapped in local optima (see \cref{sec:RandomInit}); 
randomly restarting the optimization a few times overcomes this issue.

\begin{tcolorbox}
\begin{resultW}(\textbf{Accuracy performance of DCEr})
    Label accuracy with DCEr is within $\pm 0.01$ of GS performance
    and is quasi indistinguishable from GS.	
\end{resultW}
\end{tcolorbox}

\Cref{Fig_timing_accuracy_learning_6}
show results from first estimating 
$\HVec$ on a partially labeled graph and then labeling the remaining nodes with LinBP.
We see that \emph{more accurate estimation of $\HVec$ also translates into more accurate labeling},
which provides strong evidence that state-of-the-art approaches that use simple heuristics are not optimal.
GS runs LinBP with gold standard parameters and 
is the best LinBP can do.
The holdout method was varied with $b \in \{1, 2, 4, 8\}$ in \cref{Fig_timing_accuracy_learning_6} and $b=1$ else.
Increasing the number of splits 
moderately increases the accuracy for the holdout method, 
but comes at proportionate cost in time. 
DCEr is faster and more accurate throughout all parameters. 
In all plots, estimating with DCEr gives identical or similar labeling accuracy as knowing the GS. DCE is as good as DCEr for $f>1\%$ for 10k and $f>0.1\%$ for 100k nodes.
MCE and LCE both rely on labeled neighbors and have similar accuracy.

\Cref{Fig_MHE_ScalingFactor_133,Fig_End-to-End_accuracy213} 
show that neighbor frequency distributions alone do not work with sparse labels, and that our $\ell$-distance trick successfully overcomes its shortcomings.

\begin{tcolorbox}
\begin{resultW}(\textbf{Restarts required for DCEr})
With $r=10$ restarts, DCEr obtains the performance levels of GS. %
\end{resultW}
\end{tcolorbox}

\Cref{Fig_fast_optimal_restarts_Accv2_107} shows propagation accuracy of DCEr for different number of restarts $r$  
compared against the global minimum baseline, which is calculated by initializing DCE optimization with GS. 
Notice, the optimal baseline is the best any estimation based method can perform. 
Averaged over 35 runs, \cref{Fig_fast_optimal_restarts_Accv2_107} shows that DCEr approaches the global minima with just $10$ restarts and hence we use $r = 10$ in our experiments. 

\Cref{fig:homophily_synthetic} serves as sanity check and
demonstrates what happens if we use standard random walks (here the harmonic functions method~\cite{DBLP:conf/icml/ZhuGL03}) 
to label nodes in graphs with arbitrary compatibilities:
Baselines with a homophily assumption fall behind tremendously on graphs that do not follow assortative mixing.

\begin{tcolorbox}
\begin{resultW}(\textbf{Robustness of DCEr})
Performance of DCEr remains consistently above other baselines for skewed label distributions and large number of classes, whereas other SSL methods deteriorate for $k > 3$.
\end{resultW}
\end{tcolorbox}

To illustrate the approaches for class imbalance and more general $\HVec$,
we include an experiment with 
$\bm\upalpha = [\frac{1}{6},\frac{1}{3},\frac{1}{2}]$ and
$\HVec = \left[\begin{smallmatrix}
	0.2 & \,0.6 & \,0.2  \\
	0.6 & \,0.1 & \,0.3  \\
	0.2 & \,0.3 & \,0.5 \\
\end{smallmatrix}\right]$.
\Cref{Fig_End-to-End_accuracy213} (contrast to \cref{Fig_End-to-End_accuracy108}) 
shows that
DCEr works robustly better than alternatives, can deal with label imbalance, and can learn the more general $\HVec$.
\Cref{Fig_End-to-End_accuracy507} 
compares accuracy against random labeling for
fixed $n$, $m$, $h$, $f$, and increasing $k$.
DCEr restarts up to 10 times and works robustly better than alternatives. 
Recall that the number of compatibilities to learn is $\O(k^2)$.

\begin{figure*}[t]
\begin{subfigure}[t]{.24\linewidth}
	\centering
	\includegraphics[scale=0.4]{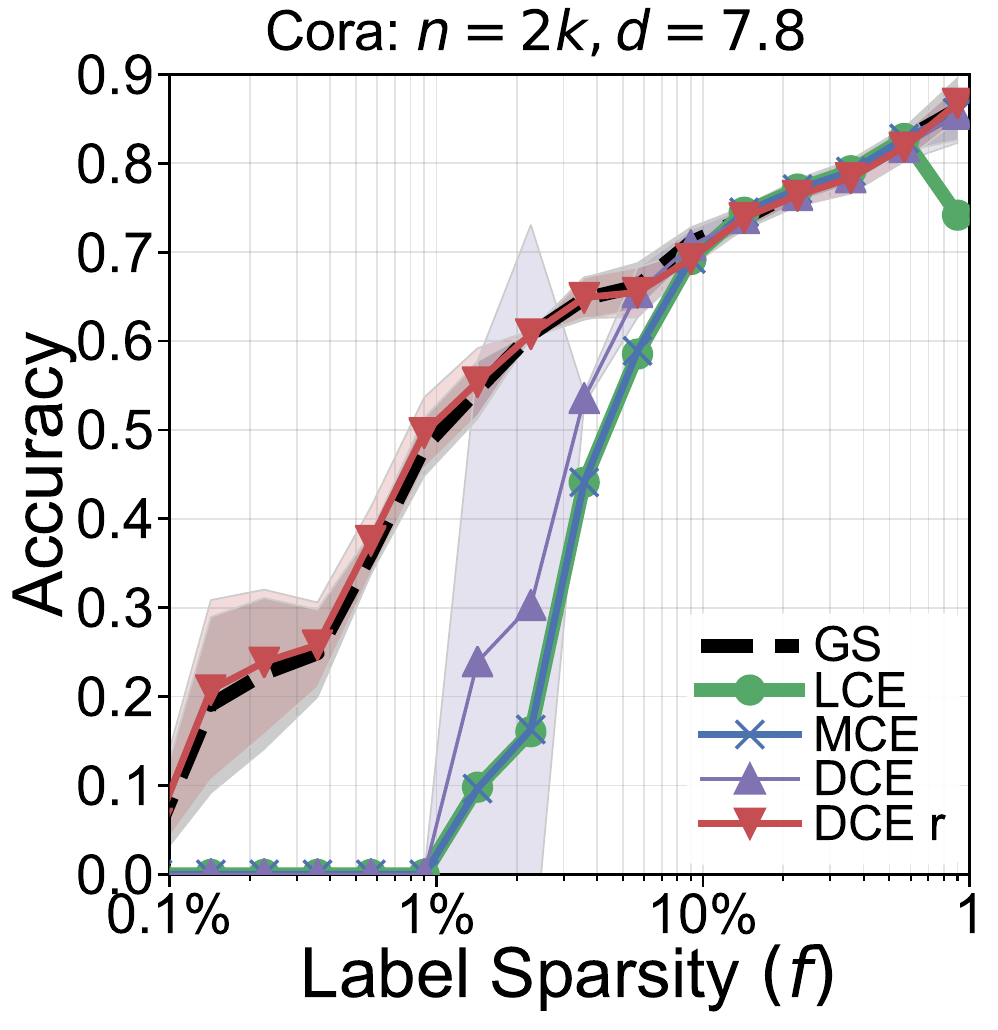}
	\caption{Cora}
\end{subfigure}
\begin{subfigure}[t]{.24\linewidth}
	\centering
	\includegraphics[scale=0.4]{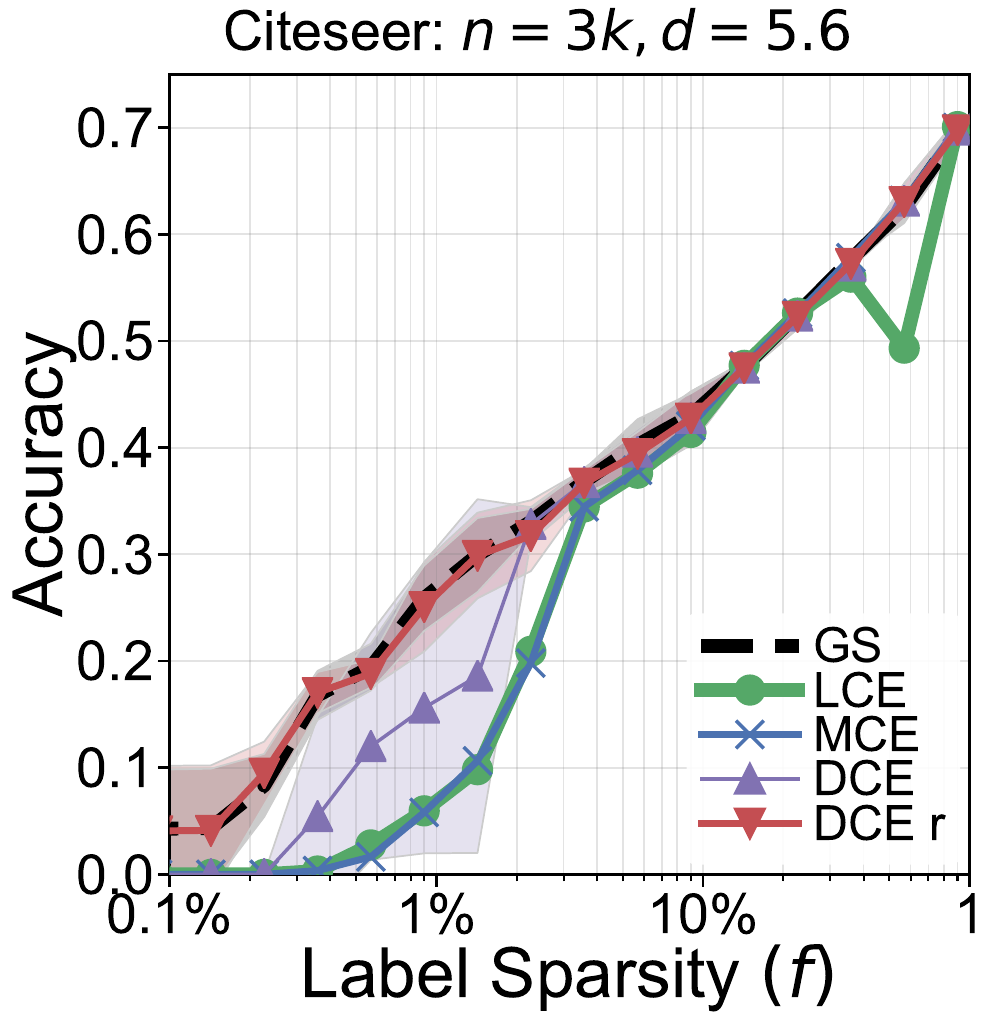}
	\caption{Citeseer}
\end{subfigure}
\begin{subfigure}[t]{.24\linewidth}
	\centering
	\includegraphics[scale=0.4]{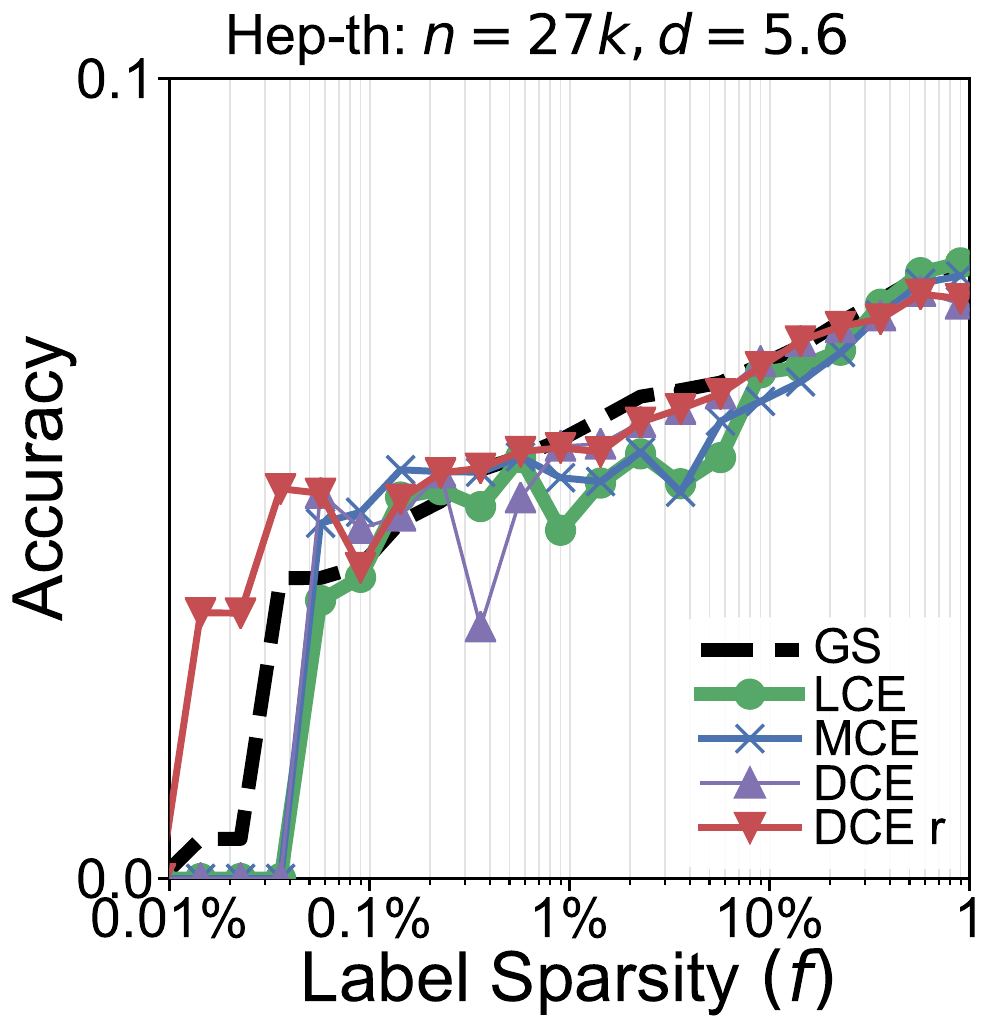}
	\caption{Hep-Th}
\end{subfigure}
\begin{subfigure}[t]{.24\linewidth}
	\includegraphics[scale=0.4]{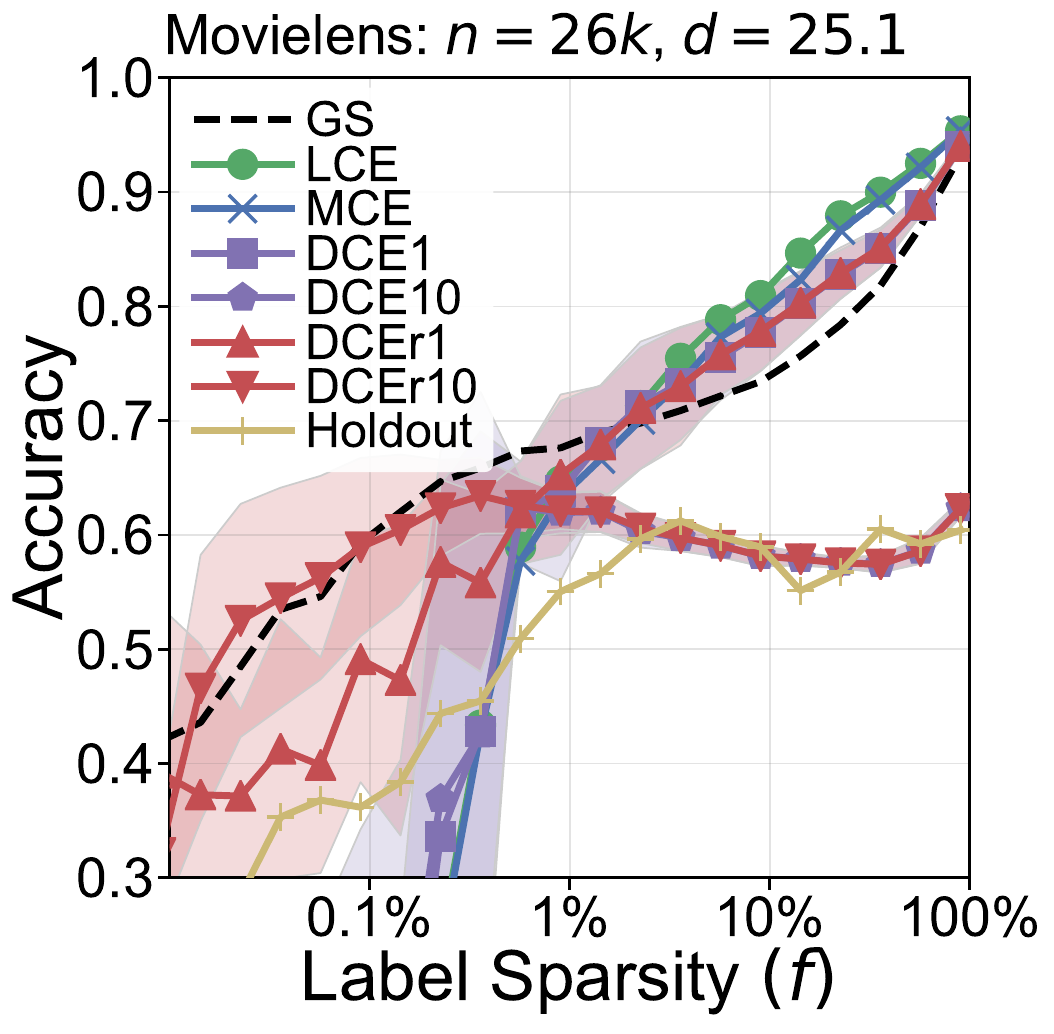}	
	\caption{MovieLens}\label{fig:movielens}
\end{subfigure}

\vspace{2mm}

\begin{subfigure}[t]{.24\linewidth}
	\includegraphics[scale=0.4]{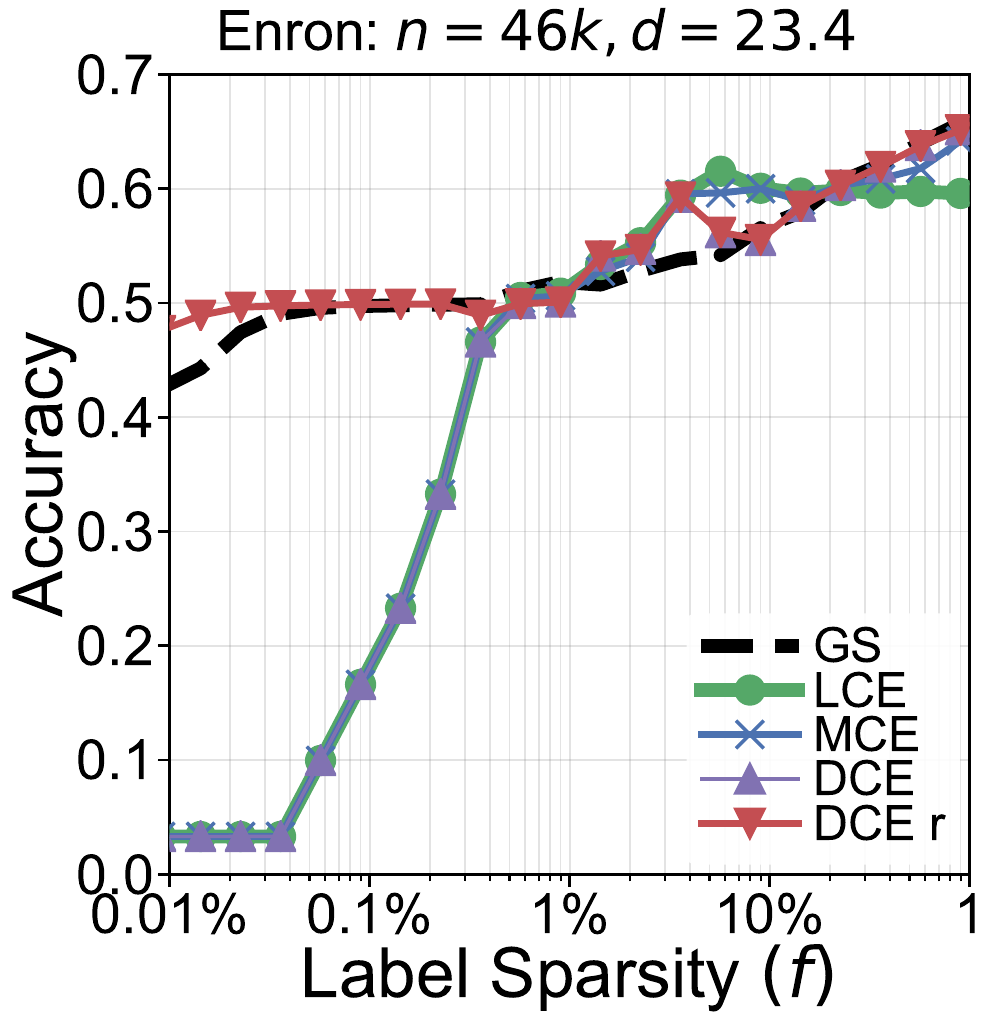}
	\caption{Enron}\label{fig:enron}
\end{subfigure}
\begin{subfigure}[t]{.24\linewidth}
	\includegraphics[scale=0.4]{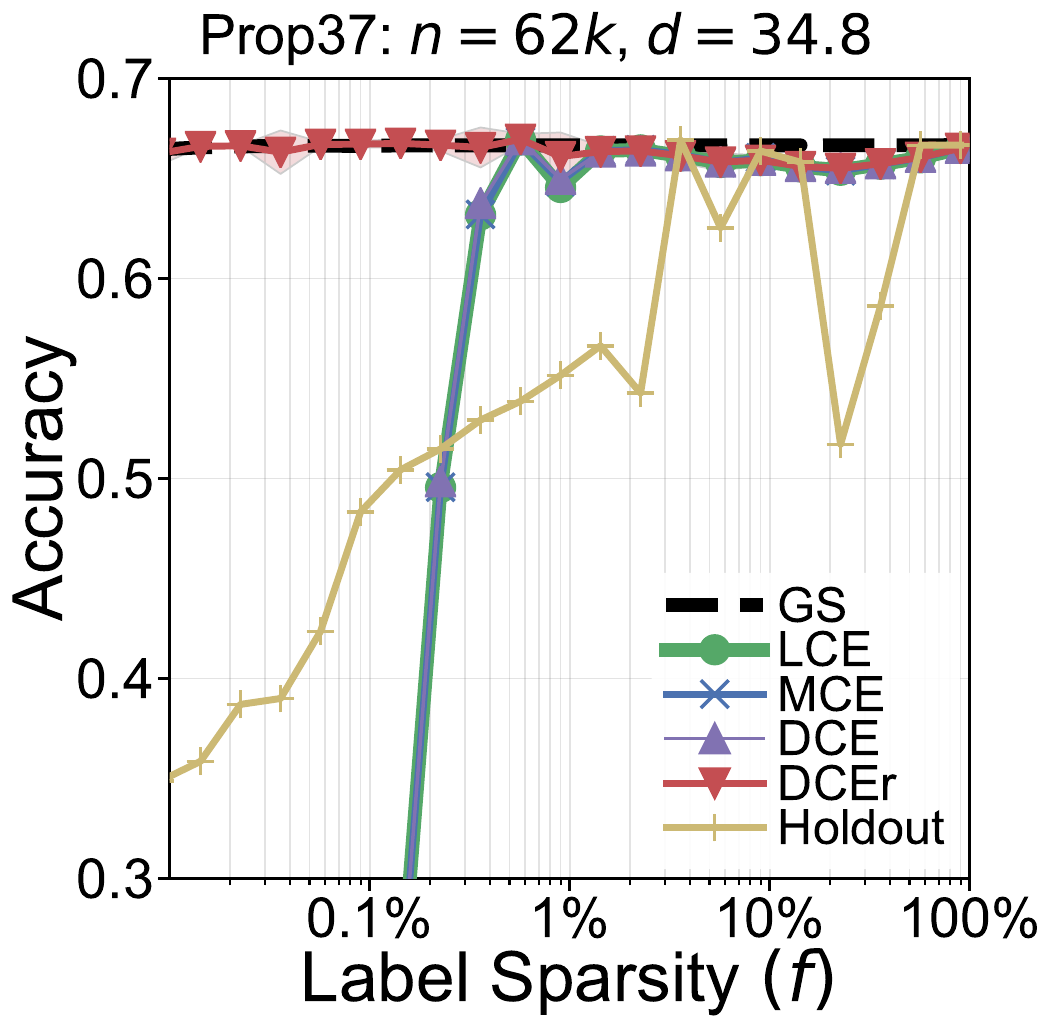}
	\caption{Prop-37}
    \label{fig:prop37}
\end{subfigure}
\begin{subfigure}[t]{.24\linewidth}
	\centering
	\includegraphics[scale=0.4]{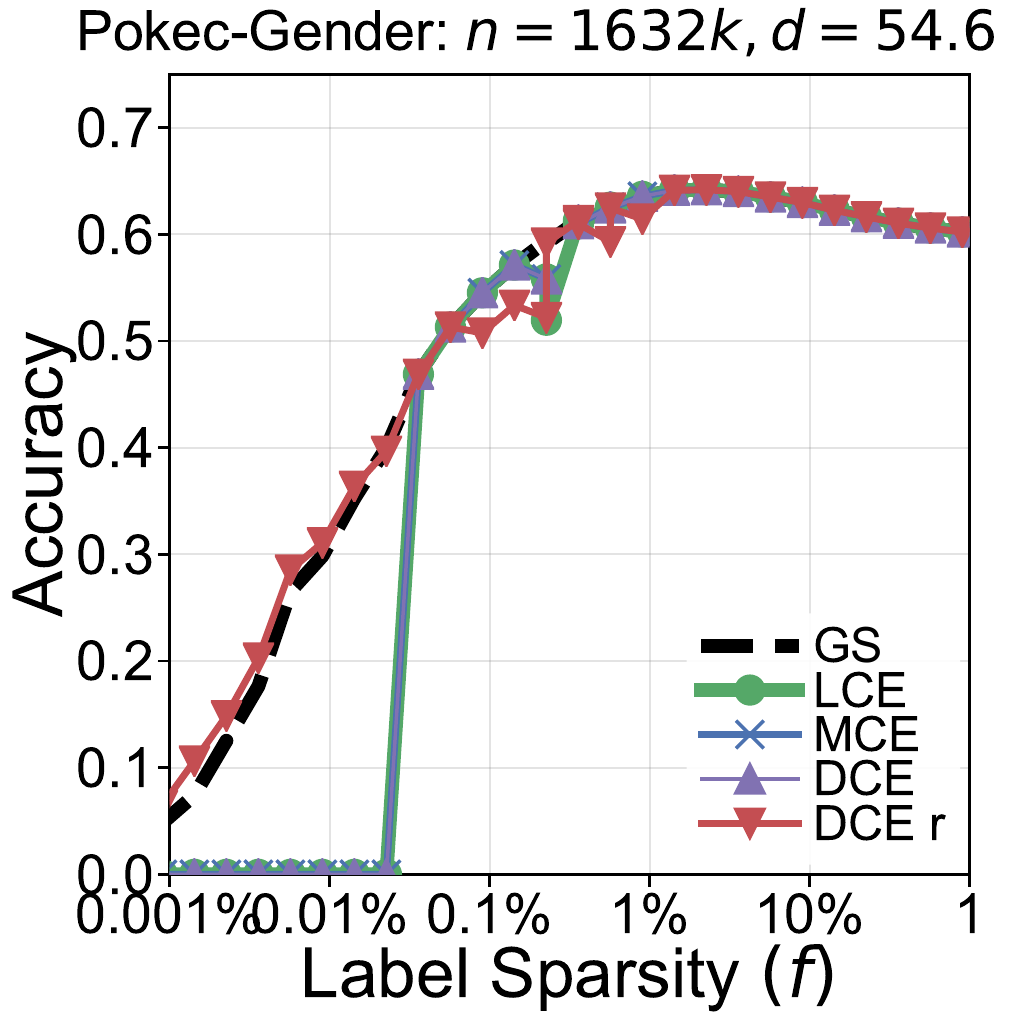}
	\caption{Pokec-Gender}
\end{subfigure}
\begin{subfigure}[t]{.24\linewidth}
	\includegraphics[scale=0.4]{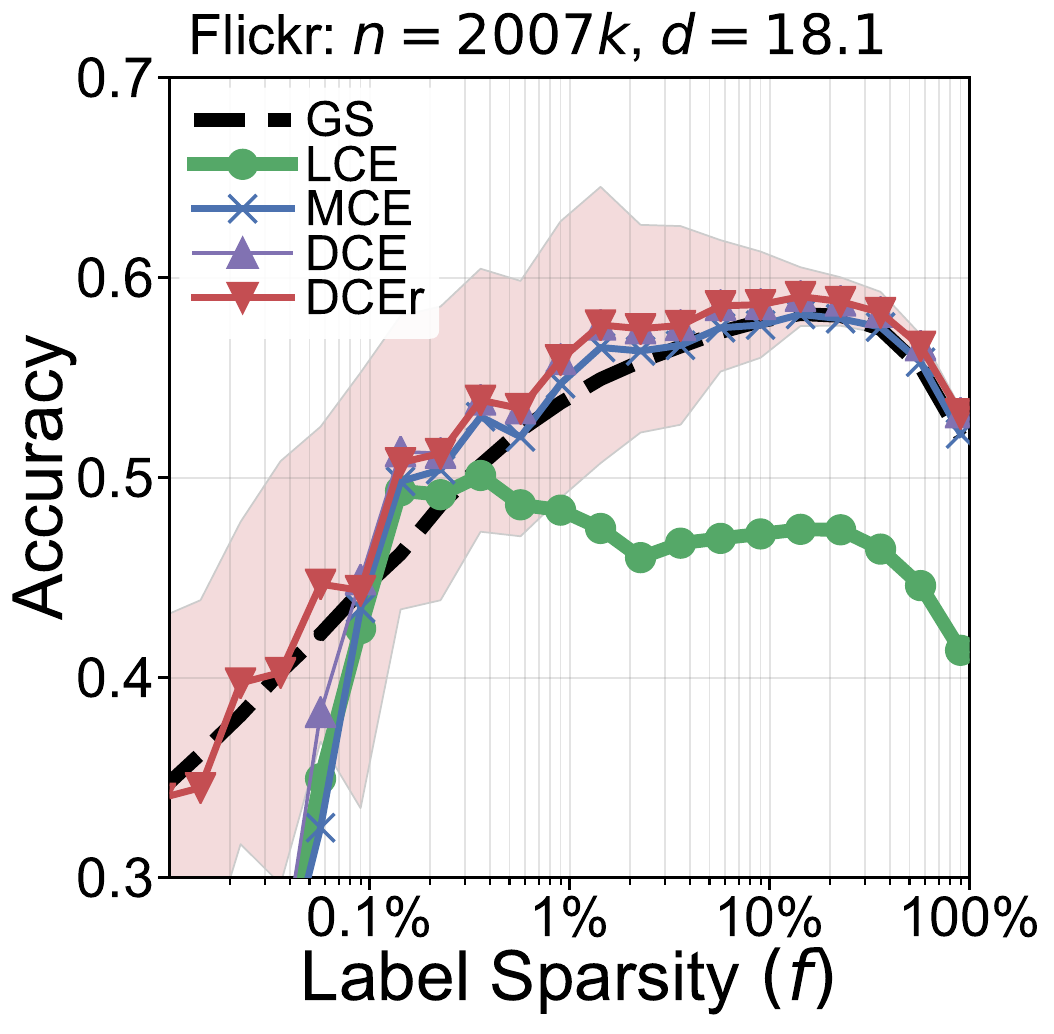}
	\caption{
        Flickr }\label{fig:flickr}
\end{subfigure}

\vspace{2mm}

\begin{subfigure}[t]{.12\linewidth}
	\centering
	\includegraphics[scale=0.22]{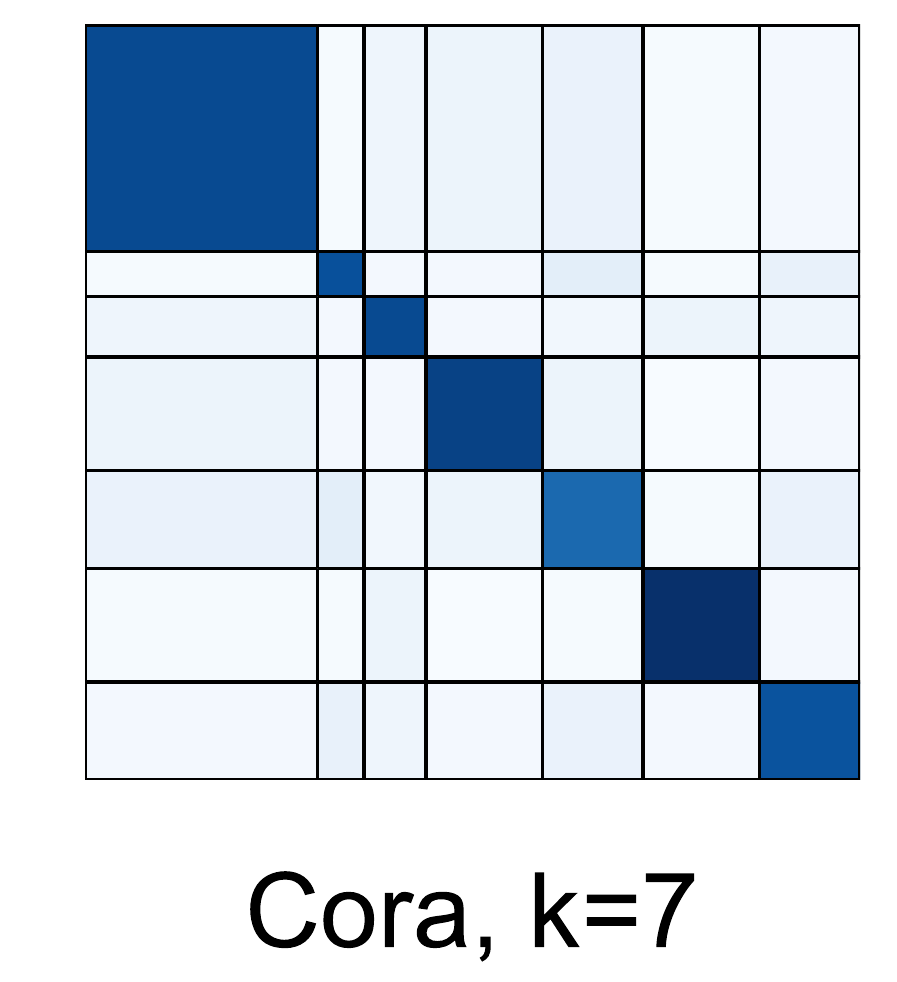}
	\caption{}
    \label{cora_r}
\end{subfigure}
\begin{subfigure}[t]{.12\linewidth}
	\centering
	\includegraphics[scale=0.22]{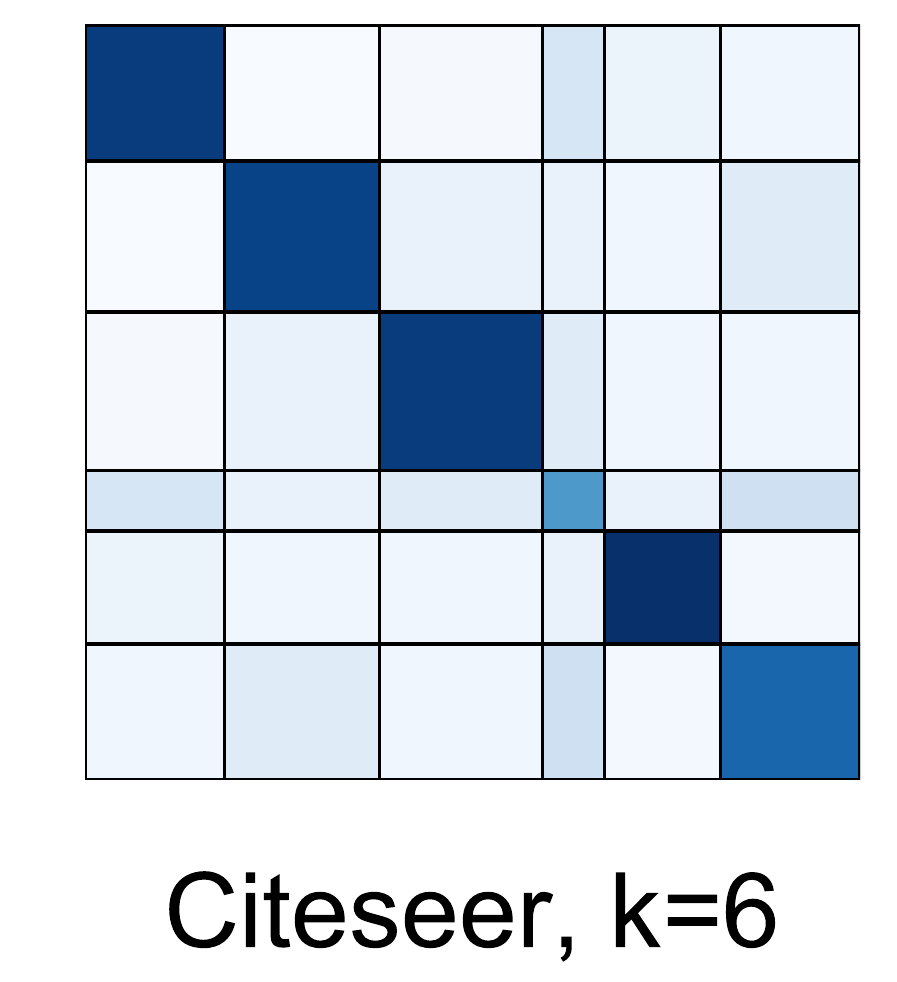}
	\caption{}
    \label{citeseer_r}
\end{subfigure} 
\begin{subfigure}[t]{.12\linewidth}
	\centering
	\includegraphics[scale=0.22]{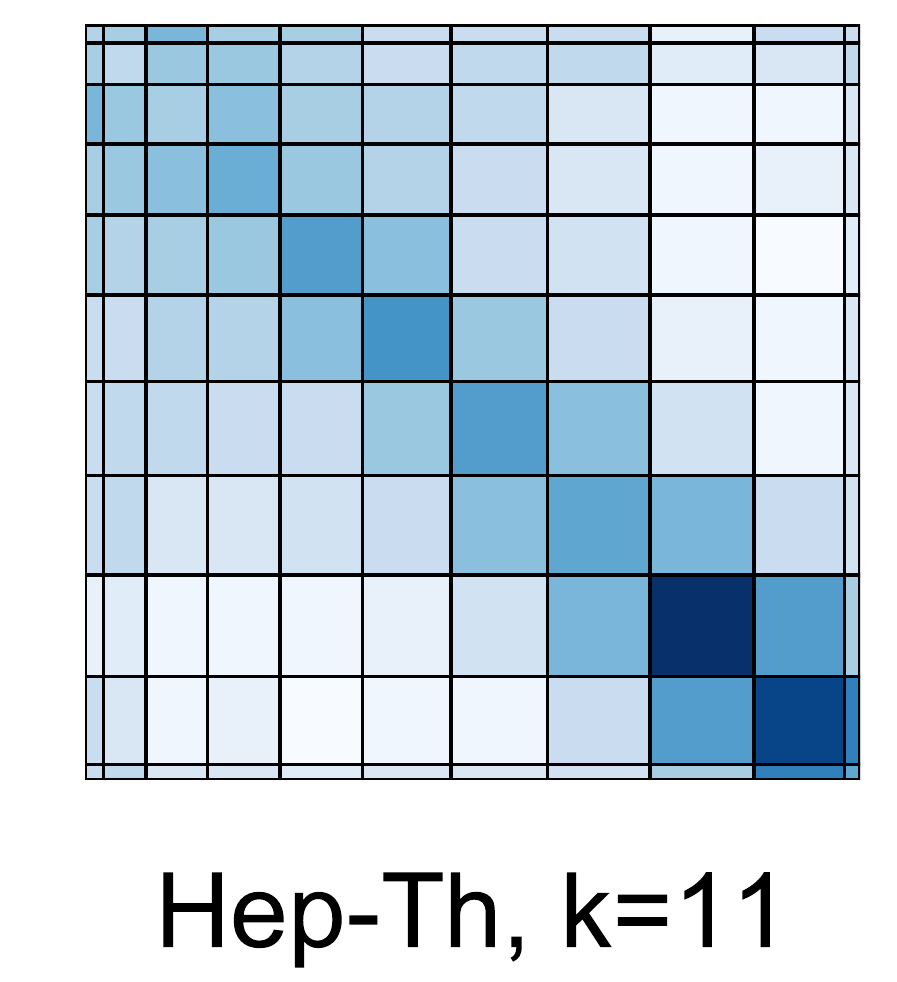}
	\caption{}
    \label{hepth_r}
\end{subfigure} 
\begin{subfigure}[t]{.12\linewidth}
	\centering
	\includegraphics[scale=0.22]{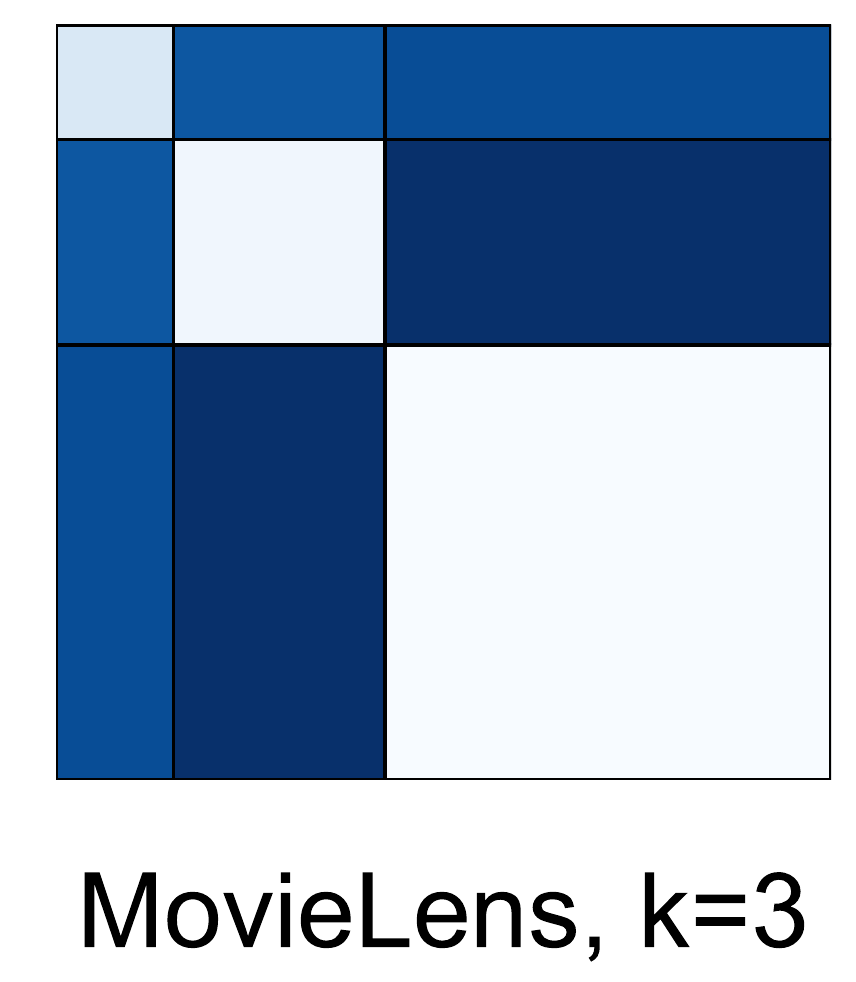}
	\caption{}
    \label{movielens_r}
\end{subfigure} 
\begin{subfigure}[t]{.12\linewidth}
	\centering
	\includegraphics[scale=0.22]{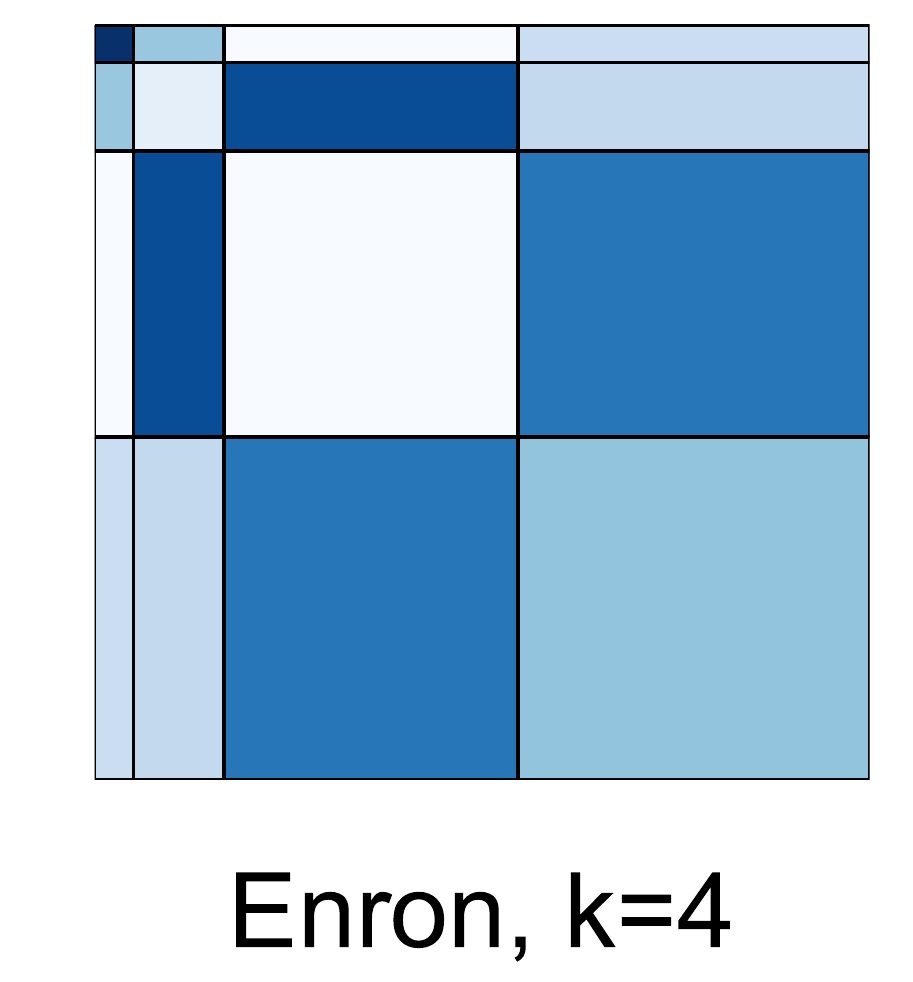}
	\caption{}
    \label{enron_r}
\end{subfigure} 
\begin{subfigure}[t]{.12\linewidth}
	\centering
	\includegraphics[scale=0.22]{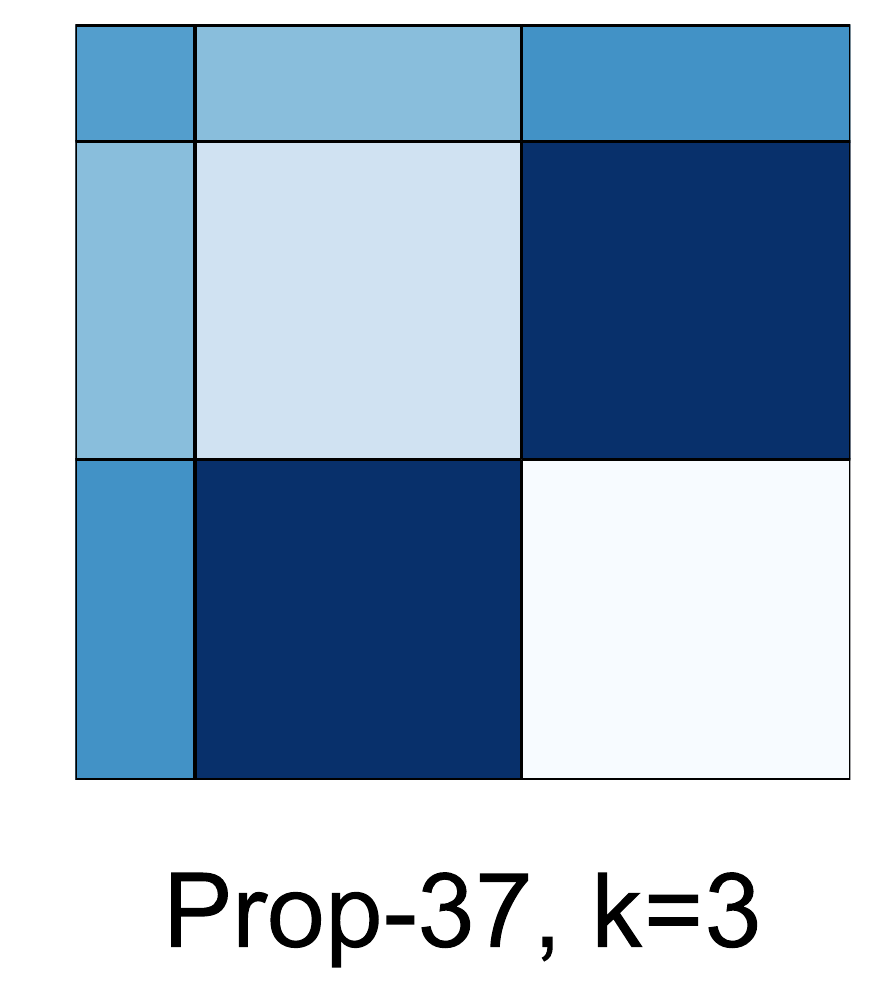}
	\caption{}
    \label{prop37_r}
\end{subfigure} 
\begin{subfigure}[t]{.12\linewidth}
	\centering
	\includegraphics[scale=0.22]{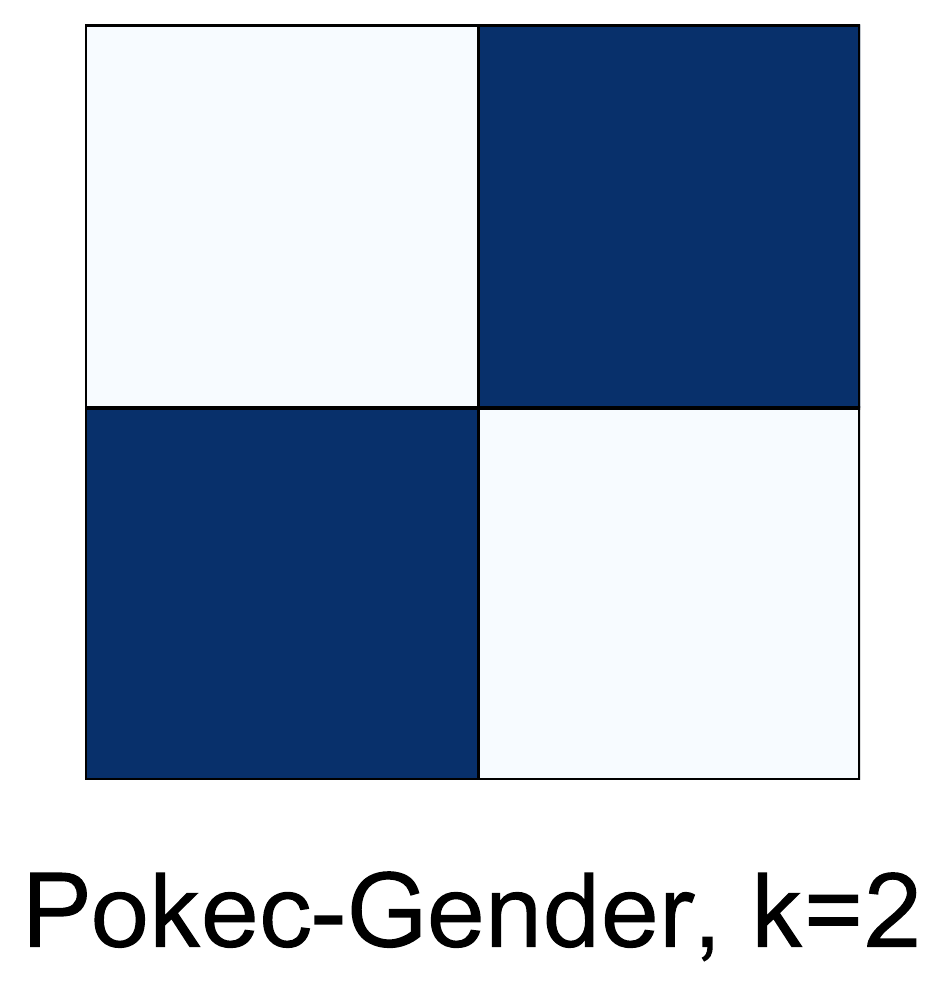}
	\caption{}
    \label{Pokec-Gender_r}
\end{subfigure} 
\begin{subfigure}[t]{.12\linewidth}
	\centering
	\includegraphics[scale=0.22]{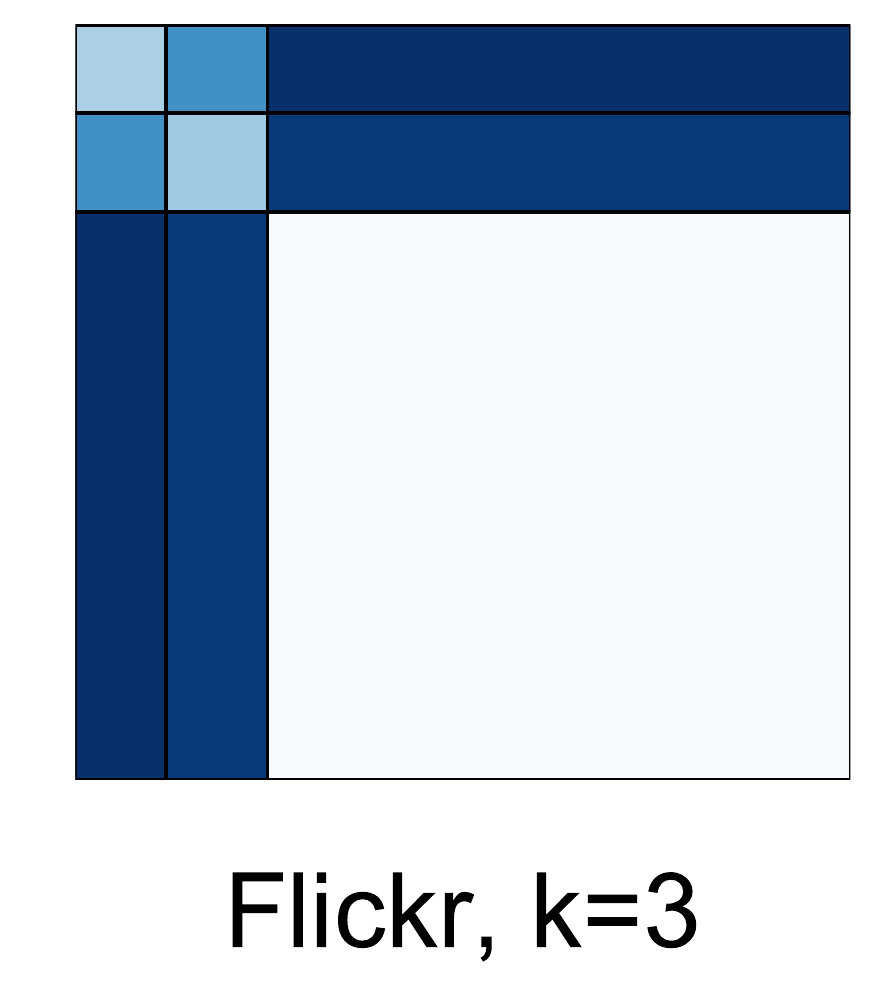}
	\caption{}
    \label{flickr_r}
\end{subfigure}

\caption{
Experiments over 8 real-world datasets
(\Cref{sec:realData}):
(a)-(h): Accuracy of end-to-end estimation and propagation.
(i)-(p): Illustration of class imbalance and heterophily in their gold standard compatibility matrices 
(darker colors represent higher number of edges): 
the first 3 show homophily, the latter 5 arbitrary heterophily.
}\label{fig:realdata}
\end{figure*}

\subsection{Scalability of Compatibility Estimation}\label{sec:scalability}
\Cref{Fig_Timing_3} 
shows the scalability of our combined methods. 
On our largest graph with 
6.6m nodes 
and 16.4m edges,
propagation takes 316 sec for 10 iterations, 
estimation
with LCE 95 sec,
DCE or DCEr 11 sec, 
and MCE 2 sec.

\begin{tcolorbox}
\begin{resultW}(\textbf{Scalability with increasing graph size})
    DCEr scales linearly in $m$ and experimentally is more than 3 orders of magnitude faster than Holdout method.
\end{resultW}
\end{tcolorbox}

\noindent
Our
estimation method DCEr 
is more than 25 times faster than inference (used by Holdout as a
subroutine) and thus comes ``for free'' as $m$ scales.
Also notice that DCE and DCEr need the same time for large graphs because of our two-step calculation:
the time needed to calculate the graph statistics $\PVecEstEC^{(\ell)}$, $\ell \in [5]$ becomes dominant; each of the 8 optimizations  of \cref{eq:DCE_energy}
then takes less than 0.1 sec \emph{for any size of the graph}.
The Holdout method for $b=1$ takes 1125 sec for a graph with 
256k edges. Thus the extrapolated difference in scalability between DCEr and Holdout is 3-4 orders of magnitude if $b=1$.
\Cref{Fig_timing_accuracy_learning_6} 
shows that increasing the number of splits $b$ for the holdout method can slightly increase accuracy at even 
higher runtime cost.

\Cref{Fig_Time_varyK_607} uses a setup identical to \Cref{Fig_End-to-End_accuracy507}
and shows that our methods also scale nicely with respect to number of classes $k$, 
as long as the graph is large and thus the graph summarization take most time.
Here, DCEr uses 10 restarts.

\begin{figure}[t]
\small
\renewcommand{\tabcolsep}{0.8mm}
\begin{tabular}{|l|r|r|r|r
			|| r|}
        \hline	
        \textbf{Dataset} 	& \multicolumn{1}{c|}{$n$} &  \multicolumn{1}{c|}{$m$}  &
        \multicolumn{1}{c|}{$d$} 	& $k$ 	& DCEr	 \\	 
        \hline
        Cora  \cite{DBLP:journals/aim/SenNBGGE08} & 2,708 & 10,858 			&
        	8.0	& 7  	& 3.33 \\ \hline  
        Citeseer  \cite{DBLP:journals/aim/SenNBGGE08} &  3,312 & 9,428 		&
			5.7	& 6 	& 1.13  \\ \hline 
        Hep-Th \cite{Gehrke:2003:OKC:980972.980992} &  27,770 & 352,807 	&
			25.4	& 11	& 10.61 \\ \hline 
        MovieLens \cite{sen2009tagommenders} & 26,850 & 336,742 			&
			25.0	& 3		& 0.07  \\ \hline
        Enron \cite{DBLP:conf/www/LiangANSP16} & 46,463 & 613,838 			&
			26.4	& 4 	& 0.20  \\	\hline 
        Prop-37 \cite{smith2013role} & 62,383 & 2,167,809 					&
			69.4	& 3 	& 0.09  \\	\hline
        Pokec-Gender \cite{pokec} &  1,632,803 & 30,622,564 				&
			37.5	& 2 	& 5.12  \\ \hline 
        Flickr   \cite{DBLP:conf/eccv/McAuleyL12} & 2,007,369 & 18,147,504 	&
			18.1	& 3	 	& 2.39  \\	
		\hline   
        \end{tabular}
\caption{Real-world dataset statistics, \cref{sec:realData}.
The last column shows runtime of DCEr (in sec).
}
\label{tbl:describe_realdata} 
\end{figure}

\subsection{Performance on Real-World Datasets}\label{sec:realData}

We next evaluate our approach over 
8 real-world graphs, described in \Cref{tbl:describe_realdata},
that have a variety of complexities : 
($i$) Graphs are formed by very different processes,
($ii$) class distributions are often highly imbalanced,
($iii$) compatibilities are often skewed by orders of magnitude, and
($iv$) graphs are so large that it is infeasible to even run Holdout.
Our 8 datasets are as follows:

\noindent
\begin{enumerate}[topsep=0mm,leftmargin=*]

\item \textbf{Cora}~ \cite{DBLP:journals/aim/SenNBGGE08} is also a citation graph containing publications from 7 categories in the area of ML 
(\emph{neural nets}, 
\emph{rule learning}, 
\emph{reinforcement learning}, 
\emph{probabilistic methods}, 
\emph{theory}, 
\emph{genetic algorithms}, 
\emph{case based}).

\item \textbf{Citeseer}~ \cite{DBLP:journals/aim/SenNBGGE08} contains 3264 publications from 6 categories in the area of Computer Science. The citation graph connects papers from six different classes (\emph{agents}, \emph{IR}, \emph{DB}, \emph{AI}, \emph{HCI}, \emph{ML}).

\item \textbf{Hep-Th} ~\cite{DBLP:conf/eccv/McAuleyL12} is the High Energy Physics Theory publication network is from arXiv e-print and covers their citations. The node are labeled based on one of 11 years of publication (from \emph{1993} to \emph{2003}).

\item \textbf{MovieLens}~\cite{sen2009tagommenders} is from a movie recommender system that connects 3 classes: \emph{users}, \emph{movies},
and assigned \emph{tags}.

\item \textbf{Enron}~\cite{DBLP:conf/www/LiangANSP16} 
has 4 types of nodes: \emph{person}, \emph{email address}, \emph{message} and \emph{topic}. 
Messages are connected to topics and email addresses; people are connected to email addresses.

\item \textbf{Prop-37}~\cite{smith2013role} comes from a California ballot
initiative on labeling genetically engineered foods. It is based
on Twitter data and has 3 classes: \emph{users}, \emph{tweets}, and \emph{words}.

\item \textbf{Pokec-Gender}~ \cite{pokec} is a social network graph connecting people (\textit{male} or \textit{female}) to their friends. 
More interaction edges exist between people of opposite gender.

\item \textbf{Flickr}~\cite{DBLP:conf/eccv/McAuleyL12} connects \emph{users} to \emph{pictures} they upload and other users pictures in the same \emph{group}. Pictures are also connected to groups they are posted in.

\end{enumerate}

\begin{tcolorbox}
\begin{resultW}(\textbf{Accuracy on real datasets})
DCEr consistently labels nodes with accuracy $\pm 0.01$  compared to true compatibility matrix for $f<10\%$ and $\pm 0.03$ for $f>10\%$ averaged across datasets, basically indistinguishable from GS. 
\end{resultW}
\end{tcolorbox}

Our experimental setup is similar to before: 
we estimate $\HVec$ on a random fraction $f$ of labeled seed sets and repeat many times.
We retrieve the Gold-Standard (GS) compatibilities
from the
relative label distribution on the fully labeled graph
(if we know all labels in a graph, 
then we can simply ``measure'' the relative frequencies of classes between neighboring nodes).
The remaining nodes are then labeled with 
LinBP, 10 iterations, $s=0.5$, as suggested by \cite{DBLP:journals/pvldb/GatterbauerGKF15}.

The relative accuracy for varying $f$ can look quite differently for different networks.
Notice that our real-world graphs show a wide variety of
($i$) label imbalance, 
($ii$) number of classes ($2 \leq k \leq 12$), 
($iii$) mixes of homophily and heterophily, 
and ($iv$) a wide variety of skews in compatibilities.
This variety can be seen in our illustrations of the gold-standard compatibility matrices 
(\cref{cora_r,citeseer_r,hepth_r,movielens_r,enron_r,prop37_r,Pokec-Gender_r,flickr_r}):
we clustered all nodes by their respective classes,
and shades of blue represent the relative frequency of edges between classes.
It appears that the accuracy of LinBP for varying $f$
can be widely affected by combinations of graph parameters.
Also, for Movielens~\cref{fig:movielens}, 
choosing $\lambda=10$ worked better for DCEr in the sparse regime ($f<1\%$),
while $\lambda=1$ worked better for $f>1\%$.
This observation appears consistent with our understanding of $\lambda$ 
where larger values can amplify weaker distant signals,
yet smaller $\lambda$ is enough to propagate stronger signals. 
Fine-tuning of $\lambda$ on real datasets remains interesting future work.
However, all our experiments consistently show that our methods for learning compatibilities 
robustly compete with the gold-standard and beat all other methods, especially for sparsely labeled graphs.
This suggests that our approach renders any prior heuristics obsolete.

\introparagraph{Scalability with increasing number of classes $k$} 
We know from \cref{sec:complexityanalysis} that  joint complexity of our two-step compatibility estimation is $\O(mk + k^4r)$. 
For very large graphs with small $k$, the first factor stemming from graph summarization is dominant and estimation runs in $O(mk)$, which is faster than propagation.
However, for moderate graphs with high $k$ (e.g. Hep-Th with $k=11$), the second factor can dominate since  calculating the Hessian matrix has $O(k^4)$ complexity. Thus for high $k$ and small graphs, our approach will not remain faster than propagation, but in most practical settings estimation will act as a computationally cheap \textit{pre-processing} step.

%% file: sections/conclusions.tex
\section{Conclusions}\label{sec:conclusions}

Label propagation methods that rely on arbitrary compatibilities between node types %
require the class-to-class compatibilities as input.
Prior works assume these compatibilities as given by domain experts.
We instead propose methods
to {accurately} learn compatibilities from sparsely labeled graphs {in a fraction of the time} it takes to later propagate the labels,
resolving the long open question of ``Where do the compatibilities come from?''
The take-away for practitioners is that \emph{prior knowledge of
compatibilities is no longer necessary}  
and estimation can become a cheap initial step before labeling.

Our approach is able to \emph{amplify signals from sparse data} by 
leveraging certain \emph{algebraic properties} (notably the distributivity law) in the update equations of linear label propagation algorithms
(an idea we call \emph{algebraic amplification}).
We start with linearized belief propagation~\cite{DBLP:journals/pvldb/GatterbauerGKF15},
which is itself derived from a widely used inference method (belief propagation)
by applying certain algebraic simplifications (namely linearizations),
and we complement it with a method for parameter estimation.
Thus our estimation method
\emph{uses the same approximations used for inference, also for learning the parameters}.
In general, linear methods 
have computational advantages (especially during learning),
are easier to interpret,
have fewer parameters (which helps with label sparcity),
have fewer hyperparameters (which simplifies tuning),
and also have favorable algebraic properties that are crucial to our approach.
It remains to be seen how deep neural networks (NN) 
can be adapted to such sparse data.

%% file: sections/nomenclature.tex
\section*{Nomenclature}\label{sec:nomenclature}
\vspace{-2mm}
\begin{table}[ht]
    \centering%
    \small
    \begin{tabularx}{\linewidth}{  @{\hspace{0pt}} >{$}l<{$}  @{\hspace{2mm}}  X @{}}  
    \hline
	n, n_L	& number of nodes, or labeled nodes \\
	m			& number of edges \\	
	k			& number of classes, $k= |L|$ \\
	k^*			& number of free parameters in $\HVec$: $\frac{k(k-1)}{2}$ \\	
    b           & number of splits evaluated by baseline Holdout method\\
	f			& fraction of labeled nodes \\
	\vec I		& $n \times n$ diagonal identity matrix \\
	\WVec		& $n \times n$ 
					symmetric adjacency matrix (data graph)
                    \\
	\vec D		& $n \times n$ diagonal degree matrix \\		
	\EVec		& $n \times k$ a priori label matrix \\
	\BVec		& $n \times k$ inferred label matrix \\
	\HVec		& $k \times k$ compatibility 
					matrix	\\
	\epsilon 	& Scaling factor for $\HVec$ \\	
	\WVecEC^{(\ell)}	
				& $n \times n$ non-backtracking (NB) walks for path length $\ell$  \\				
    \PVecEst    & $k \times k$ observed statistics matrix\\
	\PVecEstEC^{(\ell)}		
				& $k \times k$ observed statistics matrix for path length $\ell$\\					
	\rho(\EVec)		& spectral radius of a matrix $\EVec$ \\							
	\vec X_{i:}, \vec Y_{:j}
					& $i$-th row vector of $\vec X$, $j$-th column vector of $\vec Y$	\\
	||\vec X ||		& Frobenius norm of matrix $\vec X$: 
						$\sqrt{\sum_{i,j} |X_{ij}|^2}$	\\
    | \Vec Y |_{\textrm{row}} & row-normalized matrix $\Vec Y$, defined as
    $\mathrm{diag}(\vec Y \vec 1)^{-1} \vec Y$. \\
    \hline
    \end{tabularx}
\end{table}

%% file: sections/appendix.tex
\section{Proofs SECTION \ref{SEC:PROPERTIES}}

\subsection{Proof \cref{prop:nonCenteredLinBP}}

We show this result by first proving two simple lemmas.

\begin{lemma}[Modulation by a row-constant matrix]\label{lem:modulatingRoConstant}\label{lem:modulatingResiduals}
Linearly mapping a row-vector of sum $s$ with a matrix with constant row-sum $r$ leads to a vector with sum $r s$.
\end{lemma}

\begin{proof}
Let $\eVec$ be a $k$-dimensional row vector $[\e_1, \ldots, \e_k]$ whose entries
sum to $s$ ($\sum_{i=1}^k \e_i = s$), and $\HVec$ be a $[k \times k]$ matrix
with constant row-sum $r$ ($\sum_{i=1}^k \H_{ji} = r$). Consider the
$k$-dimensional row vector $\bVec$ resulting from ``modulating'' (linearly
transforming) $\eVec$ by $\HVec$. We show that the entries of $\bVec$ also sum
to $r s$:
\begin{align*}
	\sum_i \b_i	&= \sum_i\sum_\ell \e_\ell \H_{\ell i} 	
				= \sum_\ell \e_\ell \sum_i \H_{\ell i}
					= r \sum_\ell \e_\ell = r s 
\end{align*}
Notice that the same also applies to column-vectors modulated ``from the left'' by column-stochastic matrices. This follows immediately from the transpose:
$\bVec^\transpose 	= \HVec^\transpose  \eVec^\transpose$.
\end{proof}

\begin{lemma}[Modulating residual vectors]\label{lem:modulatingResiduals}
Consider a residual row vector modulated by an arbitrary matrix.
Adding or subtracting a constant value to each entry of the matrix does not change the resulting modulated vector.
\end{lemma}

\begin{proof}
Let $\eVec = [\e_1, \ldots, \e_k]$ be a $k$-dimensional residual row vector  ($\sum_{i=1}^k \e_i = 0$), and $\HVec$ to be a $[k \times k]$ matrix. Consider the $k$-dimensional row vector $\bVec$ resulting from transforming $\eVec$ by $\HVec$, 
and the row vector $\bVec'$ resulting from transforming $\eVec$ by $\HVec'$ where
$\H_{ij}' = \H_{ij} + \delta$. It is easy to show that that $\bVec' = \bVec$:
\begin{align*}
	\bVec' 		&= \eVec \HVec' 
				= \eVec (\HVec + \delta \vec 1_{k \times k}) 
	 			= \bVec + \delta \eVec \vec 1_{k \times k}
					=\bVec
\end{align*}
Notice that $\eVec$ needs to be a residual vector (i.e., centered around 0) for this equivalence to hold.
\end{proof}

\begin{proof}[\cref{prop:nonCenteredLinBP}]
We will show that each entry in the resulting label distribution changes by a
constant that depends only on the row. Hence, the relative order will not change.
Consider \cref{eq:updateEquation} and replace $\ECenterVec$ on the right with $\EVec$. This change adds a constant value to each entry.
Next replace $\BCenterVec$ in $\BCenterVec \HCenterVec$ on the right with $\BVec$:  $\BVec\HCenterVec$.
According to 
\cref{lem:modulatingResiduals}, the product will remain the same.
Next replace $\HCenterVec$ with $\HVec$: $\BVec\HVec$.
According to \cref{lem:modulatingRoConstant}, 
the constant $\frac{1}{k}$ gets added to each entry in the resulting matrix. 
Furthermore, the left multiplication with $\WVec$ will lead to a constant being
added to each row $i$; this constant depends on the degree of the node $i$. It
follows that each subsequent iteration adds constants whose magnitude depend on
the row number. This addition leaves the relative orders of classes unchanged
and proves our claim.
\end{proof}

\subsection{Proof SECTION \ref{prop:LinBP_energy}}

\begin{proof}
First notice that if the update equations~\cref{eq:updateEquation} converge, then they converge towards the fixed point defined by the following equation system:
\begin{align}
\BVec &= \EVec 
		+ \WVec \BVec \HVec 
		\label{eq:LinBP_eqSystem}
\end{align}

Observe now that if \cref{eq:LinBP_eqSystem} holds, then the energy function
\cref{eq:LinBP_lossfunction} becomes zero, as
$\BVec - \EVec - \WVec \BVec \HVec 
		 = 0$.
At the same time, the energy function is quadratic and can never become
negative. Thus its possible minimum is zero. At zero, \cref{eq:LinBP_eqSystem} holds.
Thus, minimizing 
\cref{eq:LinBP_lossfunction}
also leads to the solution after convergence of the update equations.
\end{proof}

\section{Proofs \ref{sec:learningHeterophily}}

\subsection{Proof \cref{th:consistency}}

\begin{proof}
We show our theorem for $f=1$, i.e.\ when calculating the statistics over the fully labeled graph. 
The argument then extends to 
$f<1$, which in expectation results in a fraction $f^2$ of edges.
We illustrate for $\ell=2$; induction on the path length then generalizes to any $\ell$. 
We focus on variant 1 (\cref{eq:variant1})
and make no assumption about the graph generation:
We are only given the matrix $\vec M$ and assume any graph that fulfills it is equally likely.

Consider a graph where
$\PVec 
= 
| \Vec M |_{\textrm{row}}
$
is symmetric and doubly stochastic.
From that property follows that
$\PVec^2
= 
| \Vec M^2 |_{\textrm{row}}$.
Now recall that
$\vec M^2 = (\EVec^\transpose \WVec \EVec)(\EVec^\transpose \WVec \EVec)$,
which can be written as
$\EVec^\transpose (\WVec \EVec \EVec^\transpose \WVec) \EVec$.
In contrast, our estimators
$\PVecEstEC^{(2)} = | \vec M_{\NB}^{(2)}  |_{\textrm{row}}$
and
$\PVecEst^{(2)} = | \vec M^{(2)}  |_{\textrm{row}}$
use
$\vec M_{\NB}^{(2)} 
= \EVec^\transpose \WVecEC^{(2)} \EVec
$
and
$\vec M^{(2)} 
= \EVec^\transpose \WVec^{2} \EVec$, respectively.
Our goal reduces thus to 
analyzing the bias in
(1) $\WVecEC^{(2)}$
and
(2) $\WVec^{2}$
as estimators
for $\WVec \EVec \EVec^\transpose \WVec$.

(1) First notice that $\EVec \EVec^\transpose$ is an $n \times n$ matrix that connect nodes with identical classes.
In other words, $\WVec \EVec \EVec^\transpose \WVec$ adds entries at $(u,v)$ for every pair $(x,y)$ of neighbors of respective nodes $(u,v)$ that have the same class.
For $\WVecEC^{(2)}$ to be an unbiased sample of $(\WVec \EVec \EVec^\transpose \WVec)$,
each of the latter entries would have to have equal probability of being sampled. 
Those entries also include back-tracking paths from a node to itself on the diagonal.
Now take $\WVecEC^{(2)}$ and consider an entry $(u,v)$ representing a non-backtracking path of length 2.
An $(u,w)$ with nodes $u$ and $w$ having classes $i$ and $j$, respectively, is followed by another edge $(w, v)$.
Since we assume no other information about the graph, the class $v$ is sampled uniformly from the vector $M_{:j}-\vec e^{i}$ where $\vec e^{i}$ represents the unit vector with entry 1 at position $i$ and corrects for the fact that the path cannot back-track.
We thus consistently under-sample the diagonals of $\vec M^2$, 
and the order of the bias $\delta$ is in the order of 1 over the number of edges $m$:
$\delta = O(1/m) = O(1/nd)$.
It follows that increasing either $n$ or $d$ decreases the bias.
In particular, from the law of large number it follows that $\lim_{n \rightarrow \infty} \PVecEstEC^{(\ell)} = \HVec^\ell$. 

(2) In contrast, $\PVecEst^{(2)}$ 
uses
$| \WVec^2 |_{\textrm{row}}$
as estimator for
$| \WVec \EVec \EVec^\transpose \WVec |_{\textrm{row}}$.
Whereas $\WVecEC^{(2)}$ ignored some diagonal entries, $\WVec^2$ \emph{always} always includes those corresponding to back-tracking paths, 
and thus overestimates the diagonals of $\PVec^2$. 
Consider a node of class $i$ with degree $d$; there are $d^2$ paths of length 2 passing through that node, $d$ of which are backtracking. Those $d^2$ paths are a biased sample of corresponding $\texttt{sum}(\vec M_{i:})^2$ paths
in $\WVec \EVec \EVec^\transpose \WVec$.
Thus, the positive bias for diagonals is in the order of 
$\delta = O(1/d)$.
It follows that the bias is irrespective of $n$, and thus:
$\lim_{n \rightarrow \infty} \PVecEstEC^{(\ell)} \neq \HVec^\ell$.
Only for increasing $d$, the estimator becomes consistent: 
$\lim_{d \rightarrow \infty} \PVecEstEC^{(\ell)} = \HVec^\ell$.
\end{proof}

\subsection{Proof \cref{prop:nonbacktracking}}

\begin{figure}[t]
\centering
\begin{subfigure}[b]{.34\linewidth}
	\centering
	\includegraphics[scale=0.32]{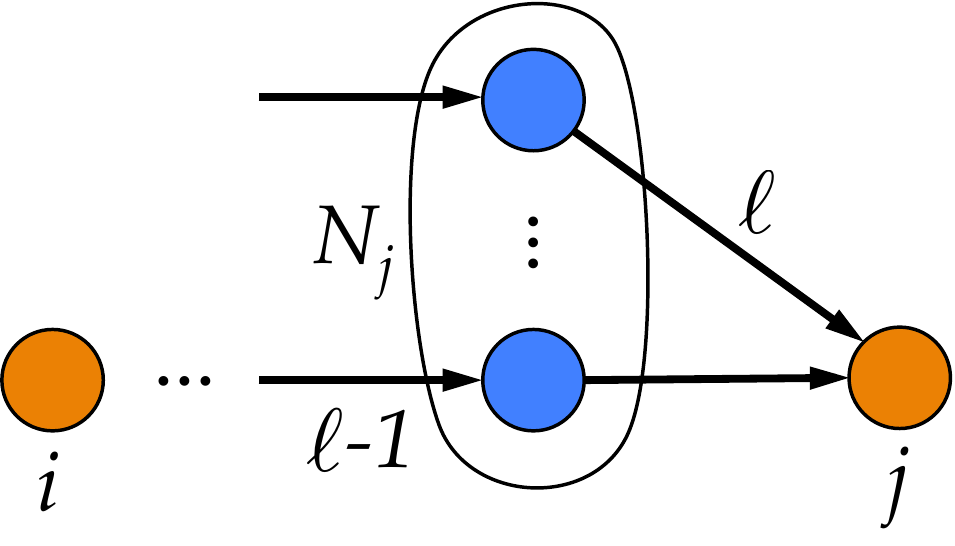}
	\caption{}	
	\label{Fig_NonBacktrackingMatrix_a}
\end{subfigure}
\hspace{1mm}
\begin{subfigure}[b]{.61\linewidth}
	\centering
	\includegraphics[scale=0.32]{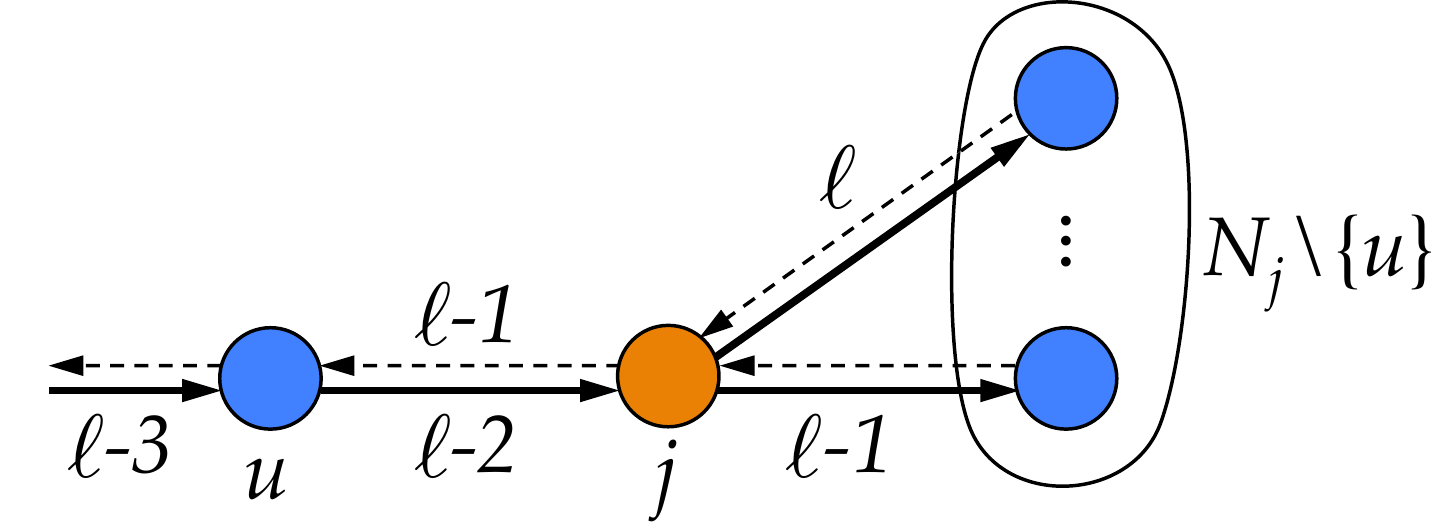}
	\caption{}	
	\label{Fig_NonBacktrackingMatrix_b}
\end{subfigure}
\caption{Illustrations for proof of \cref{prop:nonbacktracking}.}
\label{Fig_NonBacktrackingMatrix}
\end{figure}

\begin{proof}
	We prove by induction on $\ell$, the length of the paths.
	First notice that $\WVecEC^{(1)} = \WVec$, because there is no back-tracking path of length 1.
	Also, $\WVecEC^{(2)} = \WVec^2 - \vec D$, because $\WVec^2$ includes the only feasible back-tracking paths of length 2 on its diagonal $\W_{ii} = d_i$ where $d_i$ is the degree of node $i$.
	
	Next assume the recurrence relations holds for paths of length up to $\ell-1$. 
	Then the number of non-backtracking paths of length $\ell$ that start at node $i$ and arrive at node $j$ consists of two parts:
	
	(1) The number of paths that 
	arrive at node $j$ at step $\ell$ 
	and that have not backtracked 
	previously
	are
	$\sum_{u \in N_j} \WEC_{iu}^{\ell-1}$. 
	Here $N_j$ stands for the neighbors of $j$ 
	(see \cref{Fig_NonBacktrackingMatrix_a}). 
	
	(2) Among those, we need to subtract the number of walks that backtrack at the $\ell$-th step. 
	Those paths need to pass through $j$ at step $\ell-2$.
	Consider such a path that passes through $u \in N_j$ and arrives at $j$ at step $\ell-2$ 
	(see \cref{Fig_NonBacktrackingMatrix_b}). 
	At step $\ell-1$ 
	there are $d_j-1$ paths that continue from node $j$ to neighbors (because the walk that would backtrack to $u$ is forbidden) and then come back at step $\ell$.
	Thus we need to subtract $\WEC_{ij}^{(\ell-2)} (d_j-1)$ such paths. 
	This gives us
$\WEC_{ij}^{(\ell)} = \sum_{u \in N_j} \WEC_{iu}^{(\ell-1)} - \WEC_{ij}^{(\ell-2)} (d_j-1)$,
which leads to our recurrence relation \cref{eq:recurrence} in matrix notation.	
\end{proof}

\subsection{Proof \cref{prop:complexitystatistics}}

\begin{proof}
	Correctness follows from \cref{prop:nonbacktracking}
	and $\vec N_{\NB}^{(\ell)} = \WVec_{\NB}^{(\ell)} \EVec$.

	The complexity is derived as follows: 
	(1) Assume $\vec W$ is a sparse $n \times n$ matrix containing $2m$ entries, 
	and $\vec N$ is a dense $n \times k$ matrix.
	Then calculating $\vec W \vec N$ is $\O(mk)$ and each iteration for $\vec N_{\NB}^{(\ell)}$ takes $\O(mk)$.
	(2) Since $\EVec$ has $fn$ entries (recall that $f$ is the fraction of labeled nodes, and that $\EVec$ has maximal one entry per row),
	calculating $\vec M_{\NB}^{(\ell)}$ takes $\O(fnk)$.
	(3) $\PVecEstEC^{(\ell)}$ takes $\O(k^2)$.
	(4) Since $m > n > k$, the dominating term is $\O(mk)$.
	(5) Finally, all these calculations are performed $\ell_{\max}$ times, which leads to $\O(mk \ell_{\max})$
\end{proof}

\subsection{Proof \cref{prop:gradient}}

\begin{proof}
A challenge in deriving the gradient results the doubly stochastic constraints that we need to obey.
We do this by using the chain rule~\cite{MatrixCookbook12}:
\begin{align*}
	\frac{\partial}{\partial \hVec}E(\HVec) 
		= \tr\big( \big( \frac{\partial E(\HVec)}{\partial \HVec} \big)^\transpose
	  	\frac{\partial \HVec}{\partial \hVec} \big)
\end{align*}
and then dealing with each term in turn. 
To slightly simplify the notation, we substitute below
$\vec Z \define \PVecEstEC^{(\ell)}$.
First, 
$\frac{\partial E(\HVec)}{\partial \HVec}$:
\begin{align*}
	\frac{\partial}{\partial \HVec}	
	|| \HVec^\ell - \vec Z ||^2
	&= \frac{\partial}{\partial \HVec}
	\tr\big( 	(\HVec^\ell - \vec Z) 
				(\HVec^\ell - \vec Z)^\transpose \big)	\\
	&= 2 \ell \HVec^{2 \ell-1} 
	- 2 \sum_{r=0}^{\ell-1} \HVec^r \vec Z \HVec^{\ell-r-1}
\end{align*}
which leads to 
the first term
with $\vec G \define \frac{\partial E(\HVec)}{\partial \HVec}$.

To calculate the second term, we need to observe that $\HVec$ has significant structure. 
We thus define a ``structure matrix''
$\vec S^{ij} \define
\frac{\partial \vec H}{\partial H_{ij}}$
which encodes the 
relative dependencies of entries on the independent parameters.
If $\vec H$ is symmetric then straight-forward calculation~\cite{MatrixCookbook12} leads to
$\vec S^{ij} = \vec J^{ij} + \vec J^{ji}$ 
if $i \neq j$,
and 
$\vec J^{ij}$
if $i = j$.
The calculation for our symmetric, doubly stochastic matrix is more intricate as it
needs to use our parameterization in \cref{eq:H_parameterization}.
Several transformations lead to 
the formulation in \cref{prop:gradient}
with all other entries $\vec S^{ij}$ being 0 for $i=k$ or $j=k$.
The gradient of our free parameters then corresponds to the entries with index $i \leq j, j\neq k$ of the product
$\vec S \vec G$.
\end{proof}

\section{Examples moved to appendix}

\subsection{Illustration for propagating frequency distributions (\cref{sec:propFrequency})}

We illustrate next that -- while the final label distribution may be identical --  one version may converge while the other one may not.
Thus {convergence} of the relative values in the update equations vs.\ convergence of the label assignment are two
{separate} issues. 
In \cref{sec:LCE} we show how interpreting ``compatibility propagation'' as ``propagating frequency distributions'' 
helps us understand the nature of the update equations and gives a simple and intuitive interpretation to our approach to compatibility estimation.

\begin{figure}[t]
\centering
\begin{subfigure}[b]{.62\linewidth}
	\centering
	\includegraphics[scale=0.38]{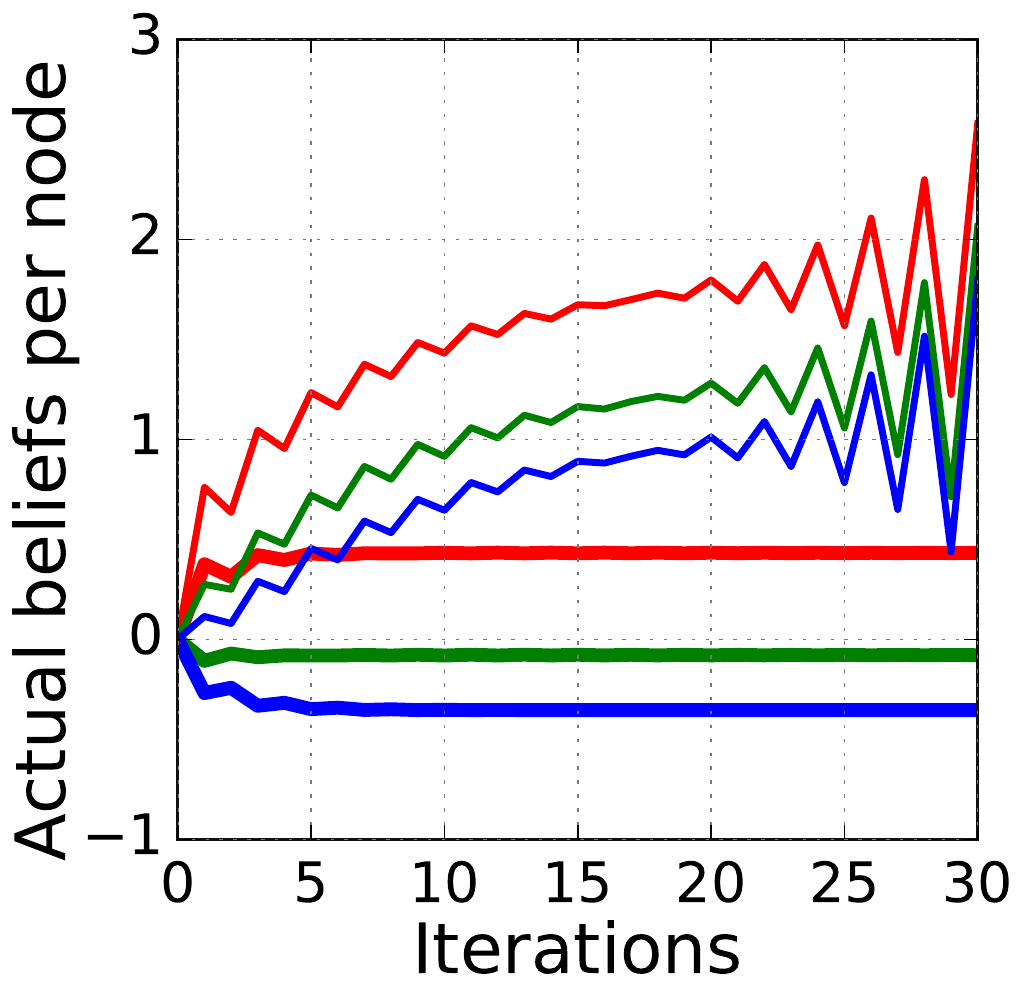}
\end{subfigure}
\caption{\Cref{ex:frequencyNonConvergence} gives an example where the values from LinBP converge when using $\HCenterVec$ (3 bold  lines),
but diverge when using $\HVec$ (3 thin lines) for one node in the graph. 
Nevertheless, in both cases, the final labels are \emph{identical} (the red labels gets highest weight in each iteration).
}
\label{Fig_WithoutCentering1}
\end{figure}

\begin{example}[Non-convergence]\label{ex:frequencyNonConvergence}
We use the following compatibility matrix
$\HVec = \left[\begin{smallmatrix}
	0.1 & \,0.8 & \,0.1  \\
	0.8 & \,0.1 & \,0.1  \\
	0.1 & \,0.1 & \,0.8  \\
\end{smallmatrix}\right]$
to label a partially labeled graph.
The spectral radius of $\rho(\HVec)$ is 1, 
whereas the spectral radius of its centered version is 
$\rho(\HCenterVec) = 0.7$.
We use a scaling factor $\epsilon$ so that the convergence parameter is $s=0.95$ for the centered version. 
This $\epsilon$ translates into $s\approx 1.18$ for the uncentered version.
\Cref{Fig_WithoutCentering1} shows the belief vector $\bVec$ 
(each color represents one label)
for one particular node, for using LinBP either with 
$\HCenterVec$ (bold colored lines), or with $\HVec$ (thin colored lines). 
While the uncentered version does not converge, 
\emph{at each iteration} the top beliefs are identical.
\end{example}

\subsection{Illustration for MCE (\cref{sec:MCE})}

\Cref{Fig_TwoSteps_MHE} illustrates the idea behind Myopic Compatibility Estimation:
we first (1) summarize the partially labeled graph into a small summary, and then (2) use this summary to perform the optimization.

\begin{figure}[t]
\centering
\includegraphics[scale=0.45]{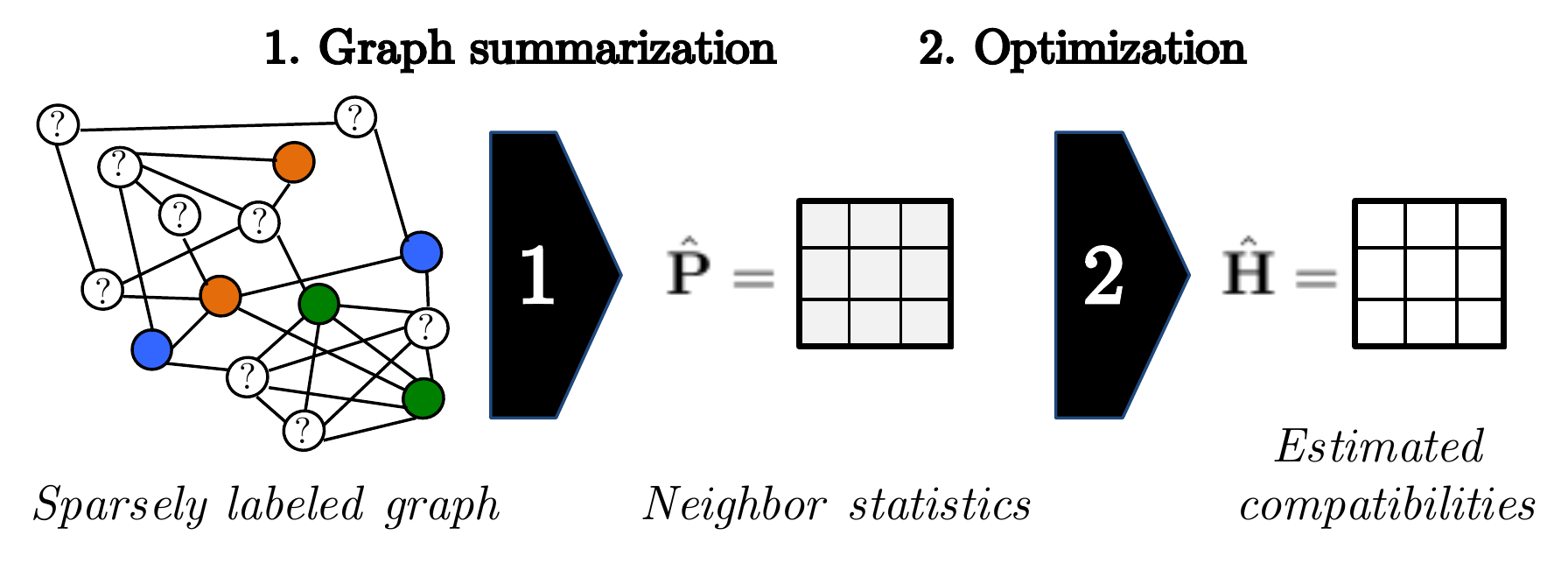}
\caption{\Cref{sec:MCE}: Myopic Compatibility Estimation (MCE)
first summarizes the graph into a normalized neighbor statistics matrix $\PVecEst$, and
then find the closest compatibility matrix $\HVecEst$.
}
\label{Fig_TwoSteps_MHE}
\end{figure}

\section{Additional related work}

To emphasize the differences with more recent work that leverages deep neural networks, such as \cite{DBLP:conf/iclr/KipfW17,Hamilton:2017:IRL:3294771.3294869,DBLP:conf/aaai/MooreN17},
we analyze additional problem dimensions:
(1) whether a work uses only the graph or uses additional information;
(2) whether it uses a linear model or a NN. 
The latter make it more difficult to understand the reason for assigning any particular label to a node due to the many 
(often thousands) of trained parameters.
(3) whether a paper leverages only neighbors of nodes or longer paths (some NN methods do learn on several random walks and thus include a local neighborhood of given nodes);
and (4) whether it can be trained on a sparsely labeled network or requires either a larger amount of labeled nodes 
within the same network, or another network with labeled nodes all together (e.g. for paper that use transfer learning 
and thus learn \emph{between} different networks).

We are not aware of any method that can labels nodes, based only the information in a graph with very sparsely labeled nodes 
(e.g. 0.1\% as in our experiments) and \emph{no additional information}.

\section{Additional experiments}

We bring here additional experiments that did not have space in the main paper.
In particular, we 
(1) illustrate the advantage our approach against prior work that suggested to use heuristics instead  (\cref{sec:heuristics});
(2) give experimental results comparing the estimated matrices against the true neighbor frequency distributions from a fully labeled graph; and
(3) investigate and try to explain how our method can have better label accuracy in certain scenarios than the actual neighbor frequency distribution.

\subsection{Comparison against Heuristics}\label{sec:heuristics}

Prior work \cite{DBLP:conf/pkdd/KoutraKKCPF11,Eswaran:2017:ZBP:3055540.3055554, DBLP:journals/pvldb/GatterbauerGKF15}
suggests simple heuristics for guessing the entries of a compatibility matrix $\HVec$.
These heuristics are minor variants of the following general approach:
\begin{enumerate}
\item Assume the compatibility matrix $\HVec$ has two types of entries, a high and a low value. Call those $H$ and $L$, respectively.
\item Assume that we can correctly guess the position of those entries in $\HVec$ based on some domain knowledge.
Equivalently, assume that we can glance at the golden standard and correctly remember the positions for $H$ and $L$ in $\HVec$.
\item Assume the difference between them is some value $\epsilon$. 
Choose $\epsilon$ so that the update propagation is guaranteed to convergence.
\end{enumerate}

We tested the performance of this approach for two real data sets
and show in
\cref{Fig_End-to-End_accuracy_realData_406_movielens_accuracy_plot} that this 
heuristic has merit: 
\emph{if} we can guess the positions correctly, \emph{and if} indeed all values in $\HVec$ are approximately represented by only two values, then
the heuristics can
perform about as well as our method on. 
However,
we see in \cref{Fig_End-to-End_accuracy_realData_307_prop37_accuracy_plot} that
this is not generally the case: The labeling from the heuristic matrix is barely visible on the plot and
does no better than random assignment. Observe that the Prop-37 compatibility matrix
described in \cref{prop37_r} has a less discrete distribution of compatibilities
than the MovieLens compatibility matrix in \cref{movielens_r}. The heuristic
approach is unable to capture this information due to its binary
``High''/``Low'' estimation process, and thus performs poorly in labeling.

\begin{figure}[t]
    \centering

\begin{subfigure}[b]{0.95\linewidth}
\setlength{\tabcolsep}{1.0pt}
\scalebox{0.96}{
\hspace{4mm}
\begin{tabular}{c c c c c}
       \multicolumn{2}{c}{MovieLens} &
	   &
       \multicolumn{2}{c}{Prop-37}
\\
	$ \left[\begin{smallmatrix}
           0.08 & 0.45 & 0.47  \\
           0.45 & 0.02 & 0.53  \\
           0.47 & 0.53 & 0.0  \\
        \end{smallmatrix}\right]$
&
	$\left[\begin{smallmatrix}
           L&   H&   H \\
           H&   L&   H \\
           H&   H&  L \\
       \end{smallmatrix}\right]$
&
\phantom{xx}
&
       $\left[\begin{smallmatrix}
               0.35 & 0.26  & 0.38\\
               0.26 & 0.12  & 0.61\\
               0.38 & 0.61  & 0.0  \\
    \end{smallmatrix}\right]$
&
       $\left[\begin{smallmatrix}
               H & L  & H\\
               L & L  & H\\
               H & H  & L  \\
    \end{smallmatrix}\right]$	
\end{tabular}
}
\caption{Heuristic approximation of $\HVec$ by two values $H$ and $L$.}
\end{subfigure}

\begin{subfigure}[b]{.48\linewidth}
	\centering
	\includegraphics[scale=0.38]{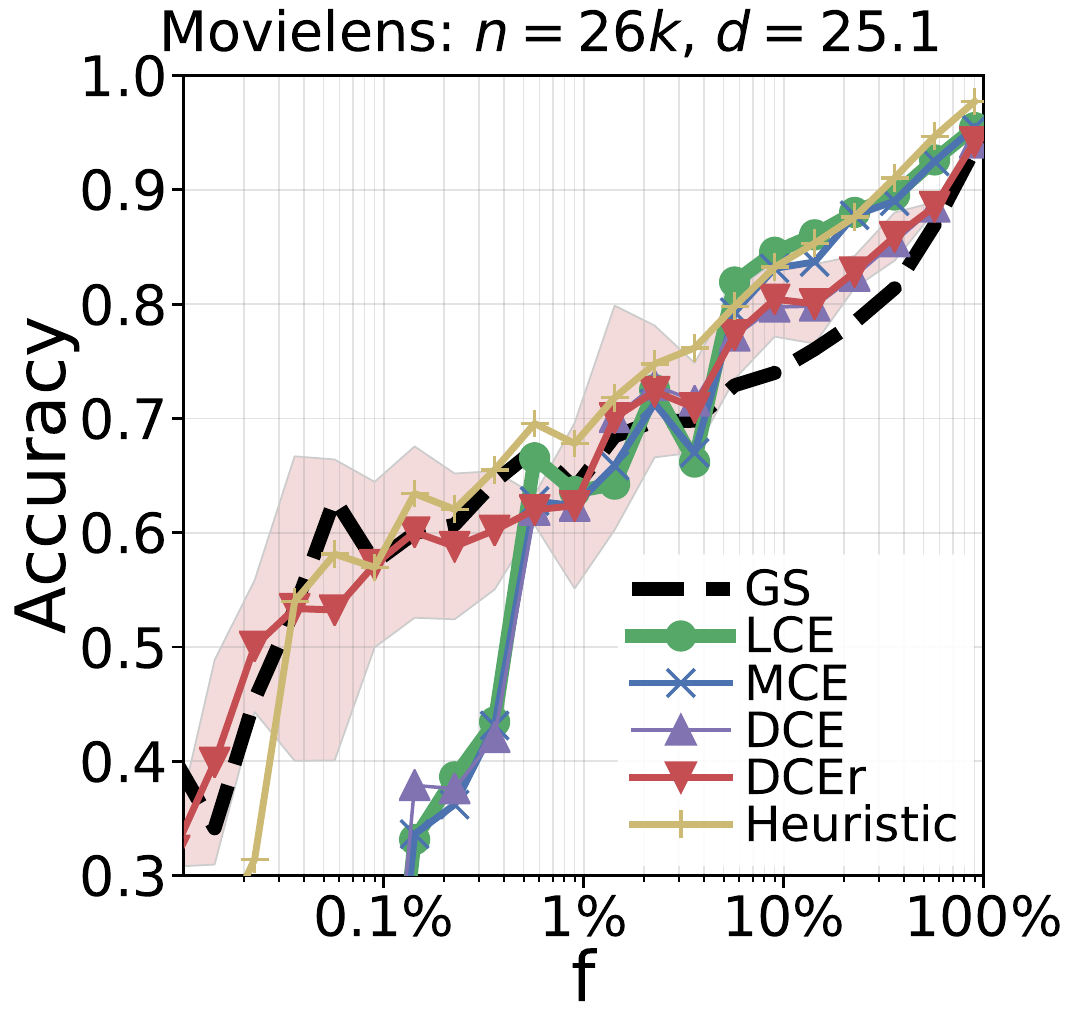}
	\caption{MovieLens}
	\label{Fig_End-to-End_accuracy_realData_406_movielens_accuracy_plot}
\end{subfigure}
\hspace{1mm}
\begin{subfigure}[b]{.48\linewidth}
    \includegraphics[scale=0.38]{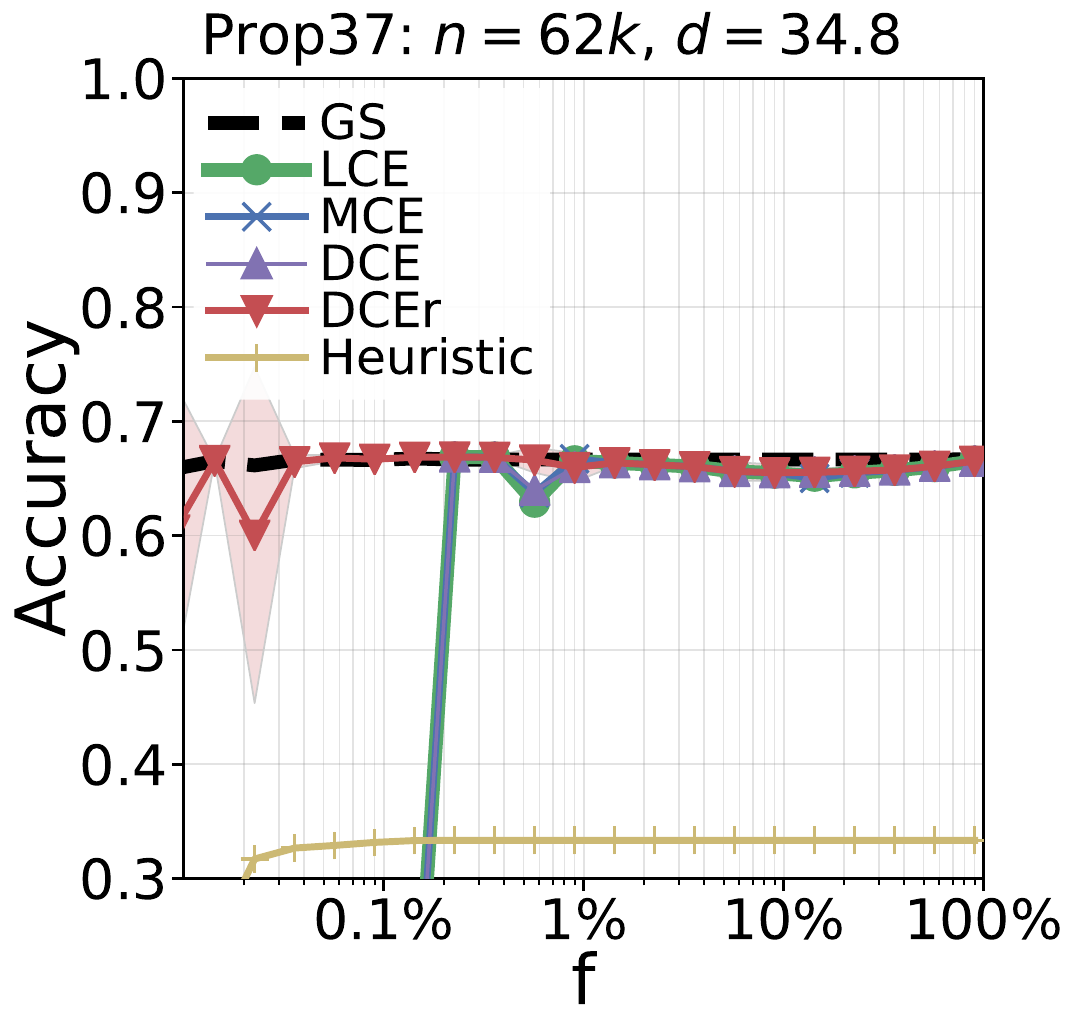}
    \caption{Prop-37}
	\label{Fig_End-to-End_accuracy_realData_307_prop37_accuracy_plot}
\end{subfigure}

\caption{Performance of heuristics 
that approximates the compatibility matrix $\HVec$ with just two values, $H$ and $L$,
on two real data sets. 
}
\end{figure}

\subsection{Learned matrices vs. neighbor frequency distributions}
\label{sec:L2_distance_analyze}

We show the numerical values of the compatibility matrices for our 8 real data sets in 
\cref{tbl:describe_realdata}.
This is in addition to the lower row of \cref{fig:realdata} which illustrated the relative values in a visual form.

\cref{fig:dataset_l2} plots the L2 distance between the learned compatibility matrix and neighbor frequency distribution $\HVec$, referred to as Gold-standard in previous sections. 
The x-axis denotes zero distance from $\HVec$ and we can see that, baring \cref{hep-th_L2} and \cref{pokec_L2}, DCEr is the closest estimate of $\HVec$. 
In cases where DCEr does not have the minimum L2 distance, we believe there to be local minimas in our optimization function \cref{eq:DCE_energyNB}, which still is able to accurately propagate labels to the unlabeled nodes.
We further studied this performance difference and present our findings in the next section. 

\begin{figure}[t]
\begin{subfigure}[b]{.95\linewidth}
\setlength{\tabcolsep}{0.0pt}
\scalebox{0.96}{
\begin{tabular}{c c c c}
       MovieLens &
       Flickr  &
       Enron &
       Prop-37 
\\
	$ \left[\begin{smallmatrix}
           0.08 & 0.45 & 0.47  \\
           0.45 & 0.02 & 0.53  \\
           0.47 & 0.53 & 0.0  \\
        \end{smallmatrix}\right]$
&
	$\left[\begin{smallmatrix}
           0.17&   0.32&   0.51 \\
           0.32&   0.19&   0.49 \\
           0.51&   0.49&  0.0 \\
       \end{smallmatrix}\right]$
&
       $\left[\begin{smallmatrix}
       0.62&    0.24 &   0.0  &   0.14 \\
       0.24 &   0.06 &   0.55&   0.16 \\
       0.0  &    0.55 &   0.0  &   0.45 \\
       0.14 &   0.16 &   0.45 &  0.25 \\
        \end{smallmatrix}\right]$
&
       $\left[\begin{smallmatrix}
               0.35 & 0.26  & 0.38\\
               0.26 & 0.12  & 0.61\\
               0.38 & 0.61  & 0.0  \\
    \end{smallmatrix}\right]$
\end{tabular}
}
\end{subfigure}

\vspace{2mm}
\begin{subfigure}[b]{.95\linewidth}
	\setlength{\tabcolsep}{0.0pt}
   \scalebox{0.96}{
       \begin{tabular}{c c}

       Cora &
       Citeseer \\
       $ \left[\begin{smallmatrix}
		0.81 & 0.01 & 0.04 & 0.05 & 0.06 & 0.01 & 0.02\\
		0.01 & 0.79 & 0.02 & 0.02 & 0.09 & 0.01 & 0.07\\
		0.04 & 0.02 & 0.81 & 0.02 & 0.03 & 0.05 & 0.04\\
		0.05 & 0.02 & 0.02 & 0.84 & 0.05 & 0. & 0.02\\
		0.06 & 0.09 & 0.03 & 0.05 & 0.7 & 0.01 & 0.06\\
		0.01 & 0.01 & 0.05 & 0. & 0.01 & 0.9 & 0.02\\
		0.02 & 0.07 & 0.04 & 0.02 & 0.06 & 0.02 & 0.78\\
        \end{smallmatrix}\right]$
      &
       $\left[\begin{smallmatrix}
		0.77 & 0. & 0.01 & 0.13 & 0.05 & 0.03 \\
		0. & 0.75 & 0.06 & 0.06 & 0.03 & 0.1 &  \\
		0.01 & 0.06 & 0.77 & 0.1 & 0.03 & 0.03 \\
		0.13 & 0.06 & 0.1 & 0.48 & 0.06 & 0.17 \\
		0.05 & 0.03 & 0.03 & 0.06 & 0.81 & 0.02 \\
		0.03 & 0.1 & 0.03 & 0.17 & 0.02 & 0.64 \\
		     \end{smallmatrix}\right]$
       \end{tabular}
   }
\end{subfigure}

\vspace{2mm}
\begin{subfigure}[b]{.95\linewidth}
	\setlength{\tabcolsep}{0.0pt}
   \scalebox{0.96}{
       \begin{tabular}{c c}

       Pokec-Gender &
       Hep-Th \\
       $ \left[\begin{smallmatrix}
		0.44 & 0.56 & \\
		0.56 & 0.44 & \\
        \end{smallmatrix}\right]$
      &
       $\left[\begin{smallmatrix}
		0.1 & 0.11 & 0.14 & 0.11 & 0.11 & 0.08 & 0.08 & 0.08 & 0.04 & 0.08 & 0.08 & \\
		0.11 & 0.09 & 0.12 & 0.12 & 0.1 & 0.08 & 0.09 & 0.09 & 0.05 & 0.06 & 0.09 & \\
		0.14 & 0.12 & 0.11 & 0.13 & 0.11 & 0.1 & 0.09 & 0.06 & 0.03 & 0.03 & 0.06 & \\
		0.11 & 0.12 & 0.13 & 0.15 & 0.12 & 0.1 & 0.08 & 0.06 & 0.03 & 0.04 & 0.06 & \\
		0.11 & 0.1 & 0.11 & 0.12 & 0.17 & 0.13 & 0.08 & 0.07 & 0.03 & 0.02 & 0.05 & \\
		0.08 & 0.08 & 0.1 & 0.1 & 0.13 & 0.18 & 0.12 & 0.08 & 0.04 & 0.03 & 0.06 & \\
		0.08 & 0.09 & 0.09 & 0.08 & 0.08 & 0.12 & 0.17 & 0.13 & 0.07 & 0.03 & 0.06 & \\
		0.08 & 0.09 & 0.06 & 0.06 & 0.07 & 0.08 & 0.13 & 0.16 & 0.14 & 0.08 & 0.07 & \\
		0.04 & 0.05 & 0.03 & 0.03 & 0.03 & 0.04 & 0.07 & 0.14 & 0.28 & 0.17 & 0.11 & \\
		0.08 & 0.06 & 0.03 & 0.04 & 0.02 & 0.03 & 0.03 & 0.08 & 0.17 & 0.26 & 0.2 & \\
		0.08 & 0.09 & 0.06 & 0.06 & 0.05 & 0.06 & 0.06 & 0.07 & 0.11 & 0.2 & 0.16 & \\
		     \end{smallmatrix}\right]$
       \end{tabular}
   }
\end{subfigure}

\caption{Compatibilities for 8 real world data sets.}
\label{tbl:describe_realdata}
\end{figure}

\begin{figure*}[t]
    \centering
\begin{subfigure}[b]{.24\linewidth}
	\includegraphics[scale=0.38]{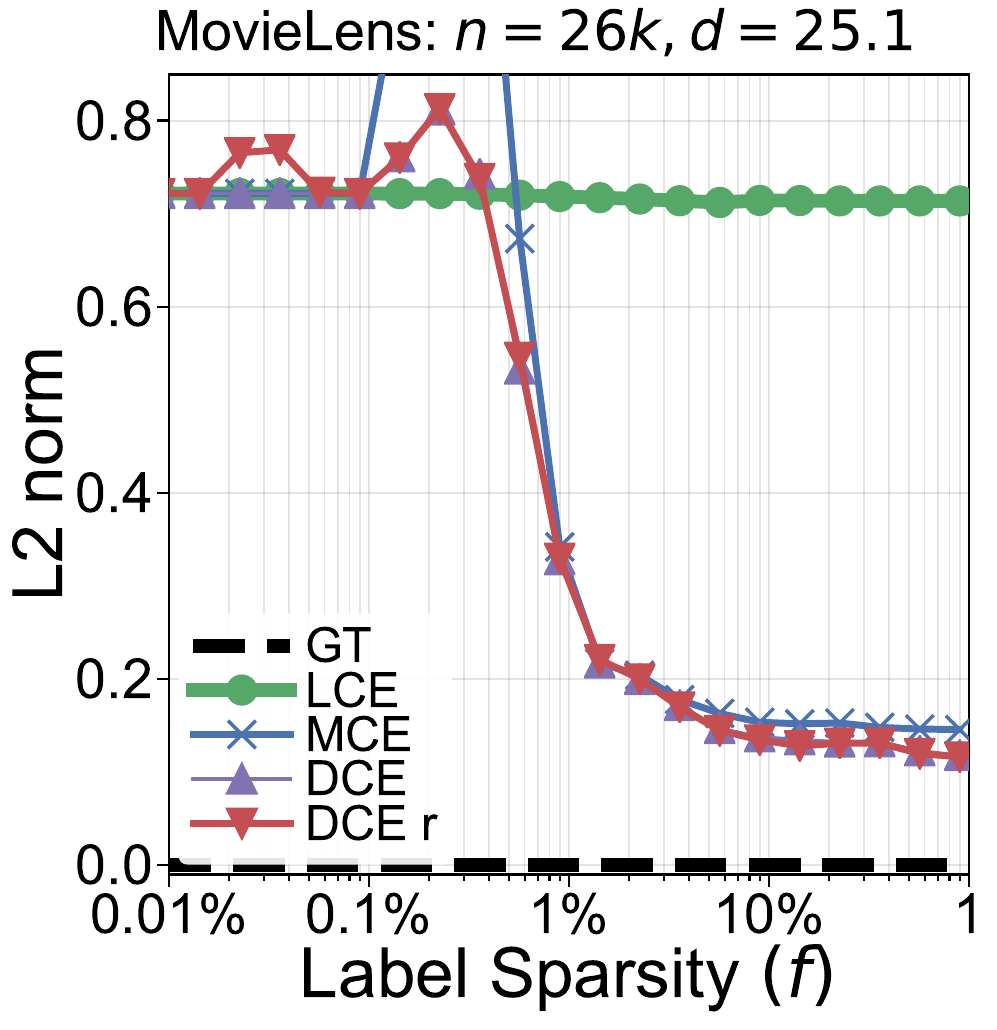}
	\caption{MovieLens}
	\label{movielens_L2}
\end{subfigure}
\begin{subfigure}[b]{.24\linewidth}
	\includegraphics[scale=0.38]{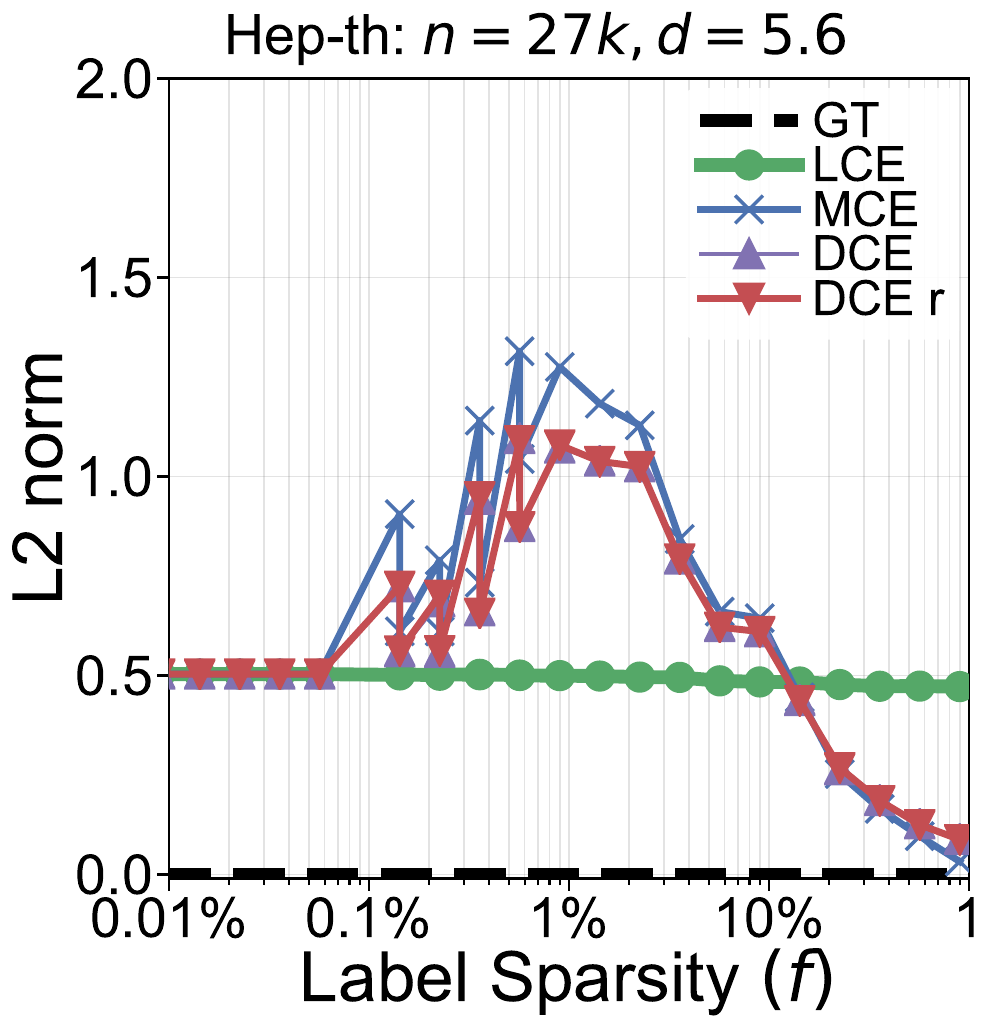}
	\caption{Hep-Th}
	\label{hep-th_L2}
\end{subfigure}
\begin{subfigure}[b]{.24\linewidth}
	\includegraphics[scale=0.38]{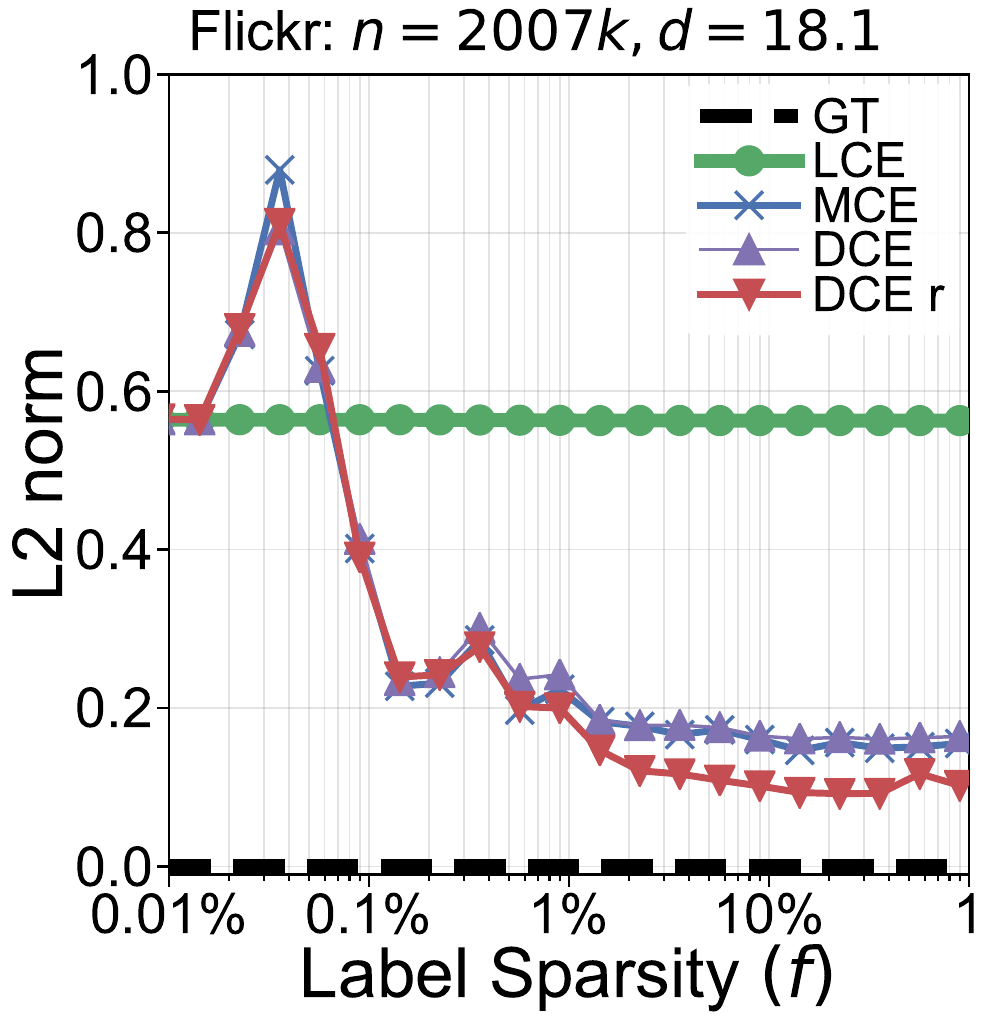}
	\caption{Flickr}
	\label{flickr_L2}
\end{subfigure}
\begin{subfigure}[b]{.24\linewidth}
	\includegraphics[scale=0.38]{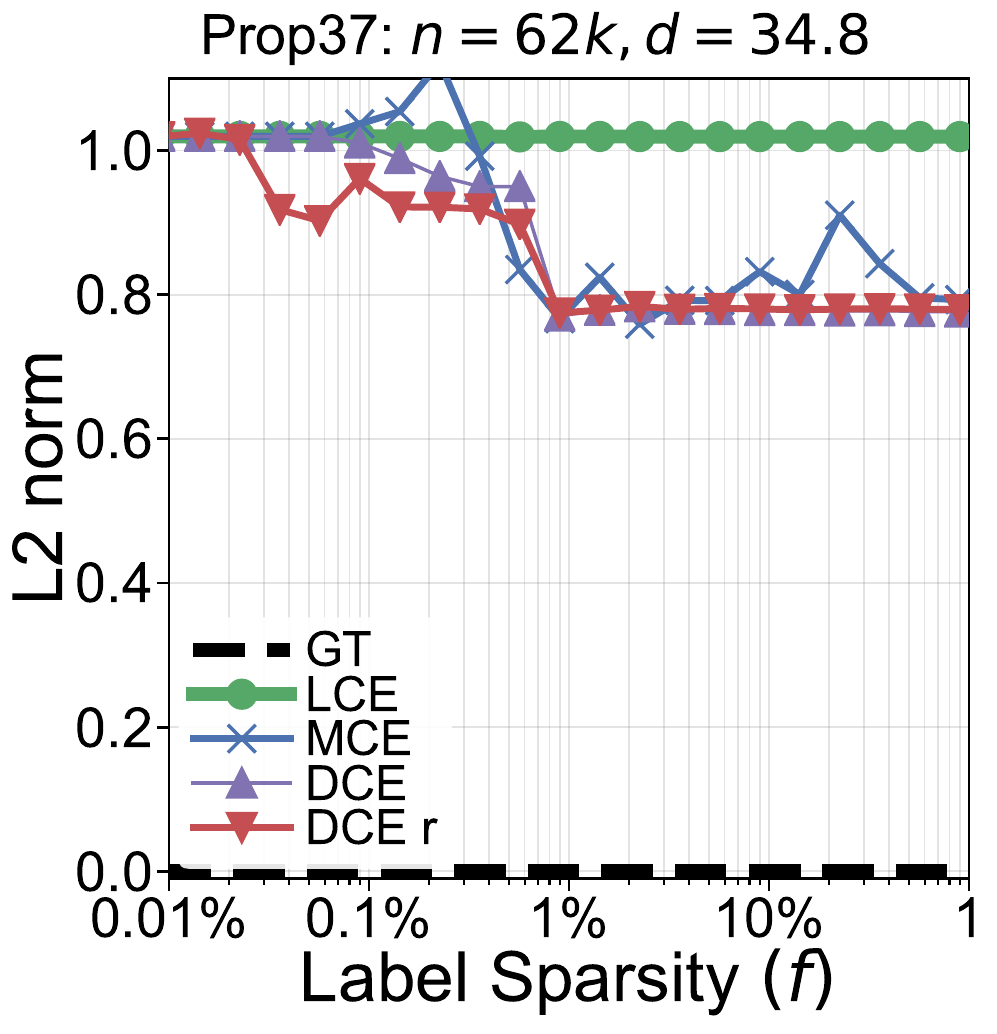}
	\caption{Prop37}
	\label{prop37_L2}
\end{subfigure}

\begin{subfigure}[b]{.24\linewidth}
	\includegraphics[scale=0.38]{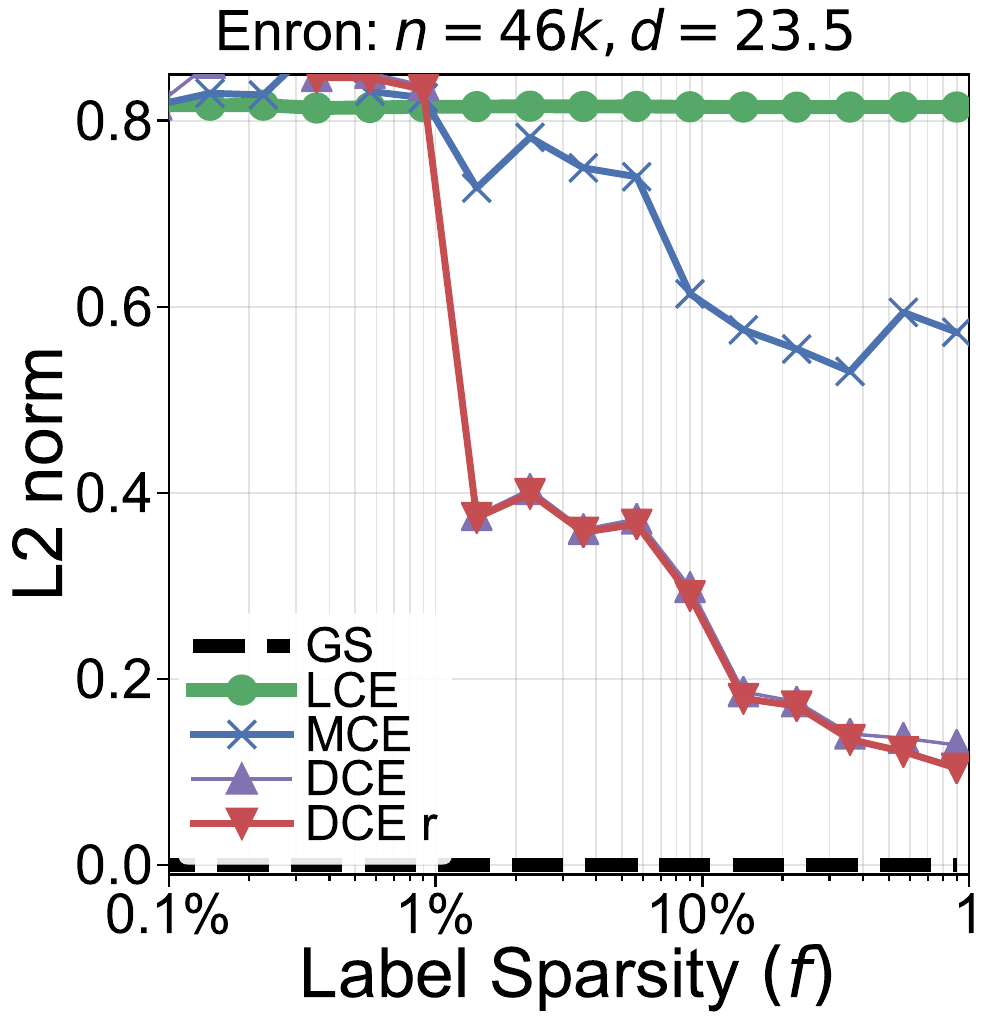}
	\caption{Enron}
	\label{enron_L2}
\end{subfigure}
\begin{subfigure}[b]{.24\linewidth}
	\includegraphics[scale=0.38]{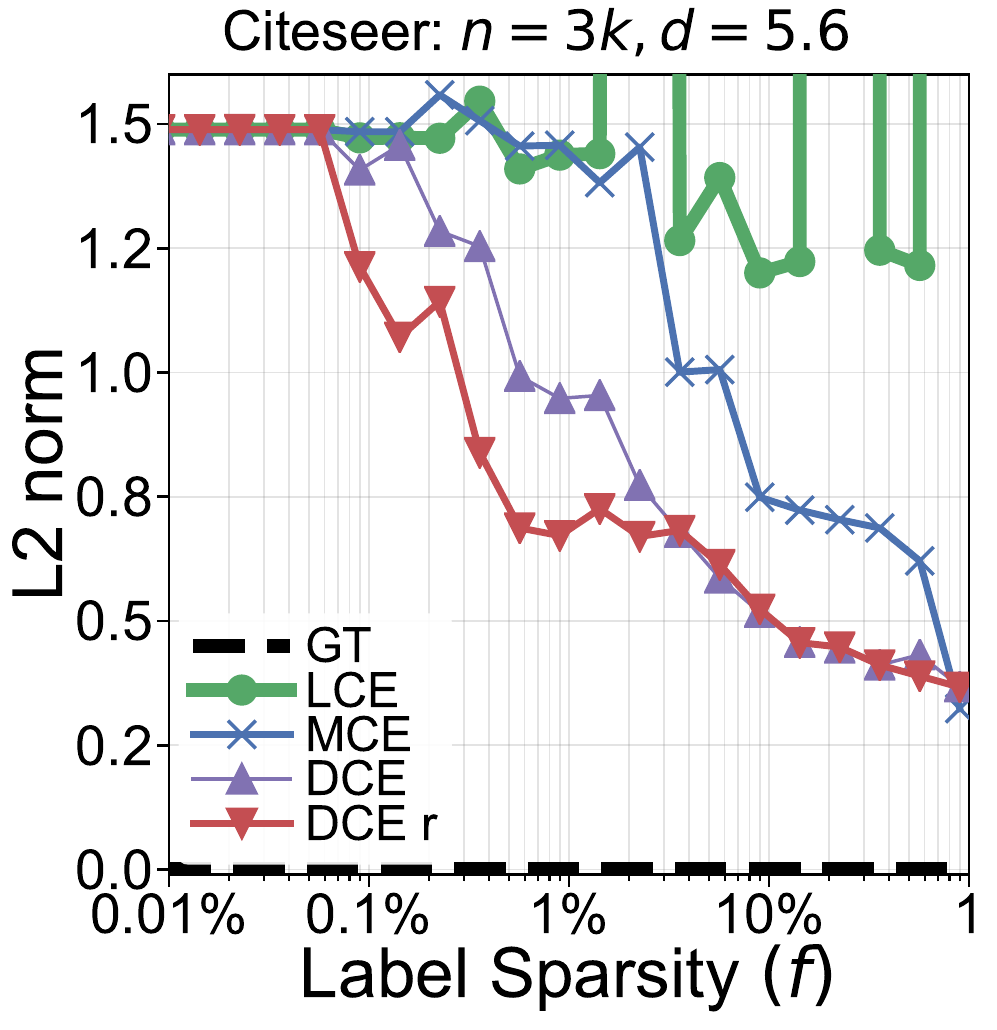}
	\caption{Citeseer}
	\label{citeseer_L2}
\end{subfigure}
\begin{subfigure}[b]{.24\linewidth}
	\includegraphics[scale=0.38]{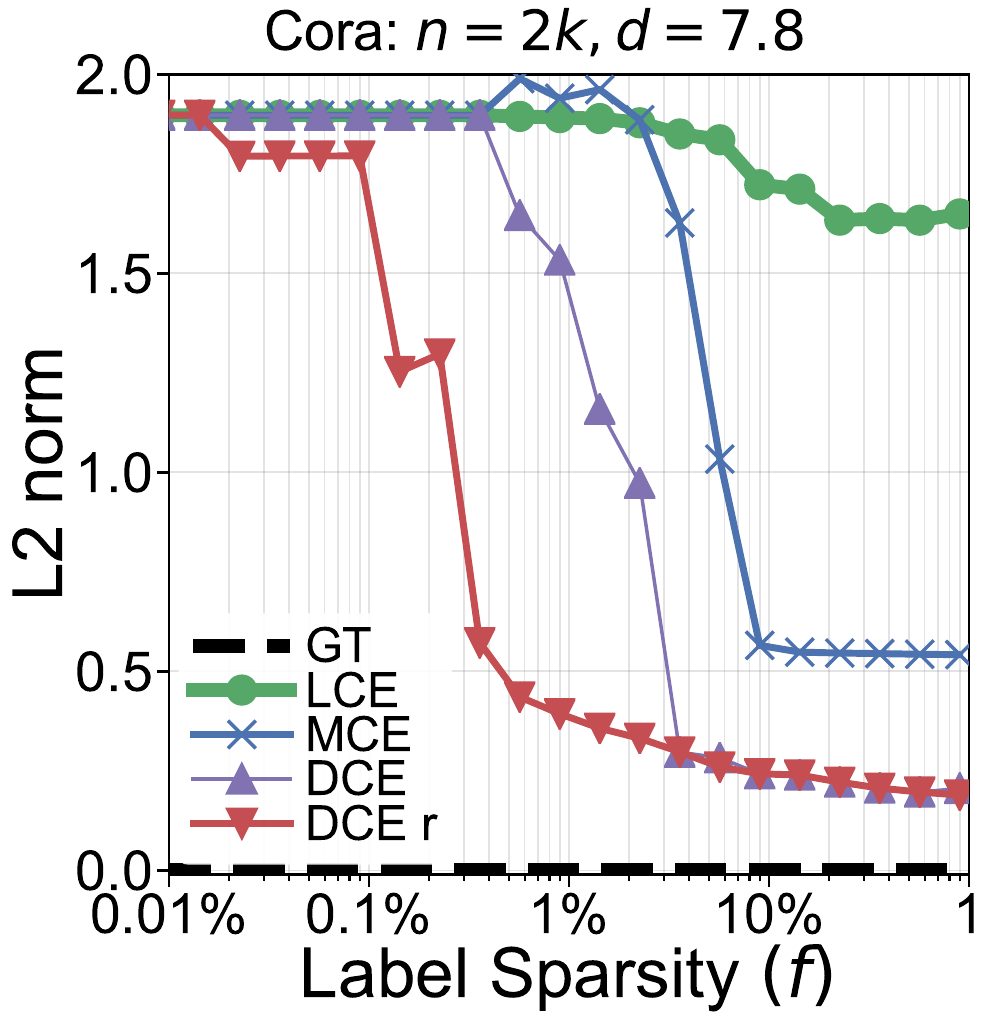}
	\caption{Cora}
	\label{cora_L2}
\end{subfigure}
\begin{subfigure}[b]{.24\linewidth}
	\includegraphics[scale=0.38]{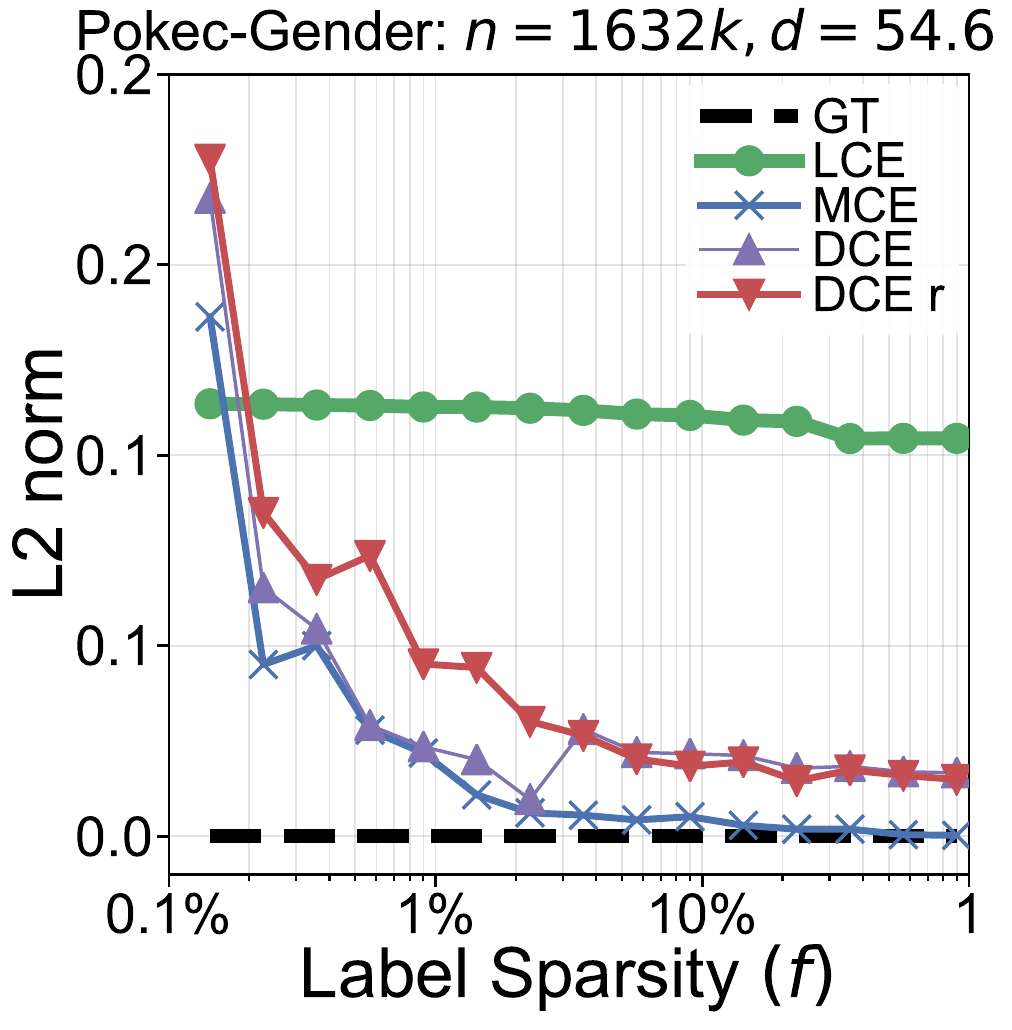}
	\caption{Pokec-Gender}
	\label{pokec_L2}
\end{subfigure}

\caption{L2 distance of estimated matrix from neighbor frequency distribution  $\HVec$ .}\label{fig:dataset_l2}
\end{figure*}

\subsection{Investigating accuracies better than actual label frequencies}

We observe that the neighbor frequency distribution matrix (referred to as gold standard) is not optimal for propagating labels for all levels of sparsity in the graph.
We have re-verified for the correctness and accuracy of results reported, and the 
While $\HVec$ contains the immediate neighbor frequency distribution, we hypothesize that there exist other propagation matrices that are some linear or sophisticated combination of one-hop neighbors information plus higher order frequency distribution.

For example, in Movielens, at around $f=0.35$, the estimated compatibility matrix using DCEr is 
$\left[\begin{smallmatrix}
0.10  & 0.41 &  0.47\\
0.41 &  0.05 &  0.52\\
0.47 &  0.52 &  0.    \\
    \end{smallmatrix}\right]$
, which is little different from neighbor frequency distribution 
$\left[\begin{smallmatrix}
0.07  &  0.45  &  0.47\\
0.45   & 0.01  &  0.52\\
0.47  &  0.52  &  0.    \\
    \end{smallmatrix}\right]$.  
While the magnitude of elements is proportional is both matrices, it remains a very interesting part of our future to fully understand when and why DCEr can estimate better matrices than neighbor frequencies.

\begin{figure}
\begin{subfigure}[b]{0.95\linewidth}
\setlength{\tabcolsep}{1.0pt}
\scalebox{0.96}{
\hspace{4mm}
\begin{tabular}{c c c c c}
       \multicolumn{2}{c}{MovieLens} &
	   &
       \multicolumn{2}{c}{Prop-37}
\\
	$ \left[\begin{smallmatrix}
           0.08 & 0.45 & 0.47  \\
           0.45 & 0.02 & 0.53  \\
           0.47 & 0.53 & 0.0  \\
        \end{smallmatrix}\right]$
&
	$\left[\begin{smallmatrix}
           L&   H&   H \\
           H&   L&   H \\
           H&   H&  L \\
       \end{smallmatrix}\right]$
&
\phantom{xx}
&
       $\left[\begin{smallmatrix}
               0.35 & 0.26  & 0.38\\
               0.26 & 0.12  & 0.61\\
               0.38 & 0.61  & 0.0  \\
    \end{smallmatrix}\right]$
&
       $\left[\begin{smallmatrix}
               H & L  & H\\
               L & L  & H\\
               H & H  & L  \\
    \end{smallmatrix}\right]$	
\end{tabular}
}
\caption{Heuristic approximation of $\HVec$ by two values $H$ and $L$.}
\end{subfigure}

\end{figure}